\documentclass{article}

\usepackage{amsmath}
\usepackage{amssymb}
\usepackage{amsthm}
\usepackage{booktabs}
\usepackage{xspace}
\usepackage{graphicx}

\usepackage{float}
\usepackage{subcaption}

\usepackage{mathtools}
\usepackage{pifont}
\usepackage{cancel}
\usepackage{algorithm}
\usepackage{algorithmic}
\usepackage{multirow}
\usepackage{colortbl}
\usepackage{utfsym}

\def\bfm{\mathbf{m}}

\def\bfv{\mathbf{v}}
\def\bfw{\mathbf{w}}
\def\bfx{\mathbf{x}}
\def\bfy{\mathbf{y}}
\def\bfz{\mathbf{z}}

\def\rmd{\mathrm{d}}
\def\rmdx{\mathrm{d}\bfx}
\def\rmdz{\mathrm{d}\bfz}

\def\bbE{\mathbb{E}}
\def\bbR{\mathbb{R}}
\def\bbV{\mathbb{V}}

\def\nablazt{\nabla_{\bfz_t}}

\def\nablaxt{\nabla_{\bfx_t}}

\def\calI{\mathcal{I}}
\def\calR{\mathcal{R}}

\def\calL{\mathcal{L}}
\def\calD{\mathcal{D}}
\def\calN{\mathcal{N}}

\def\hatx{\hat{\bfx}}

\def\bftheta{\boldsymbol{\theta}}
\def\bfeps{\boldsymbol{\epsilon}}

\def\bfphi{\boldsymbol{\phi}}
\def\bfI{\mathbf{I}}

\def\eg{\emph{e.g.}\xspace}
\def\ie{\emph{i.e.}\xspace}
\def\wrt{\emph{w.r.t.}\xspace}

\def\eqref#1{(\ref{#1})}

\def\DAE{\mathrm{DAE}}

\usepackage{microtype}
\usepackage{graphicx}
\usepackage{booktabs} 
\usepackage{xcolor}

\usepackage{hyperref}

\usepackage[accepted]{icml2024}

\usepackage{amsmath}
\usepackage{amssymb}
\usepackage{mathtools}
\usepackage{amsthm}
\usepackage[capitalize,noabbrev]{cleveref}

\theoremstyle{plain}
\newtheorem{theorem}{Theorem}[section]
\newtheorem{proposition}[theorem]{Proposition}
\newtheorem{lemma}[theorem]{Lemma}
\newtheorem{corollary}[theorem]{Corollary}
\theoremstyle{definition}

\newtheorem{assumption}[theorem]{Assumption}
\theoremstyle{remark}
\newtheorem{remark}[theorem]{Remark}

\def\ourName{GITS\xspace}

\usepackage[textsize=tiny]{todonotes}

\icmltitlerunning{On the Trajectory Regularity of ODE-based Diffusion Sampling}

\begin{document}

\twocolumn[
\icmltitle{On the Trajectory Regularity of ODE-based Diffusion Sampling}




\begin{icmlauthorlist}
\icmlauthor{Defang Chen}{zju}
\icmlauthor{Zhenyu Zhou}{zju}
\icmlauthor{Can Wang}{zju}
\icmlauthor{Chunhua Shen}{zju}
\icmlauthor{Siwei Lyu}{ub}
\end{icmlauthorlist}

\icmlaffiliation{zju}{Zhejiang University, China}
\icmlaffiliation{ub}{University at Buffalo, USA}

\icmlcorrespondingauthor{Defang Chen}{defchern@zju.edu.cn}

\icmlkeywords{Machine Learning, ICML}

\vskip 0.3in
]



\printAffiliationsAndNotice{}  

	
\begin{abstract}
    Diffusion-based generative models use stochastic differential equations (SDEs) and their equivalent ordinary differential equations (ODEs) to establish a smooth connection between a complex data distribution and a tractable prior distribution. 
    In this paper, we identify several intriguing trajectory properties in the ODE-based sampling process of diffusion models. 
    We characterize an implicit denoising trajectory and discuss its vital role in forming the coupled sampling trajectory with a strong shape regularity, regardless of the generated content. 
    We also describe a dynamic programming-based scheme to make the time schedule in sampling better fit the underlying trajectory structure. 
    This simple strategy requires minimal modification to any given ODE-based numerical solvers and incurs negligible computational cost, while delivering superior performance in image generation, especially in $5\sim 10$ function evaluations. 
\end{abstract}

\section{Introduction}
\label{sec:intro}

Diffusion-based generative models~\cite{sohl2015deep,song2019ncsn,ho2020ddpm,song2021sde} have gained significant attention and achieved remarkable results in image~\cite{dhariwal2021diffusion,rombach2022ldm}, audio~\cite{kong2021diffwave}, video~\cite{ho2022video,blattmann2023videoLDM}, and notably in text-to-image synthesis~\cite{saharia2022photorealistic,ruiz2023dreambooth,podell2024sdxl,esser2024scaling}. These models introduce noise into data through a {\em forward process} and subsequently generate data by sampling via a {\em backward process}. Both processes are characterized and modeled using stochastic differential equations (SDEs)~\cite{song2021sde}. In diffusion-based generative models, the pivotal element is the score function, defined as the gradient of the log data density \textit{w.r.t.}\ the input~\cite{hyvarinen2005estimation,lyu2009interpretation,raphan2011least,vincent2011dsm}, irrespective of the model's specific configurations. 
Training such a model involves learning the score function, which is achievable by developing a noise-dependent denoising model. This model is trained to minimize the mean square error in data reconstruction for the data-noise pairings generated during the forward process~\cite{kingma2021vdm,karras2022edm}.
To generate data, diffusion-based generative models solve the acquired score-based backward SDE through a numerical solver. Recent research has shown that the backward SDE can be effectively replaced by an equivalent probability flow ordinary differential equation (PF-ODE), preserving identical marginal distributions~\cite{song2021sde,song2021ddim,lu2022dpm,zhang2023deis,zhou2024fast}. This deterministic ODE-based generation reduces the need for stochastic sampling to just the randomness in the initial sample selection, thereby simplifying and granting more control over the entire generative process~\cite{song2021ddim,song2021sde}. Under the PF-ODE formulation, starting from white Gaussian noise, the \textit{sampling trajectory} is formed by running a numerical solver with discretized time steps. These steps collectively constitute the \textit{time schedule} used in sampling.

\begin{figure}[t]
	\centering
	\includegraphics[width=\columnwidth]{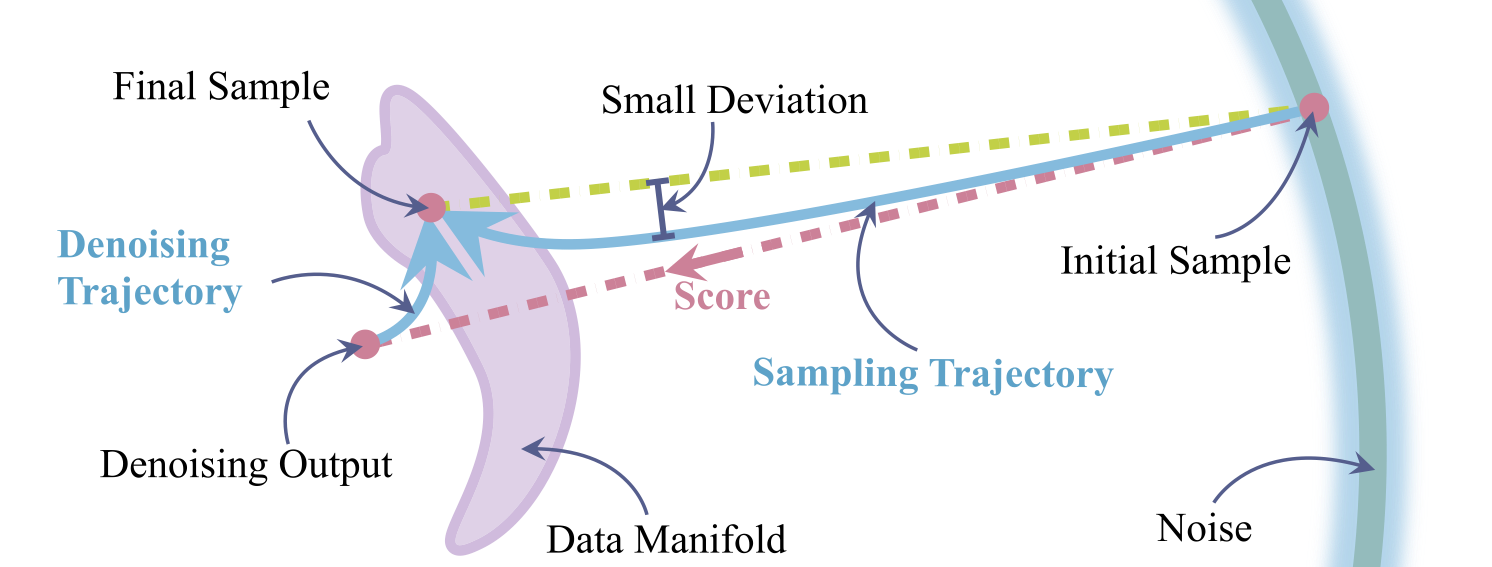}
	\caption{A geometric picture of ODE-based sampling in diffusion models. Each initial sample (from the noise distribution) starts from a big sphere and converges to the final sample (in the data manifold) along a regular sampling trajectory, which is controlled by an implicit denoising trajectory.
    }
    \label{fig:model}
\end{figure}

Despite the impressive generative capabilities exhibited by diffusion-based models, many mathematical and statistical aspects of these models remain veiled in mystery, largely due to the complex nature of SDEs and neural network models, and the substantial data dimensions involved. In particular, empirical studies indicate an intriguing regularity in the sampling trajectories of PF-ODE based diffusion models~\cite{chen2023geometric}, \ie, the tendency of sample paths to exhibit a ``boomerang" shape, or specifically, a linear-nonlinear-linear structure as depicted in Figure~\ref{fig:model}. In addition, we observe that each sampling trajectory barely strays from the straight line joining its beginning and end points, a deviation that can be effectively approximated using two or three orthogonal bases (Section~\ref{sec:trajectory_visualization}). This pattern appears consistently in different trajectories, irrespective of their initial random samples or the corresponding real data samples (see Figure~\ref{fig:traj_3d}). 
This simple structure guarantees the common use of a shared time schedule for synthesizing different samples, and enable us to safely adopt large sampling steps without incurring much truncation error~\cite{song2021ddim,karras2022edm,lu2022dpm}, especially at the first step~\cite{dockhorn2022genie,zhou2024fast}.

We hypothesize that this regularity reflects some underlying geometric structures of the sampling trajectories. This work aims to elucidate this phenomenon. We start by simplifying the ODE-based sampling trajectory, which reveals an implicit \textit{denoising trajectory}. The denoising trajectory corresponds to a rotation of each point on the sampling trajectory and thus determines the curvature of the sampling trajectory (Section~\ref{subsec:denoising_trajectory}). Built upon this insight, we show that the denoising trajectory affords a closed-form solution when we use a kernel density estimation (KDE) of varying widths to approximate the original data distribution from training samples. This is analogous to the classical mean-shift algorithm~\cite{fukunaga1975estimation,cheng1995mean,comaniciu2002mean}, albeit an important difference is that we use time-varying width in KDE (Section~\ref{subsec:meanshift}). Though not feasible for the practical solution of the sampling trajectories, the KDE-based solution converges to the optimal solution based on the real data distribution asymptotically. Its closed form lends itself to theoretical analysis. We show that the linear-nonlinear-linear structure follows naturally from this interpretation of the PF-ODE. This trajectory regularity unifies prior observations and clarifies many existing heuristics to expedite diffusion model sampling. 
Using the shape regularity of the sampling trajectories, we introduce an efficient and effective accelerated sampling approach based on dynamic programming to determine the optimal time schedule. Experimental results demonstrate that trajectory regularity-based accelerated sampling can significantly improve the performance of diffusion-based generative models in a few ($\leq10$) function evaluations. Our main contributions are summarized as follows\footnote{The unpublished, early manuscript of this paper can be found in arXiv~\cite{chen2023geometric}.}:
\begin{itemize} 
    \item We describe and demonstrate a strong shape regularity of trajectories of ODE-based diffusion sampling, \ie, the sampling trajectories approximately follow a similar shape with a linear-nonlinear-linear structure. 
    \item We explain this regularity through the closed form of the denoising trajectory under a KDE-based data modeling with the time-varying bandwidth.
    \item We develop a dynamic programming-based approach that leverages the trajectory regularity to find the optimal time schedule of the sampling steps. It introduces minimal overhead and yields improved image quality.
\end{itemize}

\section{Preliminaries}
\label{sec:preliminaries}
For successful generative modeling, it is essential to connect the data distribution $p_{d}$ with a non-informative, manageable distribution $p_{n}$. Diffusion models fulfill this objective by incrementally introducing white Gaussian noise into the data, effectively obliterating its structures, and subsequently reconstructing the synthesized data from noise samples via a series of denoising steps. The forward step can be modeled as a diffusion process $\{\bfx_t\}_{t=0}^T$ starting from $\bfx_0\sim p_d$, which is the solution of a linear It\^{o} stochastic differential equation (SDE)~\cite{song2021sde,karras2022edm}
\begin{equation}
	\label{eq:forward_sde}
	\rmd \bfx_t = f(t)\bfx_t \rmd t + g(t) \rmd \bfw_t,  
\end{equation}
where $\bfw_t$ denotes the Wiener process; $f(t)\bfx_t$ and $g(t)$ are the drift and diffusion coefficients, respectively. The marginal distribution $p_t(\bfx_t)$ evolves according to the well-known Fokker-Planck equation given the initial condition $p_0(\bfx_0)=p_d(\bfx_0)$~\cite{oksendal2013stochastic}. By properly setting the drift and diffusion coefficients, the data distribution is smoothly transformed to the (approximate) noise distribution $p_T(\bfx_T)\approx p_n$ in a forward manner. The transition kernel in this context is always a Gaussian distribution, \ie, $p_{0t}(\bfx_t | \bfx_0)= \calN(\bfx_t ; s(t)\bfx_0, s^2(t)\sigma^2(t)\bfI)$, where $s(t) = \exp \left(\int_{0}^{t} f(\xi) \rmd \xi\right)$, $\sigma(t) = \sqrt{\int_{0}^{t} \left[g(\xi)/s(\xi)\right]^2\rmd \xi}$, and we denote them as $s_t$, $\sigma_t$ hereafter for notation simplicity. The signal-to-noise ratio (SNR) is defined as $1/\sigma^2_t$~\cite{karras2022edm}. More details are provided in Section~\ref{subsec:equivalence}. 

In the literature, two forms of SDEs are commonly used, namely, the variance-preserving (VP) SDE and the variance-exploding (VE) SDE~\cite{song2021sde}, and they are both widely used in large-scale generative models~\cite{ramesh2022hierarchical,rombach2022ldm,balaji2022ediffi,yuan2023physdiff}.
Our analysis will be based on VE-SDEs and the results can be easily extended to VP-SDEs. In fact, we can safely remove the drift term in \eqref{eq:forward_sde} without changing the essential characteristics of the underlying diffusion model. 
\begin{remark}[Proofs in Section~\ref{subsec:equivalence}]
    \label{remark:ito}
    Linear diffusion processes sharing the same SNR of transition kernels 
    are equivalent according to It\^{o}'s lemma~\cite{oksendal2013stochastic}.
\end{remark}
Because of this, we only consider a standardized VE-SDE
\begin{equation}
	\label{eq:new_sde}
	\rmd \bfx_t = \sqrt{\rmd \sigma^2_t/\rmd t}\, \rmd \bfw_t, \quad \sigma_t: \bbR\rightarrow \bbR, 
\end{equation}
where $\sigma_t$ is a pre-defined increasing noise schedule. 

The reverse of the forward SDE, as expressed in \eqref{eq:new_sde}, is represented by another SDE which facilitates the synthesis of data from noise through a backward sampling  \cite{feller1949theory,anderson1982reverse}. Notably, a probability flow ordinary differential equation (PF-ODE) exists and maintains the same marginal distributions $\{p_t(\bfx_t)\}_{t=0}^T$ as the SDE at each time step throughout the diffusion process~\cite{song2021sde}:
\begin{equation}
	\label{eq:pf_ode}
    \rmd \bfx_t = - \sigma_t\nablaxt \log p_t(\bfx_t) \rmd \sigma_t.
\end{equation}
The deterministic nature of ODE offers several benefits in generative modeling, including efficient sampling, unique encoding, and meaningful latent manipulations~\cite{song2021sde,song2021ddim}. We thus choose this formula to analyze the sampling behavior of diffusion models throughout this paper.

\textbf{Training.} Simulating the above PF-ODE \eqref{eq:pf_ode} requests having the score function $\nablaxt \log p_t(\bfx_t)$~\cite{hyvarinen2005estimation,lyu2009interpretation}, which is typically estimated with \textit{denoising score matching} \cite{vincent2011dsm}. 
Thanks to a profound connection between the score function and the posterior expectation, \ie, $\bbE(\bfx_0|\bfx_t)=\bfx_t+\sigma^2_t\nablaxt \log p_t(\bfx_t)$~\cite{robbins1956eb,efron2011tweedie,raphan2011least}, we can also train a \textit{denoising autoencoder} (DAE)~\cite{vincent2008extracting,bengio2013generalized}, denoted as $r_{\bftheta}$, to estimate the conditional expectation $\bbE(\bfx_0|\bfx_t)$, and then convert it to the score function. The objective function of training such a neural network across different noise levels with the weighting function $\lambda(t)$ is $\calL_{\DAE}\left(\bftheta; \lambda(t)\right):=$
\begin{equation}
	\label{eq:dae}
	\int_{0}^{T} \lambda(t)\bbE_{\bfx_0 \sim p_d} \bbE_{\bfx_t \sim p_{0t}(\bfx_t|\bfx_0)} \lVert r_{\bftheta}(\bfx_t; \sigma_t) - \bfx_0  \rVert^2_2 \rmd t.
\end{equation}
The optimal estimator $r_{\bftheta}^{\star}(\bfx_t; \sigma_t)$, also known as Bayesian least squares estimator,
equals $\bbE(\bfx_0|\bfx_t)$. We thus have $r_{\bftheta}^{\star}(\bfx_t; \sigma_t)=\bfx_t+\sigma^2_t\nablaxt \log p_t(\bfx_t)$. In practice, it is assumed that this connection approximately holds for a converged model ($r_{\bftheta}(\bfx_t; \sigma_t)\approx r_{\bftheta}^{\star}(\bfx_t; \sigma_t)$\footnote{We slightly abuse the notation and still denote the converged model as $r_{\bftheta}(\cdot; \cdot)$ hereafter.}), and we can plug it into \eqref{eq:pf_ode} to derive the \textit{empirical} PF-ODE as follows
\begin{equation}
	\label{eq:epf_ode}
	\rmd \bfx_t = \frac{\bfx_t-r_{\bftheta}(\bfx_t; \sigma_t)}{\sigma_t}\rmd \sigma_t=\bfeps_{\bftheta}(\bfx_t; \sigma_t)\rmd \sigma_t.
\end{equation}
The noise-prediction model $\bfeps_{\bftheta}(\cdot, \cdot)$ introduced above are used in some previous works~\cite{ho2020ddpm,song2021ddim,nichol2021improved,bao2022analytic}. 

\textbf{Sampling.} Given the empirical PF-ODE~\eqref{eq:epf_ode}, we can synthesize novel samples by first drawing pure noises $\hatx_{t_N} \sim p_n$ as the initial condition, and then numerically solving this equation backward with $N$ steps to obtain a sequence $\{\hatx_{t_n}\}_{n=0}^{N}$ with a certain time schedule $\Gamma=\{t_0=\epsilon, \cdots, t_N=T\}$. We adopt the notation $\hatx_{t_n}$ to denote the generated sample by numerical methods, which differs from the exact solutions denoted as $\bfx_{t_n}$. The final sample $\hatx_{t_0}$ is considered to approximately follow the data distribution $p_{d}$. We designate this sequence as a \textit{sampling trajectory}. In practice, there exists various sampling strategies inspired from the classic numerical methods to solve the backward PF-ODE~\eqref{eq:epf_ode}, including Euler~\cite{song2021ddim}, Heun's~\cite{karras2022edm}, Runge-Kutta~\cite{song2021sde,liu2022pseudo,lu2022dpm}, and linear multistep methods~\cite{liu2022pseudo,lu2022dpmpp,zhang2023deis,zhao2023unipc}. 
\section{Regularity of PF-ODE Sampling Trajectory}
\label{sec:trajectory_visualization}

As mentioned in Section~\ref{sec:intro}, the sampling trajectories within a PF-ODE framework of the diffusion model exhibit a certain regularity in their shapes, regardless of the specific content generated. To better demonstrate this concept, we undertake a series of empirical studies.

\noindent{\bf 1-D Projections}. Visualizing the entire sampling trajectory and analyzing its geometric characteristics in the original high-dimensional space is intractable. To address this, we first examine the \textit{trajectory deviation} from the direct line connecting the two endpoints, which helps assess the trajectory's linearity (see Figure~\ref{fig:model}). This approach allows us to align and collectively observe the general behaviors of all sampling trajectories. Specifically, we determine the trajectory deviation as the perpendicular $L^2$ distance from each intermediate sample $\hatx_{t_n}$ to the vector $\hatx_{t_N}-\hatx_{t_0}$, depicted by the red curve with a ``boomerang'' shape in Figure~\ref{fig:deviation}. Additionally, we calculate the $L^2$ distance between $\hatx_{t_n}$ and the final sample $\hatx_{t_0}$ in the trajectory as $\lVert \hatx_{t_n}-\hatx_{t_0}\rVert_2$, depicted by the blue monotone curve in Figure~\ref{fig:deviation}. 

\begin{figure}[t]
    \centering
    \includegraphics[width=\columnwidth]{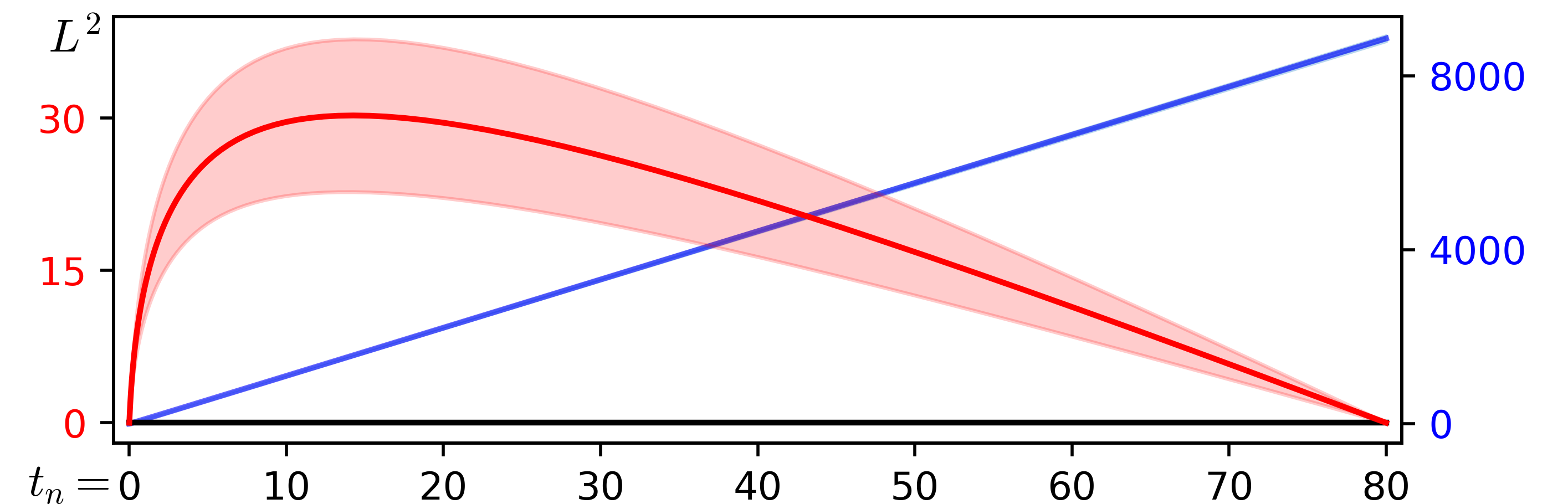}
    \caption{The sampling trajectory exhibits a very small trajectory deviation (red curve) compared to the sample distance (blue curve) in the sampling process starting from $t_{N}=80$ to $t_0=0.002$. 
    }
    \label{fig:deviation}
\end{figure}

\begin{figure*}[t]
    \centering
    \begin{subfigure}[t]{0.2\textwidth}
        \includegraphics[width=\columnwidth]{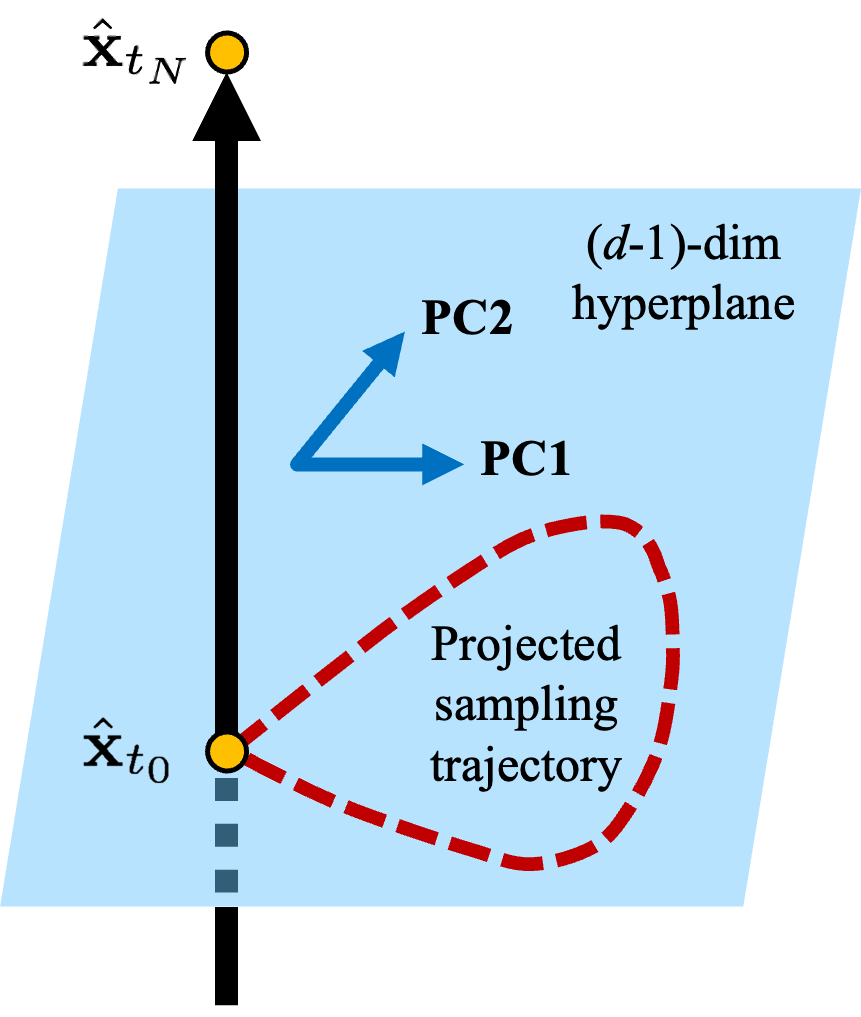}
        \caption{Trajectory projection.}
        \label{fig:recon_sketch}
    \end{subfigure}
    \hfill
    \begin{subfigure}[t]{0.38\textwidth}
        \includegraphics[width=\columnwidth]{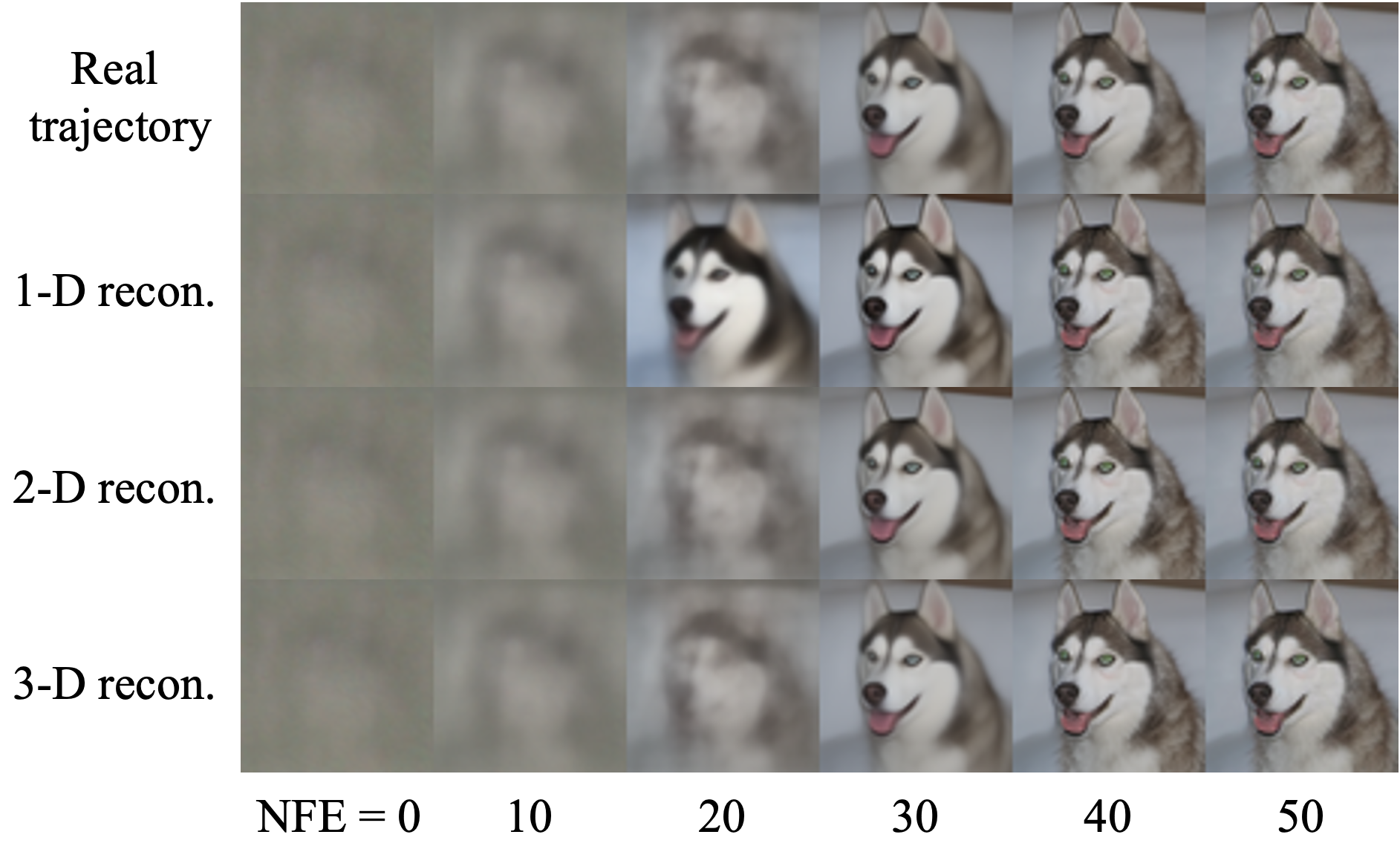}
        \caption{Visual Comparison.}
        \label{fig:recon_visual}
    \end{subfigure}
    \hfill
    \begin{subfigure}[t]{0.21\textwidth}
        \includegraphics[width=\columnwidth]{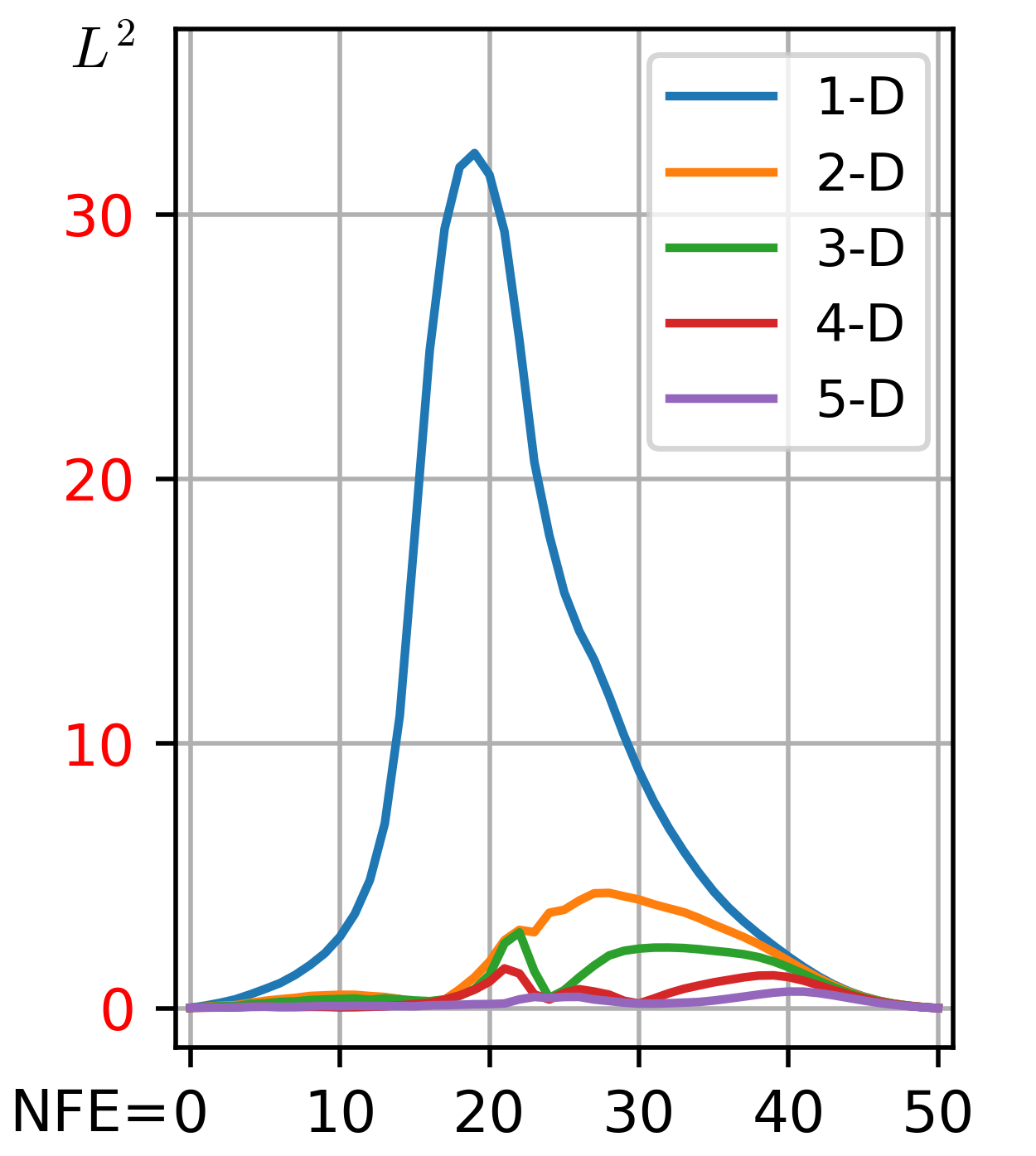}
        \caption{$L^2$ error.}
        \label{fig:recon_l2}
    \end{subfigure}
    \hfill
    \begin{subfigure}[t]{0.17\textwidth}
        \includegraphics[width=\columnwidth]{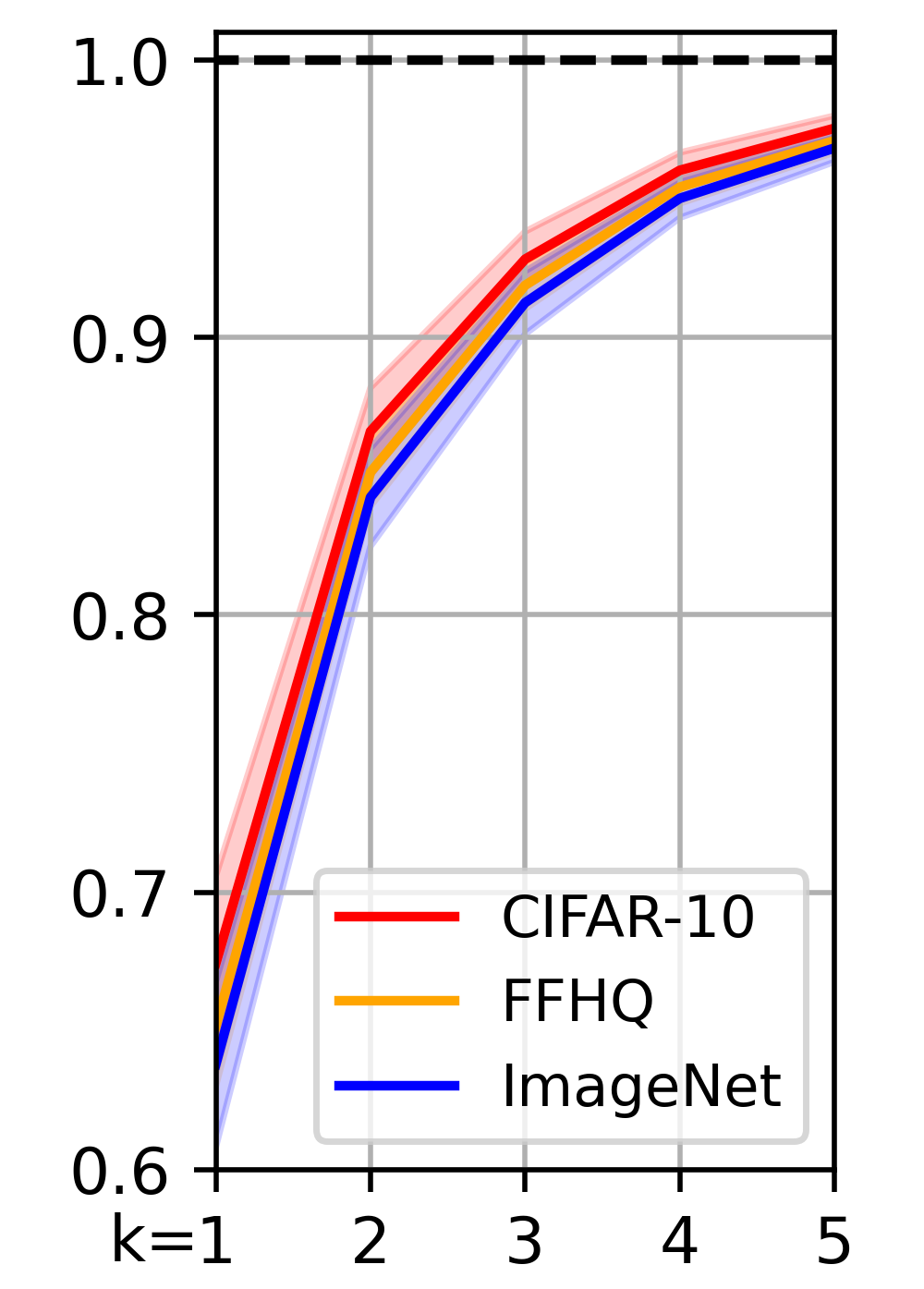}
        \caption{PCA ratio.}
        \label{fig:recon_pca}
    \end{subfigure}
    \caption{
    (a) We adopt $d$-dimensional vector $\hatx_{t_N}-\hatx_{t_0}$ and several top principal components (PCs) on its $(d-1)$-dimensional orthogonal complement to approximate the original $d$-dimensional sampling trajectory.
    (b) The visual comparison of trajectory reconstruction on Imagenet 64$\times$64. We reconstruct the real sampling trajectory (top row) using $\hatx_{t_N}-\hatx_{t_0}$ (1-D recon.) along with its top 1 or 2 principal components (2-D or 3-D recon.). To amplify the visual difference, we present the denoising outputs of these trajectories. 
    (c) We calculate the $L^2$ distance between the real trajectory and the reconstructed trajectories up to 5-D reconstruction. 
    (d) The variance explained by the top $k$ principal components. We report the ratio of the summation of the top $k$ eigenvalues to the summation of all eigenvalues.
    }
    \label{fig:recon}
\end{figure*}

\begin{figure*}[t]
    \centering
    \begin{subfigure}[t]{0.33\textwidth}
        \includegraphics[width=\columnwidth]{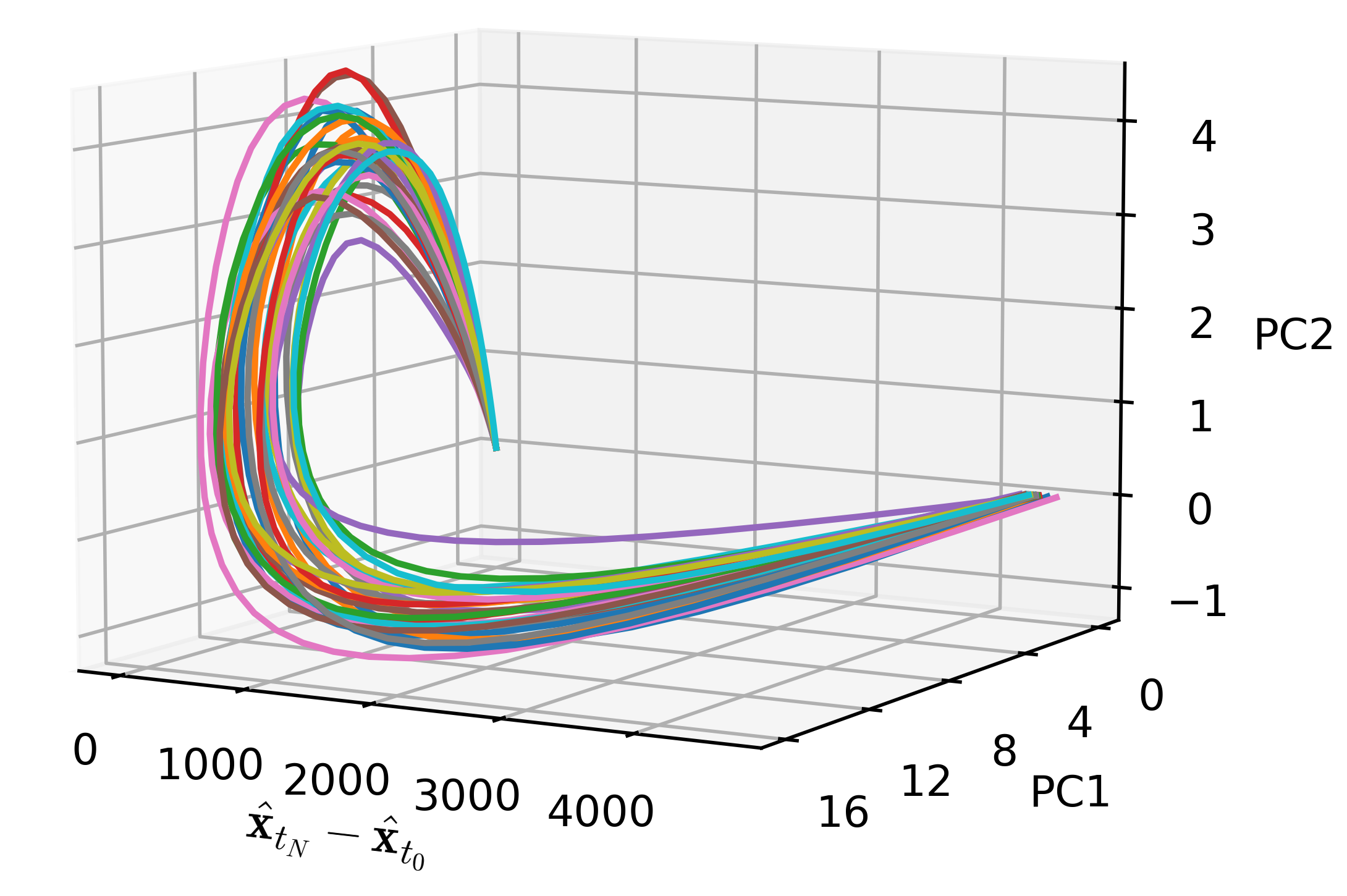}
        \caption{CIFAR-10.}
    \end{subfigure}
    \begin{subfigure}[t]{0.33\textwidth}
        \includegraphics[width=\columnwidth]{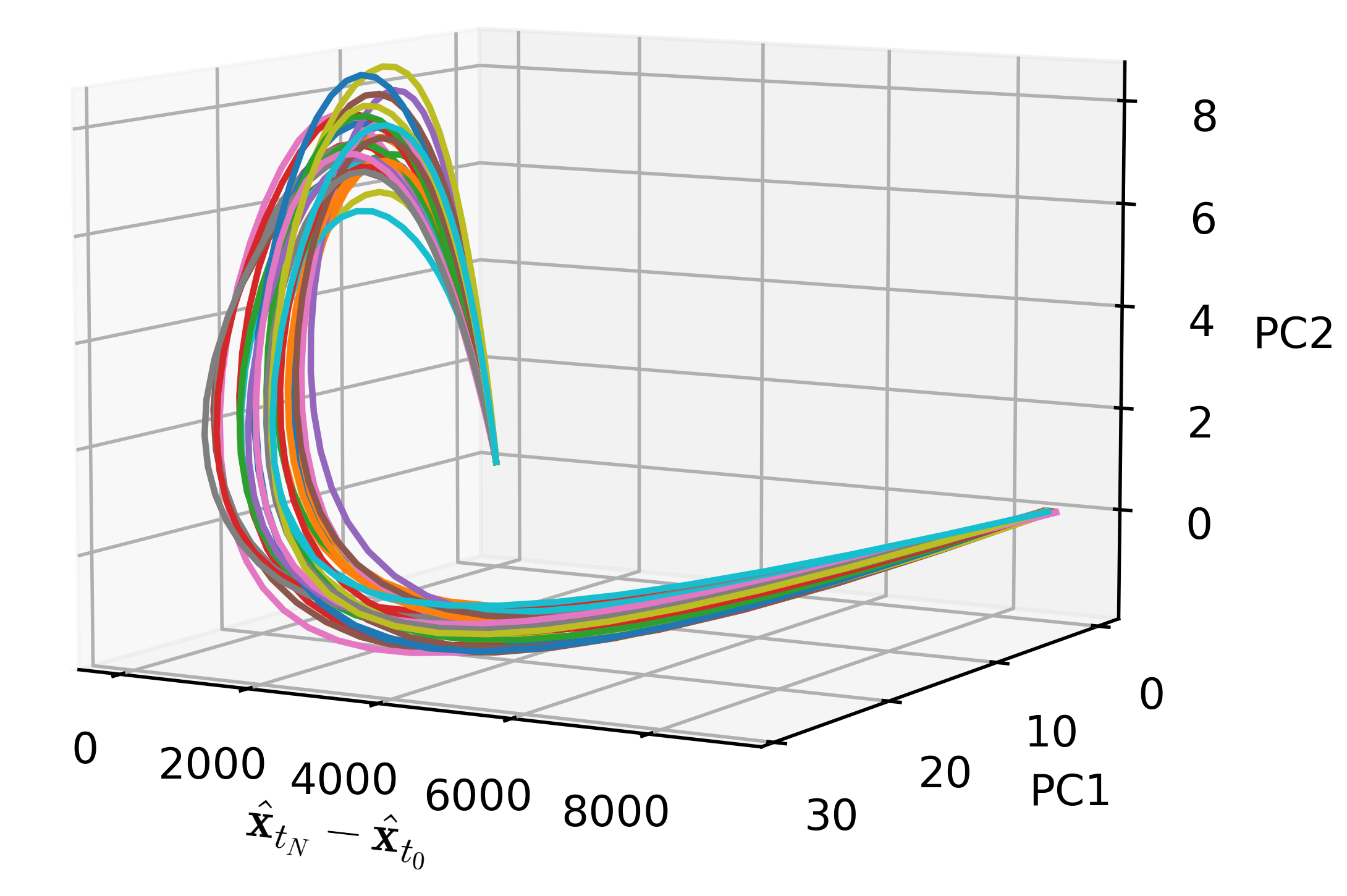}
        \caption{FFHQ.}
    \end{subfigure}
    \begin{subfigure}[t]{0.33\textwidth}
        \includegraphics[width=\columnwidth]{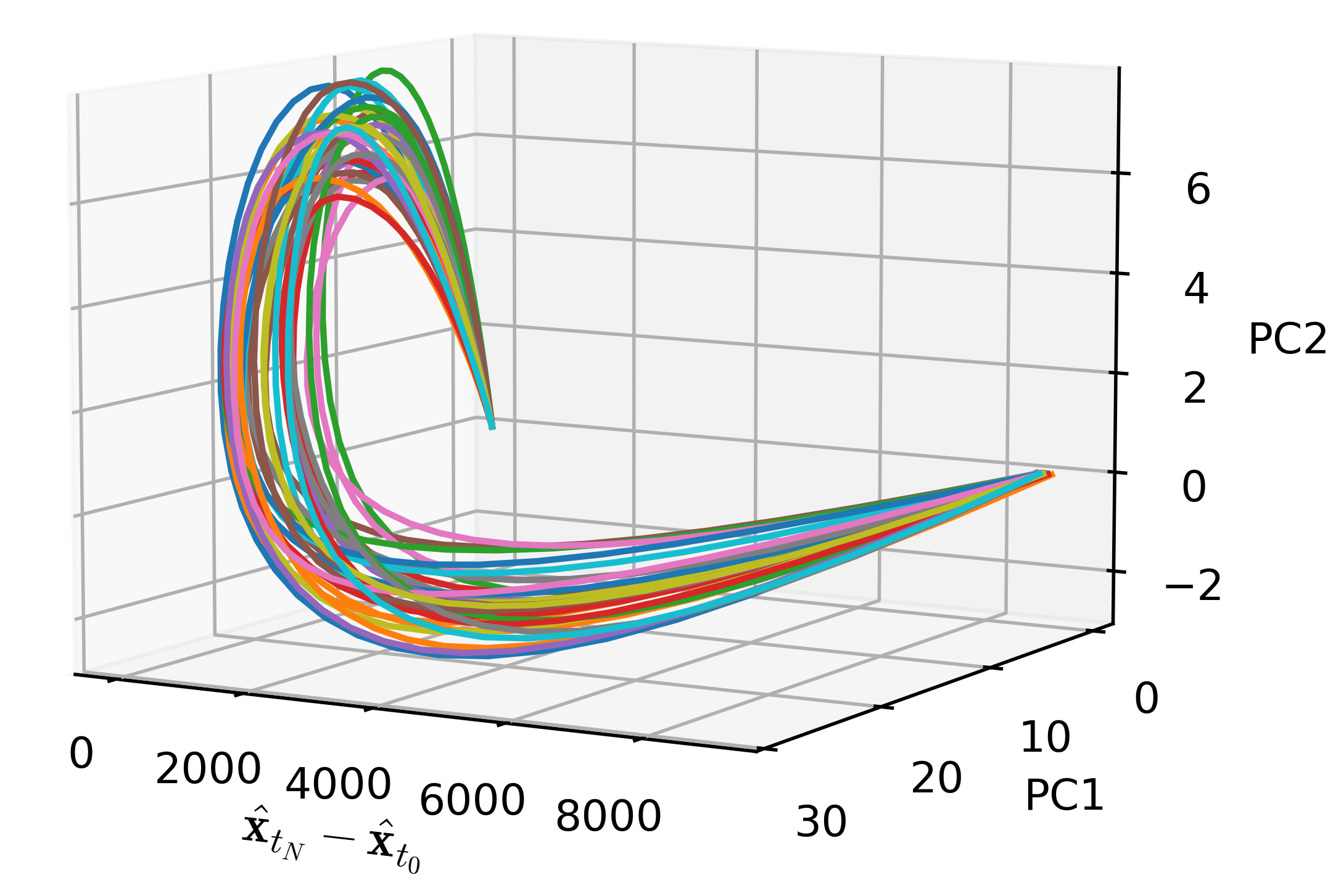}
        \caption{ImageNet $64\times 64$.}
    \end{subfigure}
    \caption{We project 30 high-dimensional sampling trajectories generated on three different datasets into 3-D subspace. These trajectories are first aligned to the direction of $\hatx_{t_N}-\hatx_{t_0}$ (this direction is different for each sample), and then projected to the top 2 principal components in the orthogonal space to $\hatx_{t_N}-\hatx_{t_0}$. See texts for more details.}
    \label{fig:traj_3d}
\end{figure*}


From Figure~\ref{fig:deviation}, we observe that the sampling trajectory's deviation gradually increases from $t=80$ to approximately $t=10$, then swiftly diminishes as it approaches the final samples. This pattern suggests that initially, each sample might be influenced by various modes, experiencing significant impact, but later becomes strongly guided by its specific mode after a certain turning point. This behavior supports the approach of arranging time intervals more densely near the minimum timestamp and sparsely towards the maximum one \cite{song2021ddim,lu2022dpm,karras2022edm,song2023consistency}. 
However, when we consider the ratio of the maximum deviation to the endpoint distance in Figure~\ref{fig:deviation}, we find that the trajectory deviation is remarkably slight (approximately $30/8868 \approx 0.0034$), indicating a pronounced straightness. Additionally, the generated samples along the sampling trajectory tend to move monotonically from their initial points toward their final points. 

The trajectory deviation mathematically equals the reconstruction error if we project all $d$-dimensional points of the sampling trajectory onto the vector $\hatx_{t_N}-\hatx_{t_0}$. As demonstrated in Figure~\ref{fig:recon_visual}-\ref{fig:recon_l2}, the $1$-D approximation proves inadequate, leading to a significant deviation from the actual trajectory. These observations imply that while all sampling trajectories share a similar macro-structure, $1$-D projection cannot accurately capture such trajectory details. 

\noindent{\bf Multi-D Projections}. Moreover, we implement \textit{principal component analysis} (PCA) on the orthogonal complement to the vector $\hatx_{t_N}-\hatx_{t_0}$, which assists in assessing the trajectory's rotational properties. As illustrated in Figure~\ref{fig:recon_sketch}, we begin by projecting each $d$-D sampling trajectory into its $(d-1)$-D orthogonal space relative to the associated $\hatx_{t_N}-\hatx_{t_0}$ vector, followed by conducting PCA. Figure~\ref{fig:recon_visual}-\ref{fig:recon_l2} show that the 2-D approximation using $\hatx_{t_N}-\hatx_{t_0}$ and the first principal component markedly narrows the visual discrepancy with the real trajectory, thereby reducing the $L^2$ reconstruction error. This finding suggests that all points in each $d$-D sampling trajectory diverge slightly from a 2-D plane. Consequently, the tangent and normal vectors of the trajectory can be effectively characterized in this manner.

By incorporating an additional principal component, we enhance our ability to capture the torsion of the sampling trajectory, thereby increasing the total explained variance to approximately 85\% (see Figure~\ref{fig:recon_pca}). This improvement allows for a more accurate approximation of the actual trajectory and further reduces the $L^2$ reconstruction error (see Figure~\ref{fig:recon_visual}-\ref{fig:recon_l2}). In practical terms, this level of approximation effectively captures all the visually pertinent information, with the deviation from the real trajectory being nearly indistinguishable. Consequently, we can confidently utilize a 3-D subspace, formed by two principal components and the vector $\hatx_{t_N}-\hatx_{t_0}$, to understand the geometric structure of high-dimensional sampling trajectories. 

Expanding on this understanding, we present a visualization of $30$ randomly chosen sampling trajectories created by a diffusion model trained on CIFAR-10~\cite{krizhevsky2009learning}, FFHQ~\cite{karras2019style}, or ImageNet $64\times64$~\cite{russakovsky2015ImageNet} in Figure~\ref{fig:traj_3d}. It is important to note that the scale along the axis of $\hatx_{t_N}-\hatx_{t_0}$ is significantly larger than that of the other two principal components by orders of magnitude. As a result, the trajectory remains very close to the straight line connecting its endpoints, corroborating our findings from the trajectory deviation experiment (see Figure~\ref{fig:deviation}). Furthermore, Figure~\ref{fig:traj_3d} accurately depicts the sampling trajectory's behavior, showing its gradual departure from the osculating plane during the sampling process. Interestingly, each trajectory displays a simple linear-nonlinear-linear structure and shares a similar shape. This consistency reveals a strong regularity in all sampling trajectories with different initial starting points, independent of the specific content generated. 

\section{Understanding the Trajectory Regularity}
\label{sec:trajectory}

We next attempt to explain the trajectory regularity observed in the previous section. We first show that there exists an implicit denoising trajectory, which controls the rotation of the sampling trajectory and determines the subsequent points in a convex combination way (Section~\ref{subsec:denoising_trajectory}). We then establish a connection between the sampling trajectory and KDE approximation of data distribution (Section~\ref{subsec:meanshift}), which is the linchpin to understanding the observed regularity.

\subsection{Implicit Denoising Trajectory}
\label{subsec:denoising_trajectory}
Given a parametric diffusion model with the \textit{denoising output} $r_{\bftheta}(\cdot; \cdot)$, the sampling trajectory is simulated by numerically solving the empirical PF-ODE~\eqref{eq:epf_ode}, and meanwhile, an implicitly coupled sequence $\{r_{\bftheta}(\hatx_{t_n}, \sigma_{t_n})\}_{n=1}^{N}$ is formed as a by-product. We designate this sequence, or simplified to $\{r_{\bftheta}(\hatx_{t_n})\}_{n=1}^{N}$ if there is no ambiguity, as a \textit{denoising trajectory}. It follows a PF-ODE $\rmd r_{\bftheta}(\bfx_t; \sigma_t)/\rmd \sigma_t = - \sigma_t \left[\rmd^2 \bfx_t/\rmd \sigma^2_t\right]$ and actually encapsulates the curvature information of the associated sampling trajectory. The following proposition reveals how these two trajectories are inherently related.
\begin{proposition}
    \label{prop:convex}
	Given the probability flow ODE~\eqref{eq:epf_ode} and a current position $\hatx_{t_{n+1}}$, $n\in[0, N-1]$ in the sampling trajectory, the next position $\hatx_{t_{n}}$ predicted by a $k$-order Taylor expansion with the time step size $\sigma_{t_{n+1}}-\sigma_{t_n}$ equals 
	\begin{equation}
    \label{eq:convex}
        \begin{aligned}
		\hatx_{t_{n}}&=\frac{\sigma_{t_n}}{\sigma_{t_{n+1}}} \hatx_{t_{n+1}} +  \frac{\sigma_{t_{n+1}}-\sigma_{t_n}}{\sigma_{t_{n+1}}} \calR_{\bftheta}(\hatx_{t_{n+1}}),
        \end{aligned}
	\end{equation}
which is a convex combination of $\hatx_{t_{n+1}}$ and the generalized denoising output $\calR_{\bftheta}(\hatx_{t_{n+1}})=$
\begin{equation}
    r_{\bftheta}(\hatx_{t_{n+1}})- \sum_{i=2}^{k}\frac{1}{i!}\frac{\rmd^{(i)} \bfx_t}{\rmd \sigma_t^{(i)}}\Big|_{\hatx_{t_{n+1}}}\sigma_{t_{n+1}}(\sigma_{t_n} - \sigma_{t_{n+1}})^{i-1}.
\end{equation}
We have $\calR_{\bftheta}(\hatx_{t_{n+1}})=r_{\bftheta}(\hatx_{t_{n+1}})$ for Euler method ($k=1$), and $\calR_{\bftheta}(\hatx_{t_{n+1}})=r_{\bftheta}(\hatx_{t_{n+1}})+\frac{\sigma_{t_{n}}-\sigma_{t_{n+1}}}{2}\frac{\rmd r_{\bftheta}(\hatx_{t_{n+1}})}{\rmd \sigma_t}$ for second-order numerical methods ($k=2$). 
\end{proposition}
\begin{corollary}
    \label{cor:jump}
	The denoising output $r_{\bftheta}(\hatx_{t_{n+1}})$ reflects the prediction made by a single Euler step from $\hatx_{t_{n+1}}$ with the time step size $\sigma_{t_{n+1}}$.
\end{corollary}
\begin{corollary}
     Each previously proposed second-order ODE-based accelerated sampling method corresponds to a specific first-order finite difference of $\rmd r_{\bftheta}(\hatx_{t_{n+1}})/\rmd \sigma_t$. 
\end{corollary}
Proofs and more discussions are provided in Section~\ref{subsec:generalized_denoising_output}. The ratio $\sigma_{t_{n}}/\sigma_{t_{n+1}}$ in~\eqref{eq:convex} reflects our inclination to maintain the current position rather than transitioning to the generalized denoising output at $t_{n+1}$. In this context, different time-schedule functions, such as uniform, quadratic, and polynomial schedules \cite{song2021ddim,lu2022dpm,karras2022edm}, essentially represent various weighting functions. We primarily focus on the Euler method to simplify subsequent discussions, though these insights can be readily extended to higher-order methods. The behavior of the trajectory in the continuous scenario is similarly discernible by examining the sampling process with an infinitesimally small Euler step. A graphical representation of two successive Euler steps is depicted in Figure~\ref{fig:convex}.
\begin{figure}[t]
	\centering
    \includegraphics[width=0.95\columnwidth]{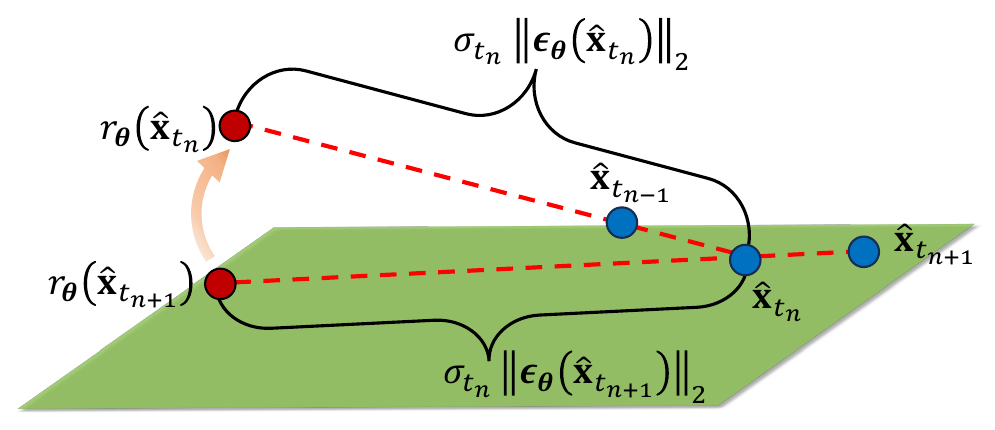}
    \caption{An illustration of two consecutive Euler steps, starting from a current sample $\hatx_{t_{n+1}}$. A single Euler step in the ODE-based sampling is a convex combination of the denoising output and the current position to determine the next position. Blue points form a piecewise linear sampling trajectory, while red points form the denoising trajectory governing the rotation direction.
    }
    \label{fig:convex}
\end{figure}

We further deduce that, approximately, each intermediate point $\hatx_{t_{n}}$, $n\in [1, N-1]$ in the piecewise linear sampling trajectory is determined by the selected time schedule, given that $\lVert r_{\bftheta}(\hatx_{t_{n+1}}) - \hatx_{t_{n}}\rVert_2 = (\sigma_{t_n}/\sigma_{t_{n+1}})\lVert r_{\bftheta}(\hatx_{t_{n+1}}) - \hatx_{t_{n+1}}\rVert_2=\sigma_{t_{n}}\lVert\bfeps_{\bftheta}(\hatx_{t_{n+1}})\rVert_2\approx\sigma_{t_{n}}\lVert\bfeps_{\bftheta}(\hatx_{t_{n}})\rVert_2=\lVert r_{\bftheta}(\hatx_{t_{n}}) - \hatx_{t_{n}}\rVert_2$. In this scenario, the denoising output $r_{\bftheta}(\hatx_{t_{n+1}})$ appears to be oscillating toward $r_{\bftheta}(\hatx_{t_{n}})$ around $\hatx_{t_n}$, akin to a simple gravity pendulum~\cite{young1996university}. The pendulum length effectively shortens by the coefficient $\sigma_{t_{n}}/\sigma_{t_{n+1}}$ in each sampling step, starting from roughly $\sigma_T\sqrt{d}$. This specific structure exits in all trajectories. Theoretical justification and empirical evidence are detailed in Section~\ref{subsubsec:rotation}. Practically, the magnitude of each oscillation is extremely small ($\approx 0^\circ$), and the entire sampling trajectory only marginally deviates from a 2-D plane. Such deviations can be further represented using a few orthogonal bases as discussed in Section~\ref{sec:trajectory_visualization}.
\subsection{Theoretical Analysis of the Trajectory Structure}
\label{subsec:meanshift}


Next we show that using a Gaussian kernel density estimate (KDE) with training data samples, the denoising trajectory has a closed-form solution. 

Given a dataset $\calD\coloneqq\{\bfy_{i}\in \mathbb{R}^d\}_{i\in\calI}$ containing $|\calI|$ i.i.d. data points drawn from the unknown data distribution $p_d$, the marginal density at each time of the forward diffusion process~\eqref{eq:new_sde} becomes a Gaussian KDE with a bandwidth $\sigma_t^2$, \ie, 
$\hat{p}_t(\bfx_t)=\int p_{0t}(\bfx_t|\bfy)\hat{p}_{d}(\bfy) = \frac{1}{|\calI|}\sum_{i\in\calI}\calN\left(\bfx_t;\bfy_i, \sigma_t^2\bfI\right),$
where the empirical data distribution $\hat{p}_{d}(\bfy)$ is a summation of multiple \textit{Dirac delta functions}, \ie, $\hat{p}_{d}(\bfy)=\frac{1}{|\calI|}\sum_{i\in\calI}\delta(\lVert \bfy - \bfy_i\rVert)$. 
In this case, the optimal denoising output of training a denoising autoencoder becomes a convex combination of original data points, $r_{\bftheta}^{\star}(\bfx_t; \sigma_t)=$
\begin{equation}
	\label{eq:optimal}
		\sum_{i} \frac{\exp \left(-\lVert \bfx_t - \bfy_i \rVert^2_2/2\sigma_t^2\right)}{\sum_{j}\exp \left(-\lVert \bfx_t - \bfy_j \rVert^2_2/2\sigma_t^2\right)} \bfy_i=\sum_{i} u_i \bfy_i,
\end{equation}
where each weight $u_i$ is calculated based on the time-scaled and normalized $L^2$ distance between the input $\bfx_t$ and $\bfy_i$ belonging to the dataset $\calD$, and 
$\sum_{i} u_i =1$. The proof is provided in Appendix~\ref{subsubsec:optimal_denoising_output}. The above equation appears to be highly similar to the iterative formula used in mean shift, which is a well-known non-parametric mode-seeking algorithm via iteratively gradient ascent with adaptive step sizes~\cite{fukunaga1975estimation,cheng1995mean,comaniciu2002mean,yamasaki2020mean}. In particular, the time-decreasing bandwidth ($\sigma_t\rightarrow 0$ as $t\rightarrow 0$) in~\eqref{eq:optimal} is strongly reminiscent of \textit{annealed mean shift}, or \textit{multi-bandwidth mean shift}~\cite{shen2005annealedms}, which was developed as a metaheuristic algorithm to escape local maxima, where classical mean shift is susceptible to stuck, by monotonically decreasing the bandwidth in iterations. 

Based on this connection, we characterize the \textit{local behavior} of diffusion sampling, \ie, each sampling trajectory monotonically converges in terms of the sample likelihood, and its coupled denoising trajectory always achieves higher likelihood (see the proof in Appendix~\ref{subsubsec:likelihood}). We also characterize the \textit{global behavior} of diffusion sampling as a linear-nonlinear-linear mode-seeking path. In the optimal case, the denoising output, or annealed mean vector, starts from a spurious mode (dataset mean), \ie, $r_{\bftheta}^{\star}(\bfx_t; \sigma_t)\approx \frac{1}{|\calI|}\sum_{i\in\calI}\bfy_i$ with a sufficiently large bandwidth $\sigma_t$. 
Meanwhile, the sampling trajectory is initially located in an approximately uni-modal Gaussian distribution with a \textit{linear} score function $\nablaxt \log p_t(\bfx_t)=\left(r_{\bftheta}^{\star}(\bfx_t; \sigma_t) - \bfx_t\right)/\sigma_t^2 \approx - \bfx_t/\sigma_t^2$. The approximation is valid as the norm of dataset mean is very close to zero compared with the norm of $\bfx_t$ due to the concentration of measure (see Lemma~\ref{lemma:concentration}). 
As $\sigma_t$ monotonically decreases in the sampling process, the mode numbers of $\hat{p}_t(\bfx_t)$ increase~\cite {silverman1981using}, and the simple distribution surface gradually shifts to the complex multi-modal one. The score function appears \textit{nonlinear} in the medium $\sigma_t$ stage since multiple data points contribute a non-negligible effect. Finally, the sampling trajectory is attracted by a certain real-data mode with a sufficiently small bandwidth $\sigma_t$, and the score function appears \textit{linear} again, \ie, 
$\nablaxt \log p_t(\bfx_t)\approx (\bfy_{k}-\bfx_t)/\sigma_t^2$, where $\bfy_k$ denotes the nearest data point to $\bfx_t$. 
In this way, the global mode is hopefully sought by the sampling trajectory, as annealed mean shift~\cite{shen2005annealedms}. 
Intriguingly, we can provide a guarantee that the whole sampling trajectory length is around $\sigma_T\sqrt{d}$ (see Section~\ref{subsubsec:rotation}). 
The above analysis implies that the optimal sampling trajectories simply replay the dataset, but in practice, a slight score deviation ensures the generative ability of diffusion models (see Section~\ref{subsec:score_deviation}).
\section{Geometry-Inspired Time Scheduling}
\label{sec:dp}

An ODE-based numerical solver such as Euler~\cite{song2021sde} or Runge-Kutta~\cite{liu2022pseudo,zhang2023deis} relies on a pre-defined time schedule $\Gamma=\{t_0=\epsilon, \cdots, t_N=T\}$ in the sampling process. Typically, given the initial time $t_N$ and the final time $t_0$, the intermediate time steps $t_1$ to $t_{N-1}$ are determined by heuristic strategies such as uniform, quadratic~\cite{song2021ddim}, log-SNR~\cite{lu2022dpm,lu2022dpmpp}, and polynomial functions~\cite{karras2022edm,song2023consistency}. In fact, the time schedule reflects our prior knowledge of the sampling trajectory shape. Under the constraint of the total {\em number of score function evaluations} (NFEs), an improved time schedule can reduce the local truncation error in each numerical step, and hopefully minimize the global truncation error. In this way, the sample quality generated by numerical methods could approach that of the exact solutions of the given empirical PF-ODE~\eqref{eq:epf_ode}. 

Our previous discussions in Section~\ref{sec:trajectory_visualization} identified each sampling trajectory as a simple low-dimensional ``boomerang'' curve. We can thus leverage this geometric structure to re-allocate the intermediate timestamps according to the principle that assigning a larger time step size when the trajectory exhibits a relatively small curvature while assigning a smaller time step size when the trajectory exhibits a relatively large curvature. Additionally, we have demonstrated that different trajectories share almost the same shape, which helps us estimate the common structure of the sampling trajectory by using just a few ``warmup'' samples. We name our approach to achieve the above goal as \textit{geometry-inspired time scheduling} (GITS) and elaborate the details as follows. We will also show that GITS is adaptable, easy to implement, and introduces negligible computing overheads.

The allocation of the intermediate timestamps can be formulated as an integer programming problem and solved by dynamic programming to search for optimal time scheduling~\cite{cormen2022introduction}. We first define a searching space denoted as $\Gamma_{g}$, which is a fine-grained grid including all possible intermediate timestamps. Then, we measure the trajectory curvature by the local truncation errors. Specifically, we define the cost from the current position $\bfx_{t_i}$ to the next position $\bfx_{t_j}$ as the difference between an Euler step and the ground-truth prediction, \ie, $c_{t_i\rightarrow t_j}\coloneqq\calD(\hatx_{t_i\rightarrow t_j}, \bfx_{t_i\rightarrow t_j})$, where $t_i$ and $t_j$ are two intermediate timestamps from $\Gamma_{g}$ and $t_i>t_j$. According to the empirical PF-ODE~\eqref{eq:epf_ode}, the ground-truth prediction is calculated as
$\bfx_{t_i\rightarrow t_j}=\int_{t_i}^{t_j}\bfx_{t_i}+\bfeps_{\bftheta}(\bfx_t)\sigma^{\prime}_t \rmd t$, and the Euler prediction is calculated as
$\hatx_{t_i\rightarrow t_j}=\bfx_{t_i}+(\sigma_{t_j}-\sigma_{t_i})\bfeps_{\bftheta}(\hatx_{t_i})\sigma^{\prime}_{t_i}$. The cost function $\calD$ can be defined as the $L^2$ distance in the original pixel space, \ie, $\calD(\bfx, \bfy)=\lVert \bfx - \bfy\rVert_2$, or any other user-specified metric. Given all computed pair-wise costs, which form a cost matrix, this becomes a standard minimum-cost path problem and can be solved with dynamic programming. The algorithm and details are provided in Appendix~\ref{subsubsec:dp}. 

Still, there also exist many different ways to determine the time schedule (\eg, using a trainable neural network) by leveraging our discovered trajectory regularity. 

\renewcommand{\arraystretch}{0.8}
\begin{table}[t!]
    \caption{Sample quality comparison in terms of Fr\'echet Inception Distance (FID~\cite{heusel2017gans}, lower is better) on four datasets (resolutions ranging from $32\times32$ to $256\times256$). $\dagger$: Results reported by authors. More results are provided in Table~\ref{tab:fid-full}. 
    }
    \label{tab:fid-1}
    \vskip 0.1in
    \centering
    \fontsize{7}{10}\selectfont
    \begin{tabular}{lcccc}
        \toprule
        \multirow{2}{*}{METHOD} & \multicolumn{4}{c}{NFE} \\
        \cmidrule{2-5}
        & 5 & 6 & 8 & 10 \\
        \midrule
        \multicolumn{5}{l}{\fontsize{7}{1}\selectfont \textbf{CIFAR-10 32$\times$32}~\cite{krizhevsky2009learning}} \\
        \midrule
        DDIM~\cite{song2021ddim}                        & 49.66 & 35.62 & 22.32 & 15.69 \\
        \rowcolor[gray]{0.9} DDIM + \ourName (\textbf{ours})           & 28.05 & 21.04 & 13.30 & 10.37 \\
        DPM-Solver-2~\cite{lu2022dpm}                   & -     & 60.00 & 10.30 & 5.01  \\
        DPM-Solver++(3M)~\cite{lu2022dpmpp}             & 24.97 & 11.99 & 4.54  & 3.00  \\
        DEIS-tAB3~\cite{zhang2023deis} & 14.39 & 9.40 & 5.55 & 4.09 \\
        UniPC~\cite{zhao2023unipc}                      & 23.98 & 11.14 & 3.99  & 2.89  \\
        AMED-Solver~\cite{zhou2024fast}                 & -     & 7.04  & 5.56  & 4.14  \\
        AMED-Plugin~\cite{zhou2024fast}                 & -     & 6.67  & 3.34  & \textbf{2.48}  \\
        iPNDM~\cite{zhang2023deis}                      & 13.59 & 7.05  & 3.69  & 2.77  \\
        \rowcolor[gray]{0.9} iPNDM + \ourName (\textbf{ours})          & \textbf{8.38} & \textbf{4.88} & \textbf{3.24} & \textbf{2.49} \\
        \midrule
        \multicolumn{5}{l}{\fontsize{7}{1}\selectfont \textbf{FFHQ 64$\times$64}~\cite{karras2019style}} \\
        \midrule
        DDIM~\cite{song2021ddim}                        & 43.93 & 35.22 & 24.39 & 18.37 \\
        \rowcolor[gray]{0.9} DDIM + \ourName (\textbf{ours})           & 29.80 & 23.67 & 16.60 & 13.06 \\
        DPM-Solver-2~\cite{lu2022dpm}                   & -     & 83.17 & 22.84 & 9.46  \\
        DPM-Solver++(3M)~\cite{lu2022dpmpp}             & 22.51 & 13.74 & 6.04  & 4.12  \\
        DEIS-tAB3~\cite{zhang2023deis} & 17.36 & 12.25 & 7.59 & 5.56 \\
        UniPC~\cite{zhao2023unipc}                      & 21.40 & 12.85 & 5.50  & 3.84  \\
        AMED-Solver~\cite{zhou2024fast}                 & -     & 10.28 & 6.90  & 5.49  \\
        AMED-Plugin~\cite{zhou2024fast}                 & -     & 9.54  & 5.28  & 3.66  \\    
        iPNDM~\cite{zhang2023deis}                      & 17.17 & 10.03 & 5.52  & 3.98  \\
        \rowcolor[gray]{0.9} iPNDM + \ourName (\textbf{ours})          & \textbf{11.22} & \textbf{7.00} & \textbf{4.52} & \textbf{3.62} \\
        \midrule
        \multicolumn{5}{l}{\fontsize{7}{1}\selectfont \textbf{ImageNet 64$\times$64}~\cite{russakovsky2015ImageNet}} \\
        \midrule
        DDIM~\cite{song2021ddim}                        & 43.81 & 34.03 & 22.59 & 16.72 \\
        \rowcolor[gray]{0.9} DDIM + \ourName (\textbf{ours})           & 24.92 & 19.54 & 13.79 & 10.83 \\
        DPM-Solver-2~\cite{lu2022dpm}                   & -     & 44.83 & 12.42 & 6.84  \\
        DPM-Solver++(3M)~\cite{lu2022dpmpp}             & 25.49 & 15.06 & 7.84  & 5.67  \\
        DEIS-tAB3~\cite{zhang2023deis} & 14.75 & 12.57 & 6.84 & 5.34 \\
        UniPC~\cite{zhao2023unipc}                      & 24.36 & 14.30 & 7.52  & 5.53  \\
        RES(M)$\dagger$~\cite{zhang2023improved}           & 25.10    & 14.32 & 7.44  & 5.12  \\
        AMED-Solver~\cite{zhou2024fast}                 & -     & 10.63 & 7.71  & 6.06  \\
        AMED-Plugin~\cite{zhou2024fast}                 & -     & 12.05 & 7.03  & 5.01 \\
        iPNDM~\cite{zhang2023deis}                      & 18.99 & 12.92 & 7.20  & 5.11  \\
        \rowcolor[gray]{0.9} iPNDM + \ourName (\textbf{ours})          & \textbf{10.79} & \textbf{8.43} & \textbf{5.82} & \textbf{4.48} \\
        \midrule
        \multicolumn{5}{l}{\fontsize{7}{1}\selectfont \textbf{LSUN Bedroom 256$\times$256}~\cite{yu2015lsun} (pixel-space)} \\
        \midrule
        DDIM~\cite{song2021ddim}                        & 34.34 & 25.25 & 15.71 & 11.42 \\
        \rowcolor[gray]{0.9} DDIM + \ourName (\textbf{ours})           & 22.04 & 16.54 & 11.20 & 9.04  \\
        DPM-Solver-2~\cite{lu2022dpm}                   & -     & 80.59 & 23.26 & 9.61  \\
        DPM-Solver++(3M)~\cite{lu2022dpmpp}             & 23.15 & 12.28 & 7.44  & 5.71  \\
        UniPC~\cite{zhao2023unipc}                      & 23.34 & 11.71 & 7.53  & 5.75  \\
        AMED-Solver~\cite{zhou2024fast}                 & -     & 12.75 & \textbf{6.95}  & 5.38  \\
        AMED-Plugin~\cite{zhou2024fast}                 & -     & 11.58 & 7.48  & 5.70  \\
        iPNDM~\cite{zhang2023deis}                      & 26.65 & 20.73 & 11.78 & 5.57  \\
        \rowcolor[gray]{0.9} iPNDM + \ourName (\textbf{ours})          & \textbf{15.85} & \textbf{10.41} & \textbf{7.31} & \textbf{5.28} \\
        \bottomrule
    \end{tabular}
    \vspace{-1em}
\end{table}

\renewcommand{\arraystretch}{0.8}
\begin{table}[t]
    \caption{Image generation on Stable Diffusion v1.4~\cite{rombach2022ldm} with the default classifier-free guidance scale 7.5 (one sampling step requires two NFEs). We follow the standard FID evaluation, and use the statistics and 30k sampled captions from the MS-COCO~\cite{lin2014microsoft} validation set provided \href{https://github.com/boomb0om/text2image-benchmark}{here}.}
    \label{tab:fid-2}
    \vskip 0.1in
    \centering
    \fontsize{7}{10}\selectfont
    \begin{tabular}{lcccc}
        \toprule
        \multirow{2}{*}{METHOD} & \multicolumn{4}{c}{NFE} \\
        \cmidrule{2-5}
        & 10 & 12 & 14 & 16 \\
        \midrule
        DPM-Solver++(2M)~\cite{lu2022dpmpp}                 & 17.16 & 15.76 & 15.06 & 14.72 \\
        \rowcolor[gray]{0.9} DPM-Solver++(2M) + \ourName (\textbf{ours})   & \textbf{15.53} & \textbf{13.29} & \textbf{12.44} & \textbf{12.26} \\
        \bottomrule
    \end{tabular}
	\vspace{-1em}
\end{table}

\section{Experiments}

We adhere to the setup and experimental designs of the EDM framework~\cite{karras2022edm,song2023consistency}, with $f(t)=0$, $g(t)=\sqrt{2t}$, and $\sigma_t=t$. Under this parameterization, the forward VE-SDE is expressed as $\rmdx_t =\sqrt{2t} \,  \rmd \bfw_t$, while the corresponding empirical PF-ODE is formulated as $\rmdx_t = \frac{\bfx_t - r_{\bftheta}\left(\bfx_t; t\right)}{t}\rmd t$. The temporal domain is segmented using a polynomial function $t_n=(t_0^{1/\rho}+\frac{n}{N}(t_N^{1/\rho}-t_0^{1/\rho}))^{\rho}$, where $t_0=0.002$, $t_N=80$, $n\in [0,N]$, and $\rho=7$. In the visual analysis about 1-D projections (Figure~\ref{fig:deviation}) and Multi-D projections (Figure~\ref{fig:traj_3d}) detailed in Section~\ref{sec:trajectory_visualization}, we simulate each sampling trajectory employing the Euler method with $100$ score function evaluations. The mean and standard deviation for trajectory deviation showcased in Figure~\ref{fig:deviation} are derived from $5,000$ synthesized samples on ImageNet $64\times 64$, and the ratios of explained variance for PCA presented in Figure~\ref{fig:recon_pca} are based on 1,000 synthesized samples. 

We initiate the dynamic programming experiments with $256$ ``warmup'' samples randomly selected from Gaussian noise to create a more refined grid, and then calculate the cost matrix. The number of ``warmup'' samples is not a critical hyper-parameter, but reducing it generally increases the variance, as shown in the Table~\ref{tab:warmup_size}. 
Due to subtle differences exist among sampling trajectories (see Figure~\ref{fig:traj_3d}), we recommend utilizing a reasonable number of ``warmup'' samples to determine the optimal time schedule, such that this time schedule hopefully works well for all the generated samples. 
The ground-truth predictions are generated by iPNDM~\cite{zhang2023deis} (employing a fourth-order multistep Runge-Kutta method with a lower-order warming start) using the polynomial time schedule specified in EDM~\cite{karras2022edm} with 60 NFEs, resulting in a grid size $|\Gamma_g|=61$. Additional findings are available in Appendix~\ref{subsec:appendix_results}. 
The reported results of all compared approaches are obtained from an open-source toolbox: \url{https://github.com/zju-pi/diff-sampler}.

\textbf{Image generation.} As shown in Tables~\ref{tab:fid-1}-\ref{tab:fid-2}, our simple time re-allocation strategy based on iPNDM~\cite{zhang2023deis} consistently beats all existing ODE-based accelerated sampling methods, with a significant margin especially in the few NFE cases. In particular, all time schedules in these datasets are searched based on the Euler method, \ie, DDIM~\cite{song2021ddim}, but they are directly applicable with high-order methods such as iPNDM~\cite{zhang2023deis}. The trajectory regularity we uncovered guarantees that the schedule determined through 256 ``warmup'' samples is effective across all generated content. Furthermore, the experimental outcomes suggest that identifying this trajectory regularity enhances our comprehension of diffusion models' mechanisms. This understanding opens avenues for developing tailored time schedules for more efficient diffusion sampling.
Note that we did not use the analytical first step (AFS) that replaces the first numerical step with an analytical Gaussian score to save one NFE, proposed in~\cite{dockhorn2022genie} and later used in~\cite{zhou2024fast}, as we found this trick is particularly effective for datasets containing the small-resolution images. DPM-Solver-2~\cite{lu2022dpm} and AMED-Solver/Plugin~\cite{zhou2024fast} are thus inapplicable with NFE$=5$ (marked as ``-'') in Table~\ref{tab:fid-1}. Ablation studies on AFS and the grid size for dynamic programming are provided in Appendix~\ref{subsec:appendix_results}. 

\textbf{Time schedule comparison.} From Table~\ref{tab:time_schedule}, we can see that compared with existing handcraft time schedules, our used schedule is expected to better fit the underlying trajectory structure in the sampling of diffusion models and achieves smaller truncation errors with improved sample quality.

\textbf{Running time.} Our strategy incurs a very low cost and does not require accessing the real dataset. We start by a few initial ``warmup'' samples and run the given ODE-solver with fine-grained and coarse-grained steps to compute the cost matrix for dynamic programming. \textit{Such a computation is performed only once on each dataset to simultaneously find optimal time schedules for different NFEs}, thanks to the optimal substructure property~\cite{cormen2022introduction}. This requires less than or about a minute on CIFAR-10, FFHQ, ImageNet $64\times64$, and 10 to 15 minutes for LSUN Bedroom and LAION (Stable Diffusion), as shown in Table~\ref{tab:consumed_time}. 


\renewcommand{\arraystretch}{0.8}
\begin{table}[t]
    \caption{The comparison of FID results on CIFAR-10. Time schedules considerably affect the image generation performance.
    }
    \label{tab:time_schedule}
    \vskip 0.1in
    \centering
    \fontsize{7}{10}\selectfont
    \begin{tabular}{lcccc}
        \toprule
        \multirow{2}{*}{TIME SCHEDULE} & \multicolumn{4}{c}{NFE} \\
        \cmidrule{2-5}
        & 5 & 6 & 8 & 10 \\
        \midrule
        DDIM-uniform                & 36.98 & 28.22 & 19.60 & 15.45 \\
        DDIM-logsnr                 & 53.53 & 38.20 & 24.06 & 16.43 \\
        DDIM-polynomial             & 49.66 & 35.62 & 22.32 & 15.69 \\
        \rowcolor[gray]{0.9} DDIM + \ourName (\textbf{ours})           & \textbf{28.05} & \textbf{21.04} & \textbf{13.30} & \textbf{10.09} \\
        \midrule
        iPNDM-uniform               & 17.34 & 9.75  & 7.56  & 7.35  \\
        iPNDM-logsnr                & 19.87 & 10.68 & 4.74  & 2.94  \\
        iPNDM-polynomial            & 13.59 & 7.05  & 3.69  & 2.77  \\
        \rowcolor[gray]{0.9} iPNDM + \ourName (\textbf{ours})          & \textbf{8.38} & \textbf{4.88} & \textbf{3.24} & \textbf{2.49} \\
        \bottomrule
    \end{tabular}
    \vspace{-1em}
\end{table}

\renewcommand{\arraystretch}{0.8}
\begin{table}[t]
    \caption{Used time (seconds) in different stages of \ourName, evaluated on an NVIDIA A100 GPU. ``warmup'' samples are generated by 60 NFE and the NFE budget for dynamic programming is 10.}
    \label{tab:consumed_time}
    \vskip 0.1in
    \centering
    \fontsize{7}{10}\selectfont
    \begin{tabular}{lcccc}
        \toprule
        \multirow{2}{*}{DATASET} & sample & cost & dynamic & total \\
        & generation & matrix & programming & time (s) \\
        \midrule
        CIFAR-10 $32\times 32$            & 27.47  & 5.29   & 0.015  & 32.78  \\
        FFHQ $64\times 64$            & 51.90  & 10.88  & 0.016  & 62.79  \\
        ImageNet $64\times 64$        & 71.77  & 13.28  & 0.016  & 85.07  \\
        LSUN Bedroom        & 517.63 & 122.13 & 0.015  & 639.78 \\
        LAION (sd-v1.4)     & 877.62 & 24.00  & 0.016  & 901.62 \\
        \bottomrule
    \end{tabular}
\end{table}

\begin{figure}[t]
    \centering
    \begin{subfigure}[t]{0.46\columnwidth}
        \centering
        \includegraphics[width=\columnwidth]{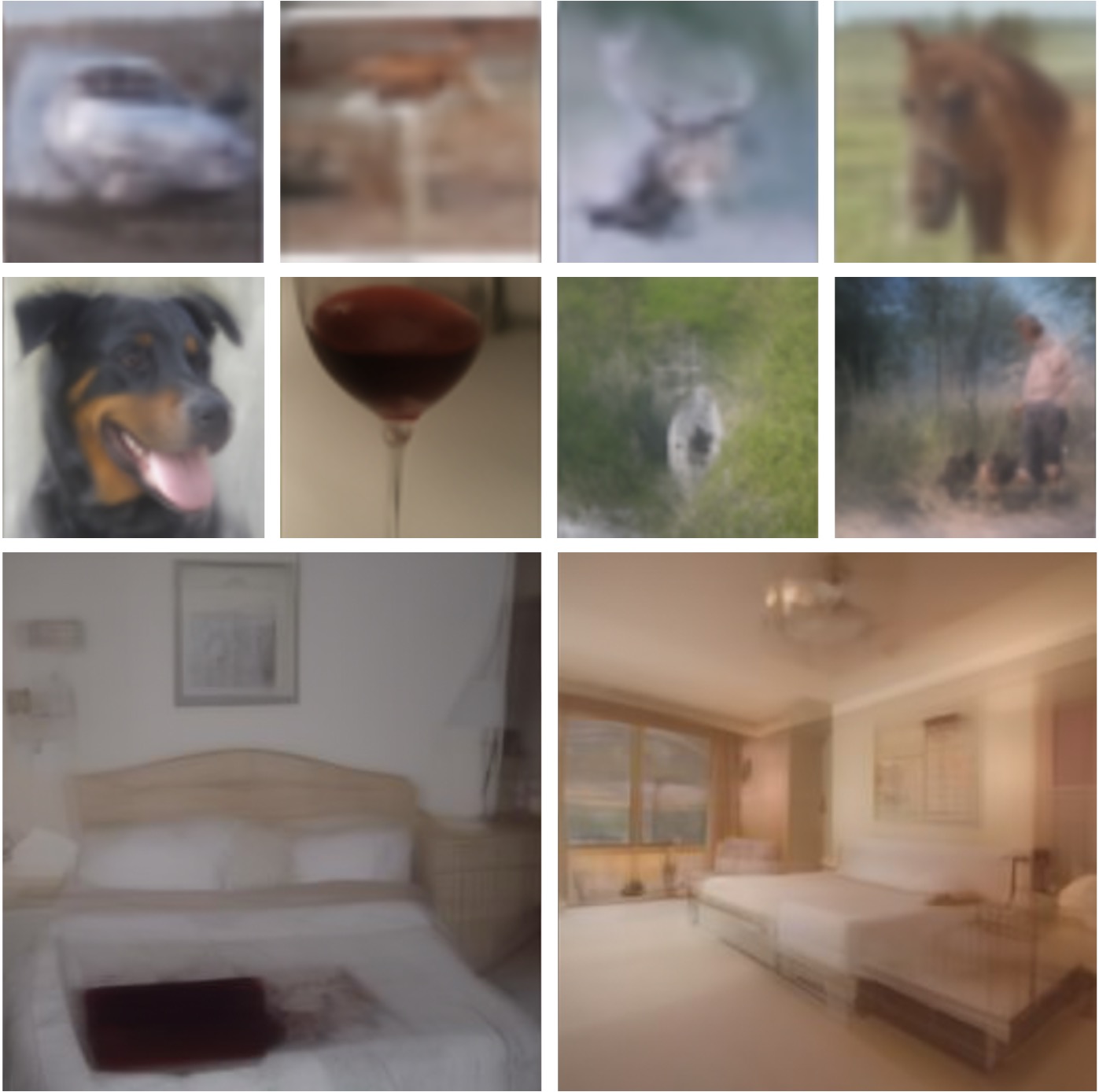}
        \caption{DDIM, NFE = 5.}
    \end{subfigure}
    \quad
    \begin{subfigure}[t]{0.46\columnwidth}
        \centering
        \includegraphics[width=\columnwidth]{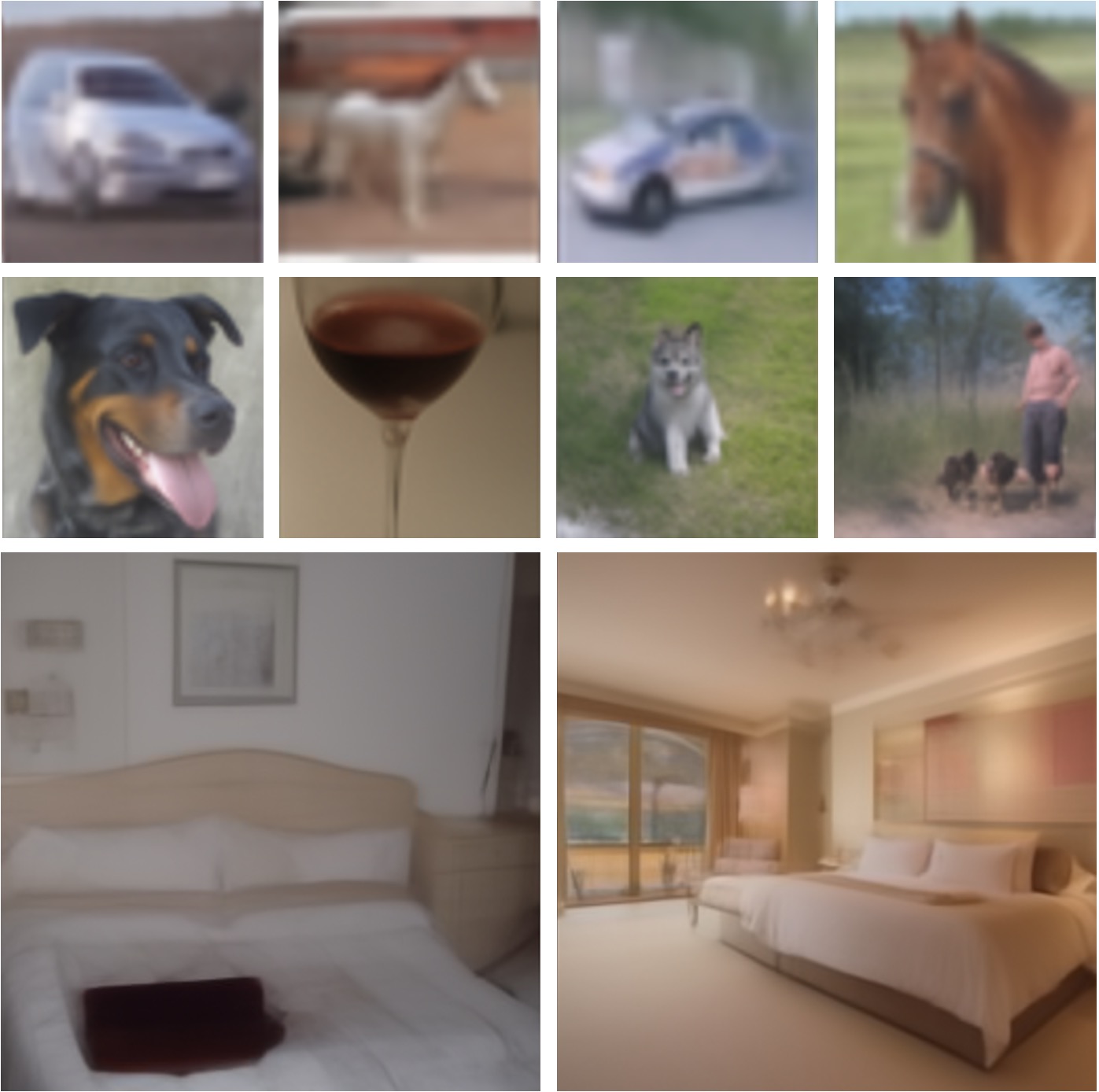}
        \caption{DDIM + \ourName, NFE = 5.}
    \end{subfigure}
    \caption{The visual comparison of samples generated by DDIM and DDIM + \ourName (1st row: CIFAR-10, 2nd row: ImageNet $64\times 64$, 3rd row: LSUN Bedroom). See more results in Appendix~\ref{subsec:appendix_results}.}
    \label{fig:euler_final}
\end{figure}

\begin{figure}[t]
    \centering
    \begin{subfigure}[t]{0.31\columnwidth}
        \centering
        \includegraphics[width=\columnwidth]{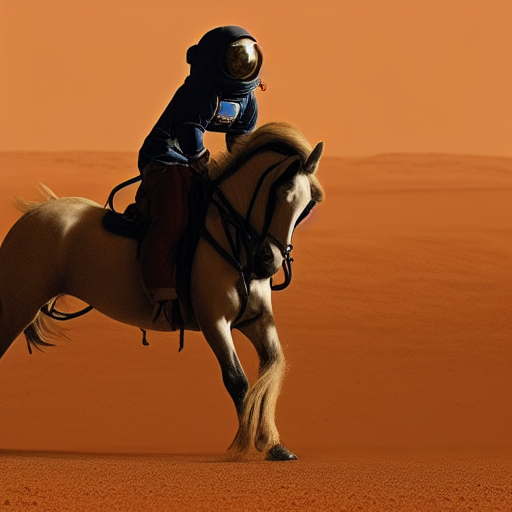}
    \end{subfigure}
    \begin{subfigure}[t]{0.31\columnwidth}
        \centering
        \includegraphics[width=\columnwidth]{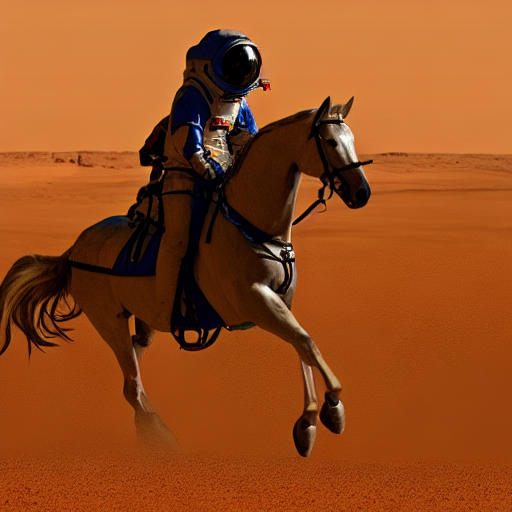}
    \end{subfigure}
    \begin{subfigure}[t]{0.31\columnwidth}
        \centering
        \includegraphics[width=\columnwidth]{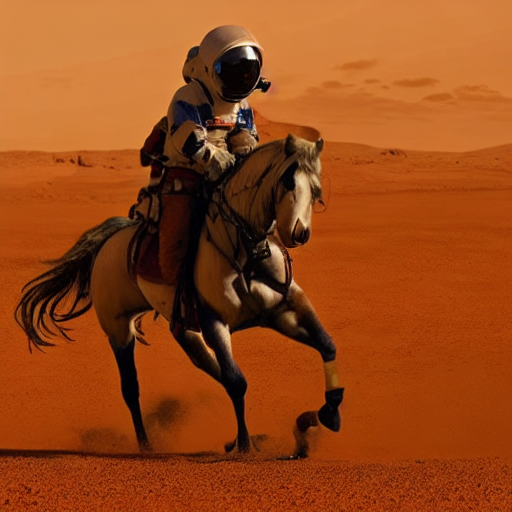}
    \end{subfigure}
    
    \vspace{0.15em}
    
    \begin{subfigure}[t]{0.31\columnwidth}
        \centering
        \includegraphics[width=\columnwidth]{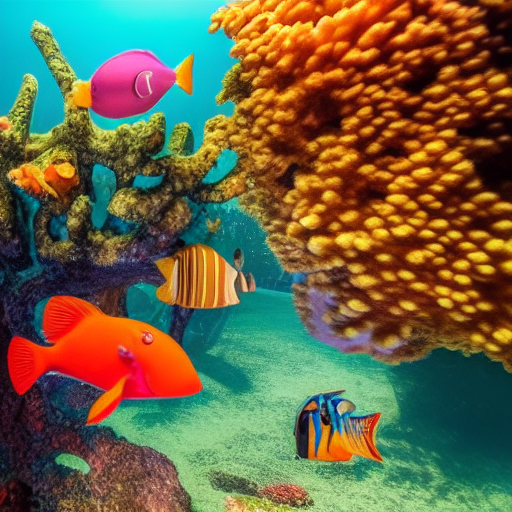}
    \end{subfigure}
    \begin{subfigure}[t]{0.31\columnwidth}
        \centering
        \includegraphics[width=\columnwidth]{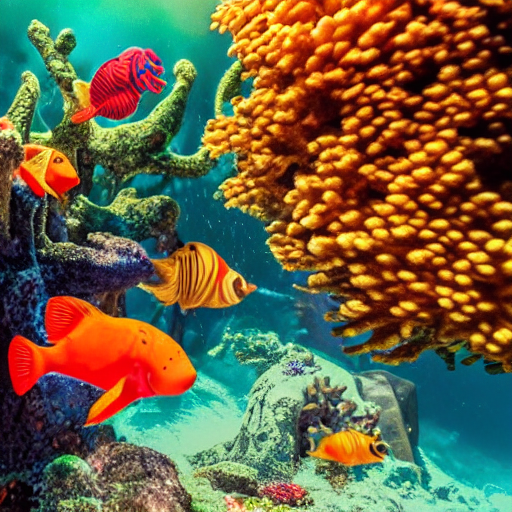}
    \end{subfigure}
    \begin{subfigure}[t]{0.31\columnwidth}
        \centering
        \includegraphics[width=\columnwidth]{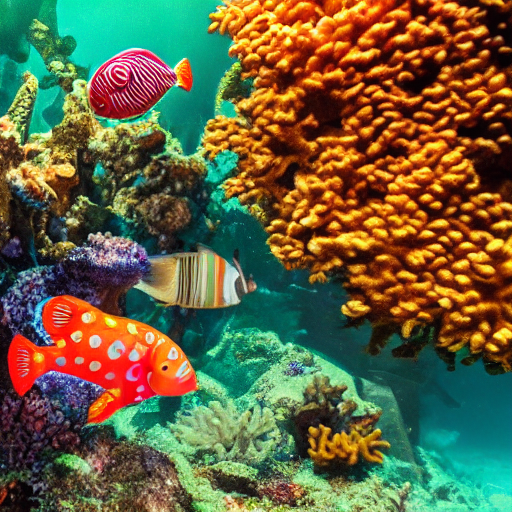}
    \end{subfigure}
    
    \vspace{0.15em}
    
    \begin{subfigure}[t]{0.31\columnwidth}
        \centering
        \includegraphics[width=\columnwidth]{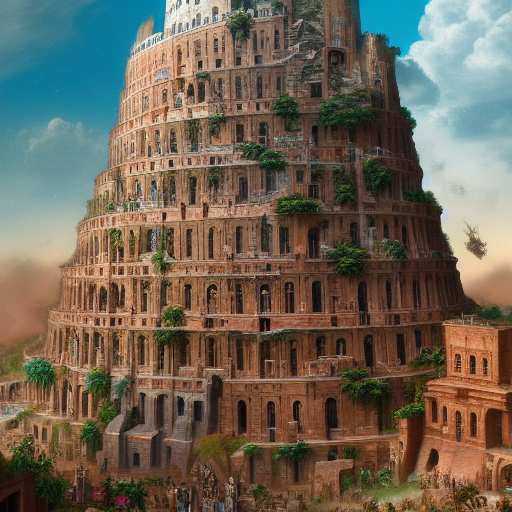}
        \caption{Uniform.}
    \end{subfigure}
    \begin{subfigure}[t]{0.31\columnwidth}
        \centering
        \includegraphics[width=\columnwidth]{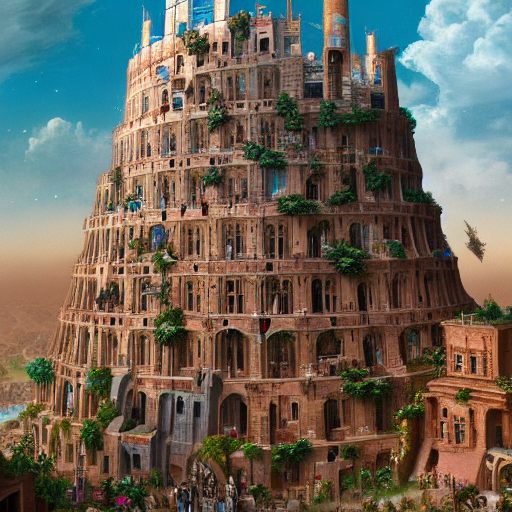}
        \caption{AYS.}
    \end{subfigure}
    \begin{subfigure}[t]{0.31\columnwidth}
        \centering
        \includegraphics[width=\columnwidth]{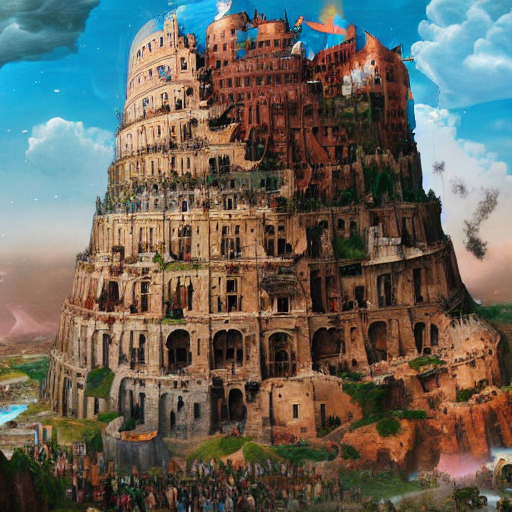}
        \caption{\ourName.}
    \end{subfigure}
    \caption{The visual comparison of samples generated by Stable Diffusion 1.5 with DPM-Solver++(2M), using the uniform, AYS-optimized~\cite{sabour2024align} or \ourName-optimized time schedule and 10 steps. The text prompts are ``a photo of an astronaut riding a horse on mars'' (1st row); ``a whimsical underwater world inhabited by colorful sea creatures and coral reefs'' (2nd row); ``a digital illustration of the Babel  tower 4k detailed trending in artstation fantasy vivid colors'' (3rd row).}
    \label{fig:sd_ays}
\end{figure}

\section{Related Work and Discussions}
\label{sec:related}

The popular variance-exploding (VE) SDEs~\cite{song2019ncsn,song2021sde} are taken as our main examples for analysis, which are equivalent to their variance-preserving (VP) counterparts according to It\^{o}'s lemma (see Remark~\ref{remark:ito} and Appendix~\ref{subsec:equivalence}). The equivalence has been established in their corresponding PF-ODE (rather than SDE) forms by using the change-of-variable formula~(see Proposition 1 of \citet{song2021ddim} and Proposition 3 of \citet{zhang2023deis}).~\citet{karras2022edm} also presented a set of steps to express different specific models in a common framework. 

Instead of training a noise-conditional score model with denoising score matching~\cite{vincent2011dsm,song2019ncsn,song2021maximum} or training a noise-prediction model to estimate the added noise in each step~\cite{ho2020ddpm,song2021ddim,nichol2021improved,vahdat2021lsgm,bao2022analytic}, we follow~\cite{kingma2021vdm,karras2022edm} and train a denoising model that predicts the reconstructed data from its corrupted version. With the help of simplified empirical PF-ODE~\eqref{eq:epf_ode}, we could characterize an implicit denoising trajectory and draw inspiration from classical non-parametric mean shift~\cite{fukunaga1975estimation,cheng1995mean,comaniciu2002mean}.

Denoising trajectory has been observed since the renaissance of diffusion models (see Figure 6 of~\cite{ho2020ddpm}) and later in Figure 3 of~\cite{kwon2023diffusion}, but did not been investigated, perhaps due to the indirect model parameterization. \citet{karras2022edm} first stated that the denoising output reflects the tangent of the sampling trajectory as our Corollary~\ref{cor:jump}, but neither characterizes the denoising trajectory in differential equations nor discusses how it controls the sampling trajectory. In fact, \citet{karras2022edm} mentioned this property to argue the sampling trajectory of \eqref{eq:epf_ode} is approximately linear due to the slow change in denoising output, and verified it in a 1D toy example. In contrast, we provide an in-depth analysis of the high-dimensional trajectories with real data and highlight a strong regularity.

We then describe one potential application to accelerate the sampling process. Different from most existing methods focusing on develop better ODE-solvers~\cite{song2021ddim,karras2022edm,liu2022pseudo,lu2022dpm,zhang2023deis,zhao2023unipc} while selecting a time schedule in a handcraft or trial-and-error way, we leverage trajectory regularity to better allocate the discretized time steps. Our approach is extremely faster than those distillation-based accelerating sampling methods~\cite{luhman2021knowledge,salimans2022progressive,zheng2023fast,song2023consistency} by several orders of magnitude. Besides, \citet{watson2021learning} used dynamic programming to optimize the time schedule based on the decomposable nature of ELBO but even worsen the sample quality. 
There are also many theoretical studies on the convergence analysis, and score estimation of diffusion models, but none of them focus on the trajectory properties~\cite{de2022convergence,pidstrigach2022score,lee2023convergence,chen2023approximation,chen2023restoration}.

Recently, a concurrent work named AYS was proposed to optimize time schedules in sampling by minimizing the mismatch between solving the true backward SDE and its approximated linear counterpart, based on techniques from stochastic calculus~\cite{sabour2024align}. In contrast, our \ourName leverages the strong trajectory regularity in diffusion models and yields time schedules with dynamic programming using only a few number of ``warmup'' samples. Our method also gets rid of the time-consuming Monte-Carlo computation in AYS~\cite{sabour2024align} and therefore is faster by several orders of magnitude. In Figure~\ref{fig:sd_ays}, we compare samples generated by different time schedules, using the colab code from \url{https://research.nvidia.com/labs/toronto-ai/AlignYourSteps/} with the default setting. We also evaluate their FID performance as the procedure used in Table~\ref{tab:fid-2}, and the results are 14.28 (uniform), 12.48 (AYS), and 12.02 (\ourName). 

\section{Conclusion}

In this paper, we illustrate the trajectory regularity that consistently appears in the ODE-based diffusion sampling, regardless of the specific content generated. We explain this regularity by characterizing and analyzing the implicit denoising trajectory, especially its behavior, under kernel density estimation-based data modeling. Such insights about the trajectory structure of diffusion-based generative models can lead to an accelerated sampling method to improve image synthesis quality with negligible computing overheads. 

For future works, we would like to explore deeper shape regularities in the sampling trajectories, describing more precise geometric structures of the regular shapes, and identifying new applications of these insights.

\section*{Impact Statement}

This work examines the theoretical properties of sampling trajectories in diffusion-based generative models. These models can produce images, audio, and videos indistinguishable from humans, raising concerns about their potential use in disseminating misinformation. We are fully aware of these negative impacts and, in subsequent work, will devote ourselves to developing mitigation measures, including detecting and watermarking synthetic contents.



\section*{Acknowledgements}

Defang Chen would like to thank Ziying Guo for helpful discussions on the sampling trajectory length and drawing the illustrations (Figures~\ref{fig:model}, \ref{fig:convex} and~\ref{fig:meanshift}), and thank Yuanzhi Zhu for helpful suggestions on the early version of this work, and thank Haowei Zheng for helpful discussions on the singularities of diffusion models and Remark~\ref{remark:eps_norm}.



\bibliography{ref}
\bibliographystyle{icml2024}

\newpage
\appendix
\onecolumn

\section{Appendix}

\subsection{The Equivalence of Linear Diffusion Models}
\label{subsec:equivalence}
In this section, we show that various linear diffusion models sharing the same signal-to-noise ratio (SNR) are closely connected with It\^{o}'s lemma. In particular, all other model types (\eg, variance-preserving (VP) diffusion process used in DDPMs \citep{ho2020ddpm}) can be transformed into the variance-exploding (VE) counterparts. Therefore, we merely focus on the mathematical properties and geometric behaviors of VE-SDEs in the main text to simplify our discussions. 

\subsubsection{General Notations of Linear Diffusion Models}

We first recap the basic notations in diffusion models. We consider the data perturbation as a continuous stochastic process $\{\bfz_t\}_{t=0}^T$ starting from $\bfz_0\sim p_d$, which is the solution of a linear stochastic differential equation (SDE)~\cite{song2021sde}:
\begin{equation}
    \label{eq:app_forward_sde}
    \rmd \bfz_t = f(t)\bfz_t\rmd t + g(t) \rmd \bfw_t, \quad f(\cdot): \bbR \rightarrow \bbR, \quad g(\cdot): \bbR \rightarrow \bbR,
\end{equation}
where $\bfw_t$ denotes the standard Wiener process; $f(\cdot)$ and $g(\cdot)$ are drift and diffusion coefficients, respectively. These two coefficients are required to be globally Lipschitz \wrt time to ensure the SDE has a unique strong solution~\cite{oksendal2013stochastic}. The probability density function $p_t(\bfz_t)$, starting from the initial condition $p_0(\bfz_0)=p_d(\bfz_0)$, evolves according to the well-known Fokker-Planck equation $\frac{\partial p_t(\bfz_t)}{\partial t} = - \nabla \cdot \left[p_t(\bfz_t)f(t)\bfz_t -\frac{g^2(t)}{2}\nablazt p_t(\bfz_t)\right]$. In this case, the transition kernel has the following form~\cite{sarkka2019applied,karras2022edm}
\begin{equation}
	\label{eq:app_kernel}
    p_{0t}(\bfz_t | \bfz_0)= \calN\left(\bfz_t ; s(t)\bfz_0, s^2(t)\sigma^2(t)\bfI\right),  
    \quad s(t) = \exp \left(\int_{0}^{t} f(\xi) \rmd \xi\right), 
    \quad \text{and} \quad  \sigma(t) = \sqrt{\int_{0}^{t} \left[g(\xi)/s(\xi)\right]^2\rmd \xi}.
\end{equation}
We denote $s(t)$ and $\sigma(t)$ as $s_t$ and $\sigma_t$ respectively hereafter for notation simplicity.
\begin{remark}
    An important implication from \eqref{eq:app_kernel} is that different linear diffusion processes sharing the same $\sigma_t$ actually have the same signal-to-noise ratio (SNR), since SNR is defined as $s^2_t/\left[s_t\sigma_t\right]^2=1/\sigma^2_t$. 
\end{remark}
Therefore, we can equivalently rewrite the forward SDE~\eqref{eq:app_forward_sde} in terms of $s_t$ and $\sigma_t$ as follows
\begin{equation}
	\label{eq:app_new_sde}
    \rmd \bfz_t =\frac{\rmd \log s_t}{\rmd t}\bfz_t \, \rmd t + s_t\sqrt{\frac{\rmd \sigma^2_t}{\rmd t}}\ \rmd \bfw_t, \quad 
    f(t) = \frac{\rmd \log s_t}{\rmd t}, \quad \text{and} \quad g(t)=s_t\sqrt{\frac{\rmd \sigma^2_t}{\rmd t}}. 
\end{equation}
By properly setting the coefficients $s_t$ and $\sigma_t$, we demonstrate that the standard notations of 
the variance-preserving (VP) SDE and the variance-exploding (VE) SDE in the literature~\cite{song2021sde,karras2022edm} can be recovered: 
\begin{itemize}
	\item VP-SDEs~\cite{ho2020ddpm,nichol2021improved,song2021ddim,song2021sde}: By setting $s_t=\sqrt{\alpha_t}$, $\sigma_t=\sqrt{(1-\alpha_t)/\alpha_t}$, $\beta_t = -\rmd \log \alpha_t/\rmd t$, and $\alpha_t\in (0, 1]$ as a decreasing sequence with $\alpha_0=1, \alpha_T\approx 0$, we have 
	\begin{equation}
			\bfz_t = \sqrt{\alpha_t}\ \bfz_0 + \sqrt{1-\alpha_t}\ \bfeps_t, \quad \bfeps_t\sim \calN(\mathbf{0}, \bfI); \qquad  
			\rmd \bfz_t = -\frac{1}{2}\beta_t\bfz_t \, \rmd t +  \sqrt{\beta_t}\rmd \bfw_t.
	\end{equation}
	\item VE-SDEs~\cite{song2019ncsn,song2020improved,song2021sde}: By setting $s_t=1$, we have
	\begin{equation}
			\bfz_t = \bfz_0 + \sigma_t\bfeps_t, \quad \bfeps_t \sim \calN(\mathbf{0}, \bfI); \qquad 
			\rmd \bfz_t =  \sqrt{\frac{\rmd \sigma^2_t}{\rmd t}}\ \rmd \bfw_t.
	\end{equation}
\end{itemize}
\begin{lemma}
    \label{lemma:posterior}
    Given $\bfz_0\sim p_d$ and the transition kernel as~\eqref{eq:app_kernel}, we have $\nablazt \log p_t(\bfz_t)=\left[s_t\sigma_t\right]^{-2}\left(s_t\bbE(\bfz_0|\bfz_t)-\bfz_t\right)$, or equivalently, $\nablazt \log p_t(\bfz_t)=-\left[s_t\sigma_t\right]^{-1}\bbE_{p_{t0}(\bfz_0|\bfz_t)}\bfeps_t$, where $\bfeps_t = \left[s_t\sigma_t\right]^{-1}(\bfz_t - s_t \bfz_0)$. 
\end{lemma}
\begin{proof}
 \begin{equation}   \begin{aligned}
       p_t(\bfz_t)=\int p_d(\bfz_0)p_t(\bfz_t|\bfz_0)\rmd \bfz_0
        , \quad &\rightarrow \quad
        \nablazt p_t(\bfz_t)
        =\int \frac{\left(s_t\bfz_0-\bfz_t\right)}{s_t^2\sigma_t^2}	
        p_d(\bfz_0)p_{0t}(\bfz_t|\bfz_0)\rmd \bfz_0
        \\
        s_t^2\sigma_t^2\nablazt p_t(\bfz_t)&=\int s_t\bfz_0 p_d(\bfz_0)p_{0t}(\bfz_t|\bfz_0)\rmd \bfz_0 -\bfz_t p_t(\bfz_t)\\
        s_t^2\sigma_t^2\frac{\nablazt p_t(\bfz_t)}{p_t(\bfz_t)}&=s_t\int \bfz_0 p_t(\bfz_0|\bfz_t)\rmd \bfz_0 -\bfz_t \\
        \nablazt \log p_t(\bfz_t)&=\left[s_t\sigma_t\right]^{-2}\left(s_t\bbE(\bfz_0|\bfz_t)-\bfz_t\right).
    \end{aligned}
\end{equation} 
We further have $\nablazt \log p_t(\bfz_t)=\left[s_t\sigma_t\right]^{-2}\left(s_t\bbE(\bfz_0|\bfz_t)-\bfz_t\right)=\left[s_t\sigma_t\right]^{-1}\bbE\left(\frac{s_t\bfz_0-\bfz_t}{s_t\sigma_t}|\bfz_t\right)=-\left[s_t\sigma_t\right]^{-1}\bbE_{p_{t0}(\bfz_0|\bfz_t)}\bfeps_t$ due to the linearity of expectation, where $\bfz_t = s_t \bfz_0 + s_t\sigma_t \bfeps_t$, $\bfeps_t \sim \calN(\mathbf{0}, \bfI)$ according to the transition kernel \eqref{eq:app_kernel}.
\end{proof}
We can train a data-prediction model $r_{\bftheta}(\bfz_t; \sigma_t)$ to approximate the posterior $\bbE(\bfz_0|\bfz_t)$, or train a noise-prediction model $\bfeps_{\bftheta}(\bfz_t; \sigma_t)$ to approximate $\bbE\left(\frac{s_t\bfz_0-\bfz_t}{s_t\sigma_t}|\bfz_t\right)$ and then substitute the score with the learned model. 

The probability flow ordinary differential equation (PF-ODE) of the forward SDE~\eqref{eq:app_forward_sde} can also be expressed with $s_t$ and $\sigma_t$
\begin{equation}
	\label{eq:app_new_ode}
    \frac{\rmd \bfz_t}{\rmd t} = \frac{\rmd \log s_t}{\rmd t}\bfz_t - \frac{1}{2} s^2_t\frac{\rmd \sigma^2_t}{\rmd t} \nablazt \log p_t(\bfz_t)= \frac{\rmd \log s_t}{\rmd t}\bfz_t - \frac{\rmd \log\sigma_t}{\rmd t} \left(s_t \bbE(\bfz_0|\bfz_t)-\bfz_t\right).
\end{equation}

In practice, we have two formulas to calculate the \textit{exact solution} from the current position $\bfz_{t_{n+1}}$ to the next position $\bfz_{t_{n}}$ ($t_{n+1}>t_{n}$) in the ODE-based sampling to obtain the sampling trajectory from $t_N$ to $t_0$. One is 
\begin{equation}
	\label{eq:app_exact_direct}
	\bfz_{t_n} = \bfz_{t_{n+1}}+\int_{t_{n+1}}^{t_n} \frac{\rmd \bfz_t}{\rmd t}\rmd t,
\end{equation}
and another leverages the semi-linear structure in~\eqref{eq:app_new_ode} to derive the following equation~\cite{lu2022dpm,zhang2023deis} with the \textit{variant of constants} formula
\begin{equation}
    \label{eq:app_exact_semi}    
    \begin{aligned}
    \bfz_{t_n} 
    &= \exp\left(\int_{t_{n+1}}^{{t_n}} f(t) \rmd t\right)\bfz_{t_{n+1}} -
    \int_{t_{n+1}}^{{t_n}} \left( \exp\left(\int_{t}^{{t_n}} f(r) \rmd r\right) \, \frac{g^2(t)}{2} \nablazt \log p_{t}(\bfz_t) \right) \rmd t\\
    &= \frac{s_{t_n}}{s_{t_{n+1}}}\bfz_{t_{n+1}} -
    s_{t_n}\int_{t_{n+1}}^{{t_n}} \left( s_t\sigma_t\sigma_t^{\prime}\nablazt \log p_t(\bfz_t)\right) \rmd t.
    \end{aligned}
\end{equation}
The above integral, whether in \eqref{eq:app_exact_direct} or \eqref{eq:app_exact_semi} is generally intractable. Therefore, the ODE-based sampling in diffusion models is all about how to solve the integral with numerical approximation methods in each step. Typical strategies include Euler method~\cite{song2021ddim}, Heun's method~\cite{karras2022edm}, Runge-Kutta method~\cite{song2021sde,liu2022pseudo,lu2022dpm}, and linear multistep method~\cite{liu2022pseudo,lu2022dpmpp,zhang2023deis,zhao2023unipc}. 

\subsubsection{The Equivalence of Linear Diffusion Models}
\label{subsubsec:equivalence}
We further prove that various linear diffusion models sharing the same SNR are equivalent up to a scaling factor. We also demonstrate that how their score functions and sampling behaviors are connected.
\begin{proposition}
	\label{prop:app_ito_lemma}
	Any diffusion process defined as \eqref{eq:app_new_sde} can be transformed into its VE counterpart with the change-of-variables formula $\bfx_t=\bfz_t / s_t$ ($t\in (0, T]$), keeping the SNR unchanged.
\end{proposition}
\begin{proof}
    We adopt the change-of-variables formula $\bfx_t=\bfphi(t, \bfz_t)=\bfz_t / s_t$ with $\bfphi:(0, T]\times\bbR^n \rightarrow \bbR^n$, and we denote the $i$-th dimension of $\bfz_t$, $\bfx_t$ and $\bfw_t$ as $\bfz_t[i]$, $\bfx_t[i]$, and $w_t$ respectively; $\bfphi = [\phi_1, \cdots, \phi_i, \cdots, \phi_n]^T$ with a twice differentiable scalar function $\phi_i(t, z)=z/s_t$ of two real variables $t$ and $z$. Since each dimension of $\bfz_t$ is independent, we can apply It\^{o}'s lemma~\cite{oksendal2013stochastic} to each dimension with $\phi_i(t, \bfz_t[i])$ separately. We have
    \begin{equation}
        \frac{\partial \phi_i}{\partial t}=-\frac{z}{s_t^2}\frac{\rmd s_t}{\rmd t}, \quad
        \frac{\partial \phi_i}{\partial z}=\frac{1}{s_t}, \quad
        \frac{\partial^2 \phi_i}{\partial z^2}=0, \quad
        \rmd \bfz_t[i]=\frac{\rmd \log s_t}{\rmd t}\bfz_t[i]\rmd t + s_t\sqrt{\frac{\rmd \sigma^2_t}{\rmd t}} \rmd w_t,
    \end{equation}
    then
    \begin{equation}
    \begin{aligned}
        \rmd \phi_i(t, \bfz_t[i])&=\left(\frac{\partial \phi_i}{\partial t}+ f(t)\bfz_t[i]\frac{\partial \phi_i}{\partial z}+\frac{g^2(t)}{2}\frac{\partial^2 \phi_i}{\partial z^2}\right) \rmd t + g(t)\frac{\partial \phi_i}{\partial z}\rmd w_t\\
        &=\left(\frac{\partial \phi_i}{\partial t}+ \frac{g^2(t)}{2}\frac{\partial^2 \phi_i}{\partial z^2}\right) \rmd t 
        +\frac{\partial \phi_i}{\partial z}\rmd \bfz_t[i]\\
        \rmd \bfx_t[i] & =-\frac{\bfz_t[i]}{s_t}\frac{\rmd \log s_t}{\rmd t}\rmd t+ \frac{1}{s_t}\left(\frac{\rmd \log s_t}{\rmd t}\bfz_t[i]\rmd t + s_t\sqrt{\frac{\rmd \sigma^2_t}{\rmd t}} \rmd w_t\right)\\
        \rmd \bfx_t[i] &= \sqrt{\frac{\rmd \sigma^2_t}{\rmd t}}\ \rmd w_t, \quad \rightarrow \quad
        \rmd \bfx_t = \sqrt{\frac{\rmd \sigma^2_t}{\rmd t}}\ \rmd \bfw_t,
    \end{aligned}
    \end{equation}
with the initial condition $\bfx_0=\bfz_0 \sim p_d$. Since $\sigma_t$ in the above VE-SDE ($\bfx$-space) is exactly the same as that used in the original SDE ($\bfz$-space, \eqref{eq:app_new_sde}), the SNR remains unchanged. 
\end{proof}
Similarly, we provide the PF-ODE in the $\bfx$-space as follows
\begin{equation}
    \rmd \bfx_t = - \sigma_t\nablaxt \log p_t(\bfx_t) \rmd \sigma_t,
\end{equation}
with the score function for $t\in (0, T]$
\begin{equation}
	\begin{aligned}
        \nablaxt \log p_t(\bfx_t)
        &=s_t\nablazt \log \int \calN\left(\bfz_t/s_t; \bfz_0, \sigma_t^2\bfI\right)p_d(\bfz_0) \rmdz_0 \\ 
        &=s_t\nablazt \log \int s_t^{d}\calN\left(\bfz_t; s_t\bfz_0, s_t^2\sigma_t^2\bfI\right)p_d(\bfz_0) \rmdz_0 \\ 
		&=s_t\nablazt \log p_t(\bfz_t).
	\end{aligned}
 \end{equation}
This equation also holds for $t=0$ since $s_0=1$. Thus, $\nablaxt \log p_t(\bfx_t)=s_t\nablazt \log p_t(\bfz_t)$ for $t\in [0, T]$.
\begin{corollary}
    With the same numerical method, the results obtained by using \eqref{eq:app_exact_direct} or \eqref{eq:app_exact_semi} are not equal in general cases. But they become exactly the same for VE-SDEs in the $\bfx$-space.  
\end{corollary}
\begin{corollary}
    \label{cor:equal}
	With the same numerical method, the result obtained by using \eqref{eq:app_exact_semi} in the $\bfz$-space is exactly the same as the result obtained by using \eqref{eq:app_exact_direct} or \eqref{eq:app_exact_semi} in the $\bfx$-space.
\end{corollary}

\begin{proof}[Sketch of proof.]
    Given the sample $\bfz_{t_n}$ obtained by solving the equation \eqref{eq:app_exact_semi} in $\bfz$-space starting from $\bfz_{t_{n+1}}$, we prove that $\bfz_{t_n}/s_{t_n}$ is exactly equal to sampling with the equation~\eqref{eq:app_exact_direct} in $\bfx$-space starting from $\bfx_{t_{n+1}}=\bfz_{t_{n+1}}/s_{t_{n+1}}$ to $\bfx_{t_n}$. We have
\begin{equation}
	\begin{aligned}
        \bfz_{t_n} &= s_{t_n}\left(\frac{\bfz_{t_{n+1}}}{s_{t_{n+1}}}-\int_{t_{n+1}}^{t_n} s_t\sigma_t\sigma_t^{\prime} \nablazt \log p_t(\bfz_t)\rmd t \right)
        = s_{t_n}\left(\bfx_{t_{n+1}}+\int_{t_{n+1}}^{t_n} - \sigma_t\nablaxt \log p_t(\bfx_t)\sigma_t^{\prime} \rmd t\right) \\
		&= s_{t_n}\left(\bfx_{t_{n+1}}+\int_{t_{n+1}}^{t_n} \frac{\rmd \bfx_t}{\rmd t}\rmd t\right) = s_{t_{n}}\bfx_{t_n}.
	\end{aligned}
\end{equation}
\end{proof}

\subsection{Generalized Denoising Output}
\label{subsec:generalized_denoising_output}
All following proofs are conducted in the context of a VE-SDE $\rmdx_t =\sqrt{2t} \, \rmd \bfw_t$, \ie, $\sigma_t=t$ for notation simplicity, and the sampling trajectory always starts from $\hatx_{t_N}\sim \calN(\mathbf{0}, T^2\bfI)$ and ends at $\hatx_{t_0}$. 

The PF-ODEs of the sampling trajectory and denoising trajectory are provided as follows
\begin{equation}
	\text{sampling-ODE:}\, \frac{\rmdx_t}{\rmd t} = \bfeps_{\bftheta}(\bfx_t; t) = \frac{\bfx_t - r_{\bftheta}(\bfx_t; t)}{t}, \quad 
    \text{denoising-ODE:}\,
    \frac{\rmd r_{\bftheta}(\bfx_t; t)}{\rmd t} = - t \frac{\rmd^2 \bfx_t}{\rmd t^2}. 
\end{equation}
The denoising-ODE above is derived by simply rearranging the sampling-ODE as $r_{\bftheta}(\bfx_t; t)=\bfx_t - t \frac{\rmd \bfx_t}{\rmd t}$, and then take the derivative of both sides, \ie, 
$\frac{\rmd r_{\bftheta}(\bfx_t; t)}{\rmd t}=\frac{\rmd \bfx_t}{\rmd t}-\left(\frac{\rmd \bfx_t}{\rmd t}+t\frac{\rmd^2 \bfx_t}{\rmd t^2}\right)= - t \frac{\rmd^2 \bfx_t}{\rmd t^2}$. 

\subsubsection{Convex Combination}
\begin{proposition}
	Given the probability flow ODE~\eqref{eq:epf_ode} and a current position $\hatx_{t_{n+1}}$, $n\in[0, N-1]$ in the sampling trajectory, the next position $\hatx_{t_{n}}$ predicted by a $k$-order Taylor expansion with the time step size ${t_{n+1}}-{t_n}$ equals 
	\begin{equation}
        \begin{aligned}
		\hatx_{t_{n}}&=\frac{{t_n}}{{t_{n+1}}} \hatx_{t_{n+1}} +  \frac{{t_{n+1}}-{t_n}}{{t_{n+1}}} h_{\bftheta}(\hatx_{t_{n+1}}).
        \end{aligned}
	\end{equation}
which is a convex combination of $\hatx_{t_{n+1}}$ and the generalized denoising output $h_{\bftheta}(\hatx_{t_{n+1}})$, 
\begin{equation}
    h_{\bftheta}(\hatx_{t_{n+1}})=r_{\bftheta}(\hatx_{t_{n+1}})- \sum_{i=2}^{k}\frac{1}{i!}\frac{\rmd^{(i)} \bfx_t}{\rmd t^{(i)}}\Big|_{\hatx_{t_{n+1}}}{t_{n+1}}({t_n} - {t_{n+1}})^{i-1}.
\end{equation}
We have $h_{\bftheta}(\hatx_{t_{n+1}})=r_{\bftheta}(\hatx_{t_{n+1}})$ for Euler method ($k=1$), and $h_{\bftheta}(\hatx_{t_{n+1}})=r_{\bftheta}(\hatx_{t_{n+1}})+\frac{{t_{n}}-{t_{n+1}}}{2}\frac{\rmd r_{\bftheta}(\hatx_{t_{n+1}})}{\rmd t}$ for second-order methods ($k=2$). 
\end{proposition}
\begin{proof}
    The $k$-order Taylor expansion at $\hatx_{t_{n+1}}$ is
    \begin{equation}
    \begin{aligned}
    \hatx_{t_{n}}
    &= \sum_{i=0}^{k}\frac{1}{i!}\frac{\rmd^{(i)} \bfx_t}{\rmd t^{(i)}}\Big|_{\hatx_{t_{n+1}}}(t_n - t_{n+1})^{i}\\
    &=\hatx_{t_{n+1}} +  (t_n - t_{n+1})\frac{\rmdx_t}{\rmd t}\Big|_{\hatx_{t_{n+1}}}
    +\sum_{i=2}^{k}\frac{1}{i!}\frac{\rmd^{(i)} \bfx_t}{\rmd t^{(i)}}\Big|_{\hatx_{t_{n+1}}}(t_n - t_{n+1})^i\\
    &=\hatx_{t_{n+1}} +  \frac{t_n - t_{n+1}}{t_{n+1}} \left(\hatx_{t_{n+1}} - r_{\bftheta}(\hatx_{t_{n+1}})\right)
    + \sum_{i=2}^{k}\frac{1}{i!}\frac{\rmd^{(i)} \bfx_t}{\rmd t^{(i)}}\Big|_{\hatx_{t_{n+1}}}(t_n - t_{n+1})^{i}\\
    &=\frac{{t_n}}{{t_{n+1}}} \hatx_{t_{n+1}} +  \frac{{t_{n+1}}-{t_n}}{{t_{n+1}}} h_{\bftheta}(\hatx_{t_{n+1}}), 
    \end{aligned}
\end{equation}
where $h_{\bftheta}(\hatx_{t_{n+1}})=r_{\bftheta}(\hatx_{t_{n+1}})- \sum_{i=2}^{k}\frac{1}{i!}\frac{\rmd^{(i)} \bfx_t}{\rmd t^{(i)}}\Big|_{\hatx_{t_{n+1}}}{t_{n+1}}(t_n - t_{n+1})^{i-1}$. As for first-order approximation ($k=1$), we have $h_{\bftheta}(\hatx_{t_{n+1}})=r_{\bftheta}(\hatx_{t_{n+1}})$. As for second-order approximation ($k=2$), we have 
\begin{equation}
    \label{eq:second-order}
    h_{\bftheta}(\hatx_{t_{n+1}})=r_{\bftheta}(\hatx_{t_{n+1}})-\frac{t_{n+1}(t_{n}-t_{n+1})}{2}\frac{\rmd^{2}\bfx_t }{\rmd t^2}\Big|_{\hatx_{t_{n+1}}}=r_{\bftheta}(\hatx_{t_{n+1}})+\frac{t_{n}-t_{n+1}}{2} \frac{\rmd r_{\bftheta}(\hatx_{t_{n+1}})}{\rmd t}.
\end{equation} 
\end{proof}
\begin{corollary}
	The denoising output $r_{\bftheta}(\hatx_{t_{n+1}})$ reflects the prediction made by a single Euler step from $\hatx_{t_{n+1}}$ with the time step size $t_{n+1}$.
\end{corollary}
\begin{proof}
	The prediction of such an Euler step equals to
	$\hatx_{t_{n+1}} + (0 - {t_{n+1}}) \left(\hatx_{t_{n+1}} - r_{\bftheta}(\hatx_{t_{n+1}})\right)/{t_{n+1}}=r_{\bftheta}(\hatx_{t_{n+1}})$. 
\end{proof}

\begin{corollary}
     Each previously proposed second-order ODE-based accelerated sampling method corresponds to a specific first-order finite difference of $\rmd r_{\bftheta}(\hatx_{t_{n+1}})/\rmd t$. 
\end{corollary}
\begin{proof}
    The proof is provided in the next section.
\end{proof}
\subsubsection{Second-Order ODE Samplers as Finite Differences of Denoising Trajectory}

To accelerate the sampling speed of diffusion models, various numerical solver-based samplers have been developed in the past several years \cite{song2021ddim,song2021sde,karras2022edm,lu2022dpm,zhang2023deis}. In particular, second-order ODE-based samplers are relatively promising in the practical use since they strike a good balance between fast sampling and decent visual quality \cite{rombach2022ldm,balaji2022ediffi}. 

We next demonstrate that each of them can be rewritten as a specific way to perform finite difference of the denoising trajectory $\rmd r_{\bftheta}(\hatx_{t_{n+1}})/\rmd t$, as shown in Table~\ref{tab:sampler}. We assume that history points are already available for samplers in PNDM and DEIS and they are calculated with Euler method (otherwise, we need to consider  Taylor expansion in terms of the data-prediction model $r_{\bftheta}(\bfx_t)$ rather than the noise-prediction model $\bfeps_{\bftheta}(\bfx_t)$ as the implementation in the original papers). 

\subsubsection{EDMs \cite{karras2022edm}}
EDMs employ Heun’s 2\textsuperscript{nd} order method, where one Euler step is first applied and followed by a second order correction, which can be written as
\begin{equation}
    \begin{aligned}
    \hatx'_{t_{n}}
    & = \hatx_{t_{n+1}} + (t_n - t_{n+1}) \bfeps_{\bftheta}(\hatx_{t_{n+1}}), \\
    \hatx_{t_{n}}
    & = \hatx_{t_{n+1}} + (t_n - t_{n+1}) \left(0.5\bfeps_{\bftheta}(\hatx_{t_{n+1}}) + 0.5\bfeps_{\bftheta}(\hatx'_{t_n})\right) \\
    & = \hatx_{t_{n+1}} + (t_n - t_{n+1}) \bfeps_{\bftheta}(\hatx_{t_{n+1}})
    + \frac{1}{2}(t_n - t_{n+1})^2 \frac{\bfeps_{\bftheta}(\hatx'_{t_n}) - \bfeps_{\bftheta}(\hatx_{t_{n+1}})}{t_n - t_{n+1}}.
    \end{aligned}
\end{equation}
By using $r_{\bftheta}(\bfx_t; t)=\bfx_t - t \bfeps_{\bftheta}(\bfx_t; t)$, the sampling iteration above is equivalent to
\begin{equation}
\begin{aligned}
    \hatx_{t_{n}}
    & = \hatx_{t_{n+1}} + (t_n - t_{n+1}) \frac{\hatx_{t_{n+1}} - r_{\bftheta}(\hatx_{t_{n+1}})}{t_{n+1}}
    + \frac{1}{2}(t_n - t_{n+1})^2 \frac{\frac{\hatx'_{t_n} - r_{\bftheta}(\hatx'_{t_n})}{t_n} - \frac{\hatx_{t_{n+1}} - r_{\bftheta}(\hatx_{t_{n+1}})}{t_{n+1}}}{t_n - t_{n+1}} \\
    & = \frac{t_n}{t_{n+1}}\hatx_{t_{n+1}} + \frac{t_{n+1} - t_n}{t_{n+1}} r_{\bftheta}(\hatx_{t_{n+1}}) 
    - \frac{1}{2}\frac{(t_n - t_{n+1})^2}{t_{n+1}} \frac{t_{n+1}}{t_n} \frac{r_{\bftheta}(\hatx'_{t_n}) - r_{\bftheta}(\hatx_{t_{n+1}})}{t_n - t_{n+1}}\\
    & = \frac{t_n}{t_{n+1}}\hatx_{t_{n+1}} + \frac{t_{n+1} - t_n}{t_{n+1}} 
    \left(
    r_{\bftheta}(\hatx_{t_{n+1}}) 
    + \frac{t_n - t_{n+1}}{2} \frac{t_{n+1}}{t_n} \frac{r_{\bftheta}(\hatx'_{t_n}) - r_{\bftheta}(\hatx_{t_{n+1}})}{t_n - t_{n+1}}
    \right).
\end{aligned}
\end{equation}
Compared with the generalized denoising output \eqref{eq:second-order}, we have $\frac{\rmd r_{\bftheta}(\hatx_{t_{n+1}})}{\rmd t}\approx \frac{t_{n+1}}{t_n} \frac{r_{\bftheta}(\hatx'_{t_n}) - r_{\bftheta}(\hatx_{t_{n+1}})}{t_n - t_{n+1}}$.

\subsubsection{DPM-Solver \cite{lu2022dpm}}
According to (3.3) in DPM-Solver, the exact solution of PF-ODE in the VE-SDE setting is given by
\begin{equation}
    \hatx_{t_{n}} = \hatx_{t_{n+1}} + \int_{t_{n+1}}^{t_n} \bfeps_{\bftheta}(\hatx_t) \rmd t.
\end{equation}
The Taylor expansion of $\bfeps_{\bftheta}(\hatx_t)$ \textit{w.r.t.}\ time at $t_{n+1}$ is
\begin{equation}
    \bfeps_{\bftheta}(\hatx_t) = \sum_{m=0}^{k-1} \frac{(t - t_{n+1})^m}{m!} \bfeps_{\bftheta}^{(m)}(\hatx_{t_{n+1}}) + \mathcal{O} \left( (t - t_{n+1})^k \right).
\end{equation}
Then, the DPM-Solver-k sampler can be written as
\begin{equation}
\begin{aligned}
    \hatx_{t_{n}} 
    & = \hatx_{t_{n+1}} + \sum_{m=0}^{k-1} \bfeps_{\bftheta}^{(m)}(\hatx_{t_{n+1}}) \int_{t_{n+1}}^{t_n} \frac{(t - t_{n+1})^m}{m!} \rmd t + \mathcal{O} \left((t_n - t_{n+1})^k \right) \\
    & = \hatx_{t_{n+1}} + \sum_{m=0}^{k-1} \frac{(t_n - t_{n+1})^{m+1}}{(m+1)!} \bfeps_{\bftheta}^{(m)}(\hatx_{t_{n+1}}) + \mathcal{O} \left((t_n - t_{n+1})^{k+1} \right).
\end{aligned}
\end{equation}
Specifically, when $k=2$, the DPM-Solver-2 sampler is given by
\begin{equation}
    \hatx_{t_{n}}
    = \hatx_{t_{n+1}} + (t_n - t_{n+1}) \bfeps_{\bftheta}(\hatx_{t_{n+1}})
    + \frac{1}{2}(t_n - t_{n+1})^2\frac{\rmd \bfeps_{\bftheta}(\hatx_{t_{n+1}})}{\rmd t},
\end{equation}
The above second-order term in \cite{lu2022dpm} is approximated by 
\begin{equation}
    \frac{\rmd \bfeps_{\bftheta}(\hatx_{t_{n+1}})}{\rmd t} \approx \frac{\bfeps_{\bftheta}(\hatx_{s_n}) - \bfeps_{\bftheta}(\hatx_{t_{n+1}})}{(t_n - t_{n+1}) / 2},
\end{equation}
where $s_n = \sqrt{t_n t_{n+1}}$ and $\hatx_{s_n} = \hatx_{t_{n+1}} + (s_n - t_{n+1}) \bfeps_{\bftheta}(\hatx_{t_{n+1}})$. We have
\begin{equation}
\begin{aligned}
    \hatx_{t_{n}}
    & = \hatx_{t_{n+1}} + (t_n - t_{n+1}) \bfeps_{\bftheta}(\hatx_{t_{n+1}})
    + \frac{1}{2}(t_n - t_{n+1})^2 \frac{\bfeps_{\bftheta}(\hatx_{s_n}) - \bfeps_{\bftheta}(\hatx_{t_{n+1}})}{(t_n - t_{n+1}) / 2} \\
    & = \frac{t_n}{t_{n+1}}\hatx_{t_{n+1}} + \frac{t_{n+1} - t_n}{t_{n+1}} r_{\bftheta}(\hatx_{t_{n+1}}) - \frac{1}{2}\frac{(t_n - t_{n+1})^2}{t_{n+1}} \frac{t_{n+1}}{s_n} \frac{r_{\bftheta}(\hatx_{s_n}) - r_{\bftheta}(\hatx_{t_{n+1}})}{(t_n - t_{n+1}) / 2}\\
    & = \frac{t_n}{t_{n+1}}\hatx_{t_{n+1}} + \frac{t_{n+1} - t_n}{t_{n+1}} 
    \left(
    r_{\bftheta}(\hatx_{t_{n+1}}) +\frac{t_n - t_{n+1}}{2} \frac{t_{n+1}}{s_n} \frac{r_{\bftheta}(\hatx_{s_n}) - r_{\bftheta}(\hatx_{t_{n+1}})}{(t_n - t_{n+1}) / 2}
    \right).
\end{aligned}
\end{equation}
Compared with the generalized denoising output \eqref{eq:second-order}, we have $\frac{\rmd r_{\bftheta}(\hatx_{t_{n+1}})}{\rmd t}\approx \frac{t_{n+1}}{s_n} \frac{r_{\bftheta}(\hatx_{s_n}) - r_{\bftheta}(\hatx_{t_{n+1}})}{(t_n - t_{n+1}) / 2}$.

\begin{table*}[t]
    \caption{Each second-order ODE-based sampler listed below corresponds to a specific finite difference of the denoising trajectory. $\gamma$ denotes a correction coefficient of forward differences. 
    DDIM is a first-order sampler listed for comparison. 
    GENIE trains a neural network to approximate high-order derivatives.
    $r_{\bftheta}(\hatx_{t_{n+2}})$ in S-PNDM and DEIS denotes a previous denoising output. 
    $s_n = \sqrt{t_n t_{n+1}}$ in DPM-Solver-2. 
    $\hatx'_{t_n}$ in EDMs denotes the output of an intermediate Euler step.}
    \label{tab:sampler}
    \begin{center}
    \begin{tabular}{lrr}
    \toprule
    \multicolumn{1}{c}{\bf ODE solver-based samplers} & \multicolumn{1}{c}{$\rmd r_{\bftheta}(\hatx_{t_{n+1}}) / \rmd t$} & \multicolumn{1}{c}{$\gamma$} \\
    \midrule
    DDIM \cite{song2021ddim} & None & None \\
    GENIE \cite{dockhorn2022genie} & Neural Networks & None\\ 
    S-PNDM \cite{liu2022pseudo} & $\gamma\left(r_{\bftheta}(\hatx_{t_{n+1}}) - r_{\bftheta}(\hatx_{t_{n+2}})\right)/(t_n - t_{n+1})$ & 1 \\
    DEIS ($\rho$AB1) \cite{zhang2023deis} & $\gamma\left(r_{\bftheta}(\hatx_{t_{n+1}}) - r_{\bftheta}(\hatx_{t_{n+2}})\right) / (t_{n+1} - t_{n+2})$ & 1\\
    DPM-Solver-2 \cite{lu2022dpm} & $\gamma\left(r_{\bftheta}(\hatx_{s_n}) - r_{\bftheta}(\hatx_{t_{n+1}})\right)/\left((t_n - t_{n+1}) / 2\right)$ & $t_{n+1}/s_n $\\
    EDMs (Heun) \cite{karras2022edm} & $\gamma\left(r_{\bftheta}(\hatx'_{t_n}) - r_{\bftheta}(\hatx_{t_{n+1}})\right) / (t_n - t_{n+1})$ & $t_{n+1}/t_n$\\
    \bottomrule
    \end{tabular}
    \end{center}
\end{table*}

\subsubsection{PNDM \cite{liu2022pseudo}}
Assume that the previous denoising output $r_{\bftheta}(\hatx_{t_{n+2}})$ is available, then one S-PNDM sampler step can be written as
\begin{equation}
\begin{aligned}
    \hatx_{t_{n}}
    & = \hatx_{t_{n+1}} + (t_n - t_{n+1}) \frac{1}{2} \left( 3\bfeps_{\bftheta}(\hatx_{t_{n+1}}) - \bfeps_{\bftheta}(\hatx_{t_{n+2}}) \right) \\
    & = \hatx_{t_{n+1}} + (t_n - t_{n+1}) \bfeps_{\bftheta}(\hatx_{t_{n+1}})
    + \frac{1}{2}(t_n - t_{n+1})^2 \frac{\bfeps_{\bftheta}(\hatx_{t_{n+1}}) - \bfeps_{\bftheta}(\hatx_{t_{n+2}})}{t_n - t_{n+1}} \\
    & = \frac{t_n}{t_{n+1}}\hatx_{t_{n+1}} + \frac{t_{n+1} - t_n}{t_{n+1}} r_{\bftheta}(\hatx_{t_{n+1}}) - \frac{1}{2} \frac{(t_n - t_{n+1})^2}{t_{n+1}} \frac{r_{\bftheta}(\hatx_{t_{n+1}}) - r_{\bftheta}(\hatx_{t_{n+2}})}{t_n - t_{n+1}}\\
    & = \frac{t_n}{t_{n+1}}\hatx_{t_{n+1}} + \frac{t_{n+1} - t_n}{t_{n+1}} 
    \left(
    r_{\bftheta}(\hatx_{t_{n+1}}) + \frac{t_n - t_{n+1}}{2} \frac{r_{\bftheta}(\hatx_{t_{n+1}}) - r_{\bftheta}(\hatx_{t_{n+2}})}{t_n - t_{n+1}}
    \right).
\end{aligned}
\end{equation}
Compared with the generalized denoising output \eqref{eq:second-order}, we have $\frac{\rmd r_{\bftheta}(\hatx_{t_{n+1}})}{\rmd t}\approx \frac{r_{\bftheta}(\hatx_{t_{n+1}}) - r_{\bftheta}(\hatx_{t_{n+2}})}{t_n - t_{n+1}}$.

\subsubsection{DEIS \cite{zhang2023deis}}
In DEIS paper, the solution of PF-ODE in the VE-SDE setting is given by
\begin{equation}
\begin{aligned}
    \hatx_{t_{n}}
    & = \hatx_{t_{n+1}} + \sum_{j=0}^{r} C_{(n+1)j} \bfeps_{\bftheta}(\hatx_{t_{n+1+j}}), \\
    C_{(n+1)j}
    & = \int_{t_{n+1}}^{t_n} \prod_{k \neq j} \frac{\tau - t_{n+1+k}}{t_{n+1+j} - t_{n+1+k}} \rmd \tau.
\end{aligned}
\end{equation}
$r=1$ yields the $\rho$AB1 sampler:
\begin{equation}
\begin{aligned}
    \hatx_{t_{n}}
    & = \hatx_{t_{n+1}} + \bfeps_{\bftheta}(\hatx_{t_{n+1}}) \int_{t_{n+1}}^{t_n} \frac{\tau - t_{n+2}}{t_{n+1} - t_{n+2}} \rmd \tau
    + \bfeps_{\bftheta}(\hatx_{t_{n+2}}) \int_{t_{n+1}}^{t_n} \frac{\tau - t_{n+1}}{t_{n+2} - t_{n+1}} \rmd \tau \\
    & = \hatx_{t_{n+1}} + \bfeps_{\bftheta}(\hatx_{t_{n+1}}) \frac{(t_n - t_{n+2})^2-(t_{n+1} - t_{n+2})^2}{2(t_{n+1} - t_{n+2})}
    + \bfeps_{\bftheta}(\hatx_{t_{n+2}}) \frac{(t_n - t_{n+1})^2}{2(t_{n+2} - t_{n+1})} \\
    & = \hatx_{t_{n+1}} + (t_n - t_{n+1}) \bfeps_{\bftheta}(\hatx_{t_{n+1}})
    + \frac{1}{2}(t_n - t_{n+1})^2 \frac{\bfeps_{\bftheta}(\hatx_{t_{n+1}}) - \bfeps_{\bftheta}(\hatx_{t_{n+2}})}{t_{n+1} - t_{n+2}} \\
    & = \frac{t_n}{t_{n+1}}\hatx_{t_{n+1}} + \frac{t_{n+1} - t_n}{t_{n+1}} r_{\bftheta}(\hatx_{t_{n+1}}) - \frac{1}{2} \frac{(t_n - t_{n+1})^2}{t_{n+1}} \frac{r_{\bftheta}(\hatx_{t_{n+1}}) - r_{\bftheta}(\hatx_{t_{n+2}})}{t_{n+1} - t_{n+2}}\\
    & = \frac{t_n}{t_{n+1}}\hatx_{t_{n+1}} + \frac{t_{n+1} - t_n}{t_{n+1}} 
    \left(
    r_{\bftheta}(\hatx_{t_{n+1}}) + \frac{t_n - t_{n+1}}{2} \frac{r_{\bftheta}(\hatx_{t_{n+1}}) - r_{\bftheta}(\hatx_{t_{n+2}})}{t_{n+1} - t_{n+2}}
    \right).
\end{aligned}
\end{equation}
Compared with the generalized denoising output \eqref{eq:second-order}, we have $\frac{\rmd r_{\bftheta}(\hatx_{t_{n+1}})}{\rmd t}\approx \frac{r_{\bftheta}(\hatx_{t_{n+1}}) - r_{\bftheta}(\hatx_{t_{n+2}})}{t_{n+1} - t_{n+2}}$.

\subsection{Theoretical Analysis of the Trajectory Structure}

\subsubsection{The Optimal Denoising Output}
\label{subsubsec:optimal_denoising_output}

\begin{lemma}
	\label{lemma:dae}
	The optimal estimator $r_{\bftheta}^{\star}\left(\bfx_t; \sigma_t\right)$, also known as Bayesian least squares estimator, of the minimization of denoising autoencoder (DAE) objective is the conditional expectation $\bbE\left(\bfx_0|\bfx_t\right)$
	\begin{equation}
		\label{eq:app_dae}
		\begin{aligned}
			\calL_{\DAE} 
            =\bbE_{\bfx_0 \sim p_d(\bfx_0)} \bbE_{\bfx_t \sim p_{0t}(\bfx_t|\bfx_0)} \lVert r_{\bftheta}\left(\bfx_t; \sigma_t\right) - \bfx_0  \rVert^2_2
            =\int{p_t(\bfx_t) p_{t0}(\bfx_0|\bfx_t)} \lVert r_{\bftheta}\left(\bfx_t; \sigma_t\right) - \bfx_0 \rVert^2_2.
		\end{aligned}
	\end{equation} 
\end{lemma}
\begin{proof}
	The solution can be easily obtained by setting the derivative of $\calL_{\DAE}$ equal to zero.
\end{proof}
Suppose we have a training dataset $\calD=\{\bfy_i\}_{i=1}^{|\calI|}$ where each data $\bfy_i$ is sampled from an unknown data distribution $p_d$. 
The empirical data distribution $\hat{p}_{d}$ is denoted as a summation of multiple \textit{Dirac delta functions} (\textit{a.k.a} a mixture of Gaussian distribution): $\hat{p}_{d}(\bfy)=\frac{1}{|\calI|}\sum_{i=1}^{|\calI|}\delta(\lVert \bfy - \bfy_i\rVert)$, and the Gaussian kernel density estimate (KDE) is 
\begin{equation}
    \hat{p}_t(\bfx_t)=\int p_{0t}(\bfx_t|\bfy)\hat{p}_{d}(\bfy) = \frac{1}{|\calI|}\sum_{i}\calN\left(\bfx_t;\bfy_i, \sigma_t^2\bfI\right).
\end{equation}
\begin{proposition}
The optimal denoising output of training a denoising autoencoder with the empirical data distribution is a convex combination of original data points, where each weight $u_i$ is calculated based on the time-scaled and normalized $\ell_2$ distance between the input $\bfx_t$ and $\bfy_i$ belonging to the dataset $\calD$:
\begin{equation}
	\label{eq:app_optimal}
	\begin{aligned}
		r_{\bftheta}^{\star}(\bfx_t; \sigma_t)
		&= \min_{r_{\bftheta}} \bbE_{\bfy \sim \hat{p}_d} \bbE_{\bfx_t \sim p_{0t}(\bfx_t|\bfy)} \lVert r_{\bftheta}(\bfx_t; \sigma_t) - \bfy  \rVert^2_2
		= \sum_{i} \frac{\exp \left(-\lVert \bfx_t - \bfy_i \rVert^2_2/2\sigma_t^2\right)}{\sum_{j}\exp \left(-\lVert \bfx_t - \bfy_j \rVert^2_2/2\sigma_t^2\right)} \bfy_i=\sum_{i} u_i \bfy_i,
	\end{aligned}
\end{equation}
with the coefficients satisfying $\sum_{i} u_i =1$.
\end{proposition}
This equation appears to be highly similar to the iterative formula used in mean shift \cite{fukunaga1975estimation,cheng1995mean,comaniciu2002mean,yamasaki2020mean}, especially annealed mean shift \cite{shen2005annealedms}. 
\begin{proof}
    
Based on Lemmas~\ref{lemma:posterior} and~\ref{lemma:dae}, the optimal denoising output is
\begin{equation}
    \begin{aligned}
        r_{\bftheta}^{\star}\left(\bfx_t; \sigma_t\right) =
        \bbE\left(\bfy|\bfx_t\right)&=\bfx_t + \sigma_t^2\nablaxt \log \hat{p}_t(\bfx_t)\\
        &=\bfx_t + \sigma_t^2 \sum_{i}\frac{\nablaxt\calN\left(\bfx_t;\bfy_i, \sigma_t^2\bfI\right)}{\sum_{j}\calN\left(\bfx_t;\bfy_j, \sigma_t^2\bfI\right)}\\
        &=\bfx_t + \sigma_t^2 \sum_{i}\frac{\calN\left(\bfx_t;\bfy_i, \sigma_t^2\bfI\right)}{\sum_{j}\calN\left(\bfx_t;\bfy_j, \sigma_t^2\bfI\right)}\left( \frac{\bfy_i-\bfx_t}{\sigma_t^2}\right)\\ 
        &=\bfx_t + \sum_{i}\frac{\calN\left(\bfx_t;\bfy_i, \sigma_t^2\bfI\right)}{\sum_{j}\calN\left(\bfx_t;\bfy_j, \sigma_t^2\bfI\right)}\left(\bfy_i-\bfx_t\right)\\ 
        &=\sum_{i}\frac{\calN\left(\bfx_t;\bfy_i, \sigma_t^2\bfI\right)}{\sum_{j}\calN\left(\bfx_t;\bfy_j, \sigma_t^2\bfI\right)}\bfy_i\\ 
        &=\sum_{i} \frac{\exp \left(-\lVert \bfx_t - \bfy_i \rVert^2_2/2\sigma_t^2\right)}{\sum_{j}\exp \left(-\lVert \bfx_t - \bfy_j \rVert^2_2/2\sigma_t^2\right)} \bfy_i.
    \end{aligned}
\end{equation}
\end{proof}

\subsubsection{Theoretical Connection to Mean Shift}
\label{subsubsec:meanshift}
Mean shift is a well-known non-parametric algorithm designed to seek modes of a density function, typical a KDE, via iteratively gradient ascent with adaptive step sizes. 
Given a current position $\bfx$, mean shift with a Gaussian kernel and bandwidth $h$ iteratively adds a vector $\bfm(\bfx)-\bfx$, which points toward the maximum increase in the kernel density estimate $p_h(\bfx)=\frac{1}{|\calI|}\sum_{i=1}^{|\calI|}\calN(\bfx; \bfy_i, h^2\bfI)$, to itself, \ie, $\bfx \leftarrow \left[\bfm(\bfx)-\bfx\right] + \bfx$. The \textit{mean vector} is
\begin{equation}
	\label{eq:mean_shift}
	\bfm (\bfx, h)=
	\sum_{i} v_i \bfy_i = \sum_{i} \frac{\exp \left(-\lVert \bfx - \bfy_i \rVert^2_2/2h^2\right)}{\sum_{j}\exp \left(-\lVert \bfx - \bfy_j \rVert^2_2/2h^2\right)} \bfy_i,
\end{equation}
with the coefficients satisfying $\sum_{i} v_i =1$. 
As a mode-seeking algorithm, mean shift has shown particularly successful in clustering \cite{cheng1995mean,carreira2015review}, image segmentation \cite{comaniciu1999mean,comaniciu2002mean} and video tracking \cite{comaniciu2000real,comaniciu2003kernel}. 
By treating the bandwidth $\sigma_t$ in \eqref{eq:app_optimal} as the bandwidth $h$ in \eqref{eq:mean_shift}, we build a connection between the optimal denoising output of a diffusion model and annealed mean shift under the KDE-based data modeling.
\subsubsection{The Constant Magnitude of Vector Field and Sampling Trajectory Length}
\label{subsubsec:rotation}
\begin{lemma}[see Section 3.1 in \cite{vershynin2018high}]
    \label{lemma:concentration}
    Given a high-dimensional isotropic Gaussian noise $\bfz \sim \calN(\mathbf{0} ; {\sigma^2\bfI_d})$, $\sigma>0$, we have $\bbE \left\lVert \bfz \right\rVert^2=\sigma^2 d$, and with high probability, $\bfz$ stays within a ``thin spherical shell'': $\lVert \bfz \rVert = \sigma \sqrt{d} \pm O(1)$. 
\end{lemma}
\begin{proof}
	We denote $\bfz_i$ as the $i$-th dimension of random variable $\bfz$, then the expectation and variance is $\bbE \left[\bfz_i\right]=0$, $\bbV \left[\bfz_i\right]=\sigma^2$, respectively. The fourth central moment is $\bbE\left[ \bfz_i^4\right] = 3{\sigma ^4}$. Additionally, 
    \begin{equation}
        \begin{aligned}
            \bbE \left[\bfz_i^2\right]=\bbV \left[\bfz_i\right]+\bbE\left[\bfz_i\right]^2
            &=\sigma^2,\qquad
            \bbE \left[\lVert \bfz \rVert^2\right]=\bbE \left[\sum_{i=1}^{d}\bfz_i^2\right]=\sum_{i=1}^{d}\bbE \left[\bfz_i^2\right]=\sigma^2 d, \\
            \bbV\left[{\lVert \bfz  \rVert^2} \right] 
            &= \mathbb{E} \left[ {\lVert \bfz  \rVert}^4\right] - \left(\bbE\left[
    		  \lVert \bfz  \rVert^2 \right] \right)^2 
    		= 2d \sigma ^4,
            \end{aligned}
    \end{equation}
	Then, we have 
	\begin{equation}
		\begin{aligned}
		\bbE \left[ 
		\lVert \bfx + \bfz \rVert^2 - 
		\lVert \bfx \rVert^2
		\right] 
		& = \bbE \left[ 
		\lVert \bfz \rVert^2 + 2 \bfx^T\bfz 
		\right]
        =
        \bbE \left[ 
		\lVert \bfz \rVert^2
		\right]=\sigma^2 d. 
	\end{aligned}	
	\end{equation}

Furthermore, the standard deviation of $\lVert \bfz  \rVert^2$ is ${\sigma ^2}\sqrt {2d}$, which means
\begin{equation}
	\begin{aligned}
		\lVert \bfz  \rVert^2 = {\sigma ^2}d \pm {\sigma ^2}\sqrt {2d}  = {\sigma ^2}d \pm O\left( \sqrt{d} \right), \qquad
		\lVert \bfz \rVert \hspace{0.4em} = \sigma \sqrt{d} \pm O\left( 1 \right),
	\end{aligned}
\end{equation}
holds with high probability.
\end{proof}
We denote the optimal denoising output in the KDE-based data modeling as $r_{\bftheta}^{\star}(\bfx_t; \sigma_t)$ (see Section~\ref{subsubsec:optimal_denoising_output}). In this case, the optimal noise prediction is denoted as $\bfeps_{\bftheta}^{\star}(\bfx_t; \sigma_t)=\frac{\bfx_t - r_{\bftheta}^{\star}(\bfx_t)}{\sigma_t}$, and the optimal empirical PF-ODE in \eqref{eq:epf_ode} becomes 
\begin{equation}
    \rmd \bfx_t = \bfeps_{\bftheta}^{\star}(\bfx_t; \sigma_t) \rmd \sigma_t.
\end{equation}
\begin{remark}
    \label{remark:eps_norm}
    Intriguingly, the magnitude of $\bfeps_{\bftheta}^{\star}(\bfx_t; \sigma_t)$ approximately distributes around $\sqrt{d}$. The total trajectory length approximately equals $\sigma_T \sqrt{d}$, where $d$ denotes the data dimension.   
\end{remark}
We next provide a sketch of proof. 
Suppose the data distribution lies in a smooth real low-dimensional manifold with the intrinsic dimension as $m$. According to the \textit{Whitney embedding theorem}~\cite{whitney1936differentiable}, it can be smoothly embedded in a real $2m$ Euclidean space. We then decompose each $\bfeps_{\bftheta}^{\star}\in \bbR^{d}$ vector as $\bfeps_{\bftheta, \parallel}^{\star}$ and $\bfeps_{\bftheta, \perp}^{\star}$, which are parallel and perpendicular to the $2m$ Euclidean space, respectively. Therefore, we have $\lVert \bfeps_{\bftheta}^{\star} \rVert_2=\lVert \bfeps_{\bftheta, \parallel}^{\star}+\bfeps_{\bftheta, \perp}^{\star}\rVert_2
\geq\lVert\bfeps_{\bftheta, \perp}^{\star}\rVert_2\approx\sqrt{d-2m}$. 

We provide a upper bound for the $\lVert \bfeps_{\bftheta}^{\star} \rVert_2$ below
\begin{equation}
    \begin{aligned}
        \bbE_{p_t(\bfx_t)}\lVert \bfeps_{\bftheta}^{\star} \rVert_2
        &=\bbE_{p_t(\bfx_t)}\lVert \frac{\bfx_t - r_{\bftheta}^{\star}(\bfx_t)}{\sigma_t} \rVert_2
        =\bbE_{p_t(\bfx_t)}\lVert \frac{\bfx_t - \bbE(\bfx_0|\bfx_t)}{\sigma_t} \rVert_2\\
        &=\bbE_{p_t(\bfx_t)}\lVert \bbE(\frac{\bfx_t - \bfx_0}{\sigma_t}|\bfx_t) \rVert_2
        =\bbE_{p_t(\bfx_t)}\lVert \bbE_{p_{t0}(\bfx_0|\bfx_t)}\bfeps\rVert_2\\
        &\leq\bbE_{p_t(\bfx_t)}\bbE_{p_{t0}(\bfx_0|\bfx_t)} \lVert \bfeps\rVert_2\\
        &\leq\bbE_{p_0(\bfx_0)}\bbE_{p_{0t}(\bfx_t|\bfx_0)} \lVert \bfeps\rVert_2\\
        &\approx \sqrt{d} \qquad\qquad (\text{concentration of measure, Lemma~\ref{lemma:concentration}}). 
    \end{aligned}
\end{equation}
Additionally, the variance of $\lVert \bfeps_{\bftheta}^{\star} \rVert_2$ is relatively small.
\begin{equation}
    \begin{aligned}
        \mathrm{Var}_{p_t(\bfx_t)}\lVert \bfeps_{\bftheta}^{\star} \rVert_2
        &=\mathrm{Var}_{p_t(\bfx_t)}\lVert \bbE_{p_{t0}(\bfx_0|\bfx_t)}\bfeps\rVert_2
        = \bbE_{p_t(\bfx_t)}\lVert \bbE_{p_{t0}(\bfx_0|\bfx_t)}\bfeps\rVert_2^2 - \left[\bbE_{p_t(\bfx_t)}\lVert \bbE_{p_{t0}(\bfx_0|\bfx_t)}\bfeps\rVert_2\right]^2\\
        &\leq \bbE_{p_t(\bfx_t)}\bbE_{p_{t0}(\bfx_0|\bfx_t)}\lVert \bfeps\rVert_2^2 - (d-2m)
        = \bbE_{p_0(\bfx_0)}\bbE_{p_{0t}(\bfx_0|\bfx_t)}\lVert \bfeps\rVert_2^2 - (d-2m)\\
        &= d - (d-2m)\\
        &= 2m
    \end{aligned}
\end{equation}
Therefore, the standard deviation of $\lVert \bfeps_{\bftheta}^{\star} \rVert_2$ is upper bounded by $\sqrt{2m}$.
Since $d\gg m$, we can conclude that in the optimal case, the magnitude of vector field is approximately constant, \ie, $\lVert \bfeps_{\bftheta}^{\star} \rVert_2\approx \sqrt{d}$.

The total sampling trajectory length is
$\sum_{n=0}^{N-1} (\sigma_{t_{n+1}} - \sigma_{t_{n}}) \lVert \bfeps_{\bftheta}^{\star}(\bfx_{t_{n+1}}) \rVert_2\approx \sigma_{T} \sqrt{d}$. Therefore, in the optimal case, we have $\lVert r_{\bftheta}(\hatx_{t_{n+1}}) - \hatx_{t_{n}}\rVert_2 = (\sigma_{t_n}/\sigma_{t_{n+1}})\lVert r_{\bftheta}(\hatx_{t_{n+1}}) - \hatx_{t_{n+1}}\rVert_2=\sigma_{t_{n}}\lVert\bfeps_{\bftheta}^{\star}(\hatx_{t_{n+1}})\rVert_2\approx\sigma_{t_{n}}\lVert\bfeps_{\bftheta}^{\star}(\hatx_{t_{n}})\rVert_2=\lVert r_{\bftheta}(\hatx_{t_{n}}) - \hatx_{t_{n}}\rVert_2$. In this scenario, the denoising output $r_{\bftheta}(\hatx_{t_{n+1}})$ appears to be oscillating toward $r_{\bftheta}(\hatx_{t_{n}})$ around $\hatx_{t_n}$, akin to a simple gravity pendulum~\cite{young1996university}. The length of this pendulum effectively shortens by the coefficient $\sigma_{t_{n}}/\sigma_{t_{n+1}}$, starting from roughly $\sigma_T\sqrt{d}$. This specific structure is common to all sampling trajectories. 
Empirical verification of the constant magnitude of vector field is illustrated in Figure~\ref{fig:traj_len} and the trajectory length is illustrated in Figure~\ref{fig:cos_and_cumu}. 

\begin{figure*}[t]
    \centering
    \begin{subfigure}[t]{0.45\textwidth}
        \centering
        \includegraphics[width=\columnwidth]{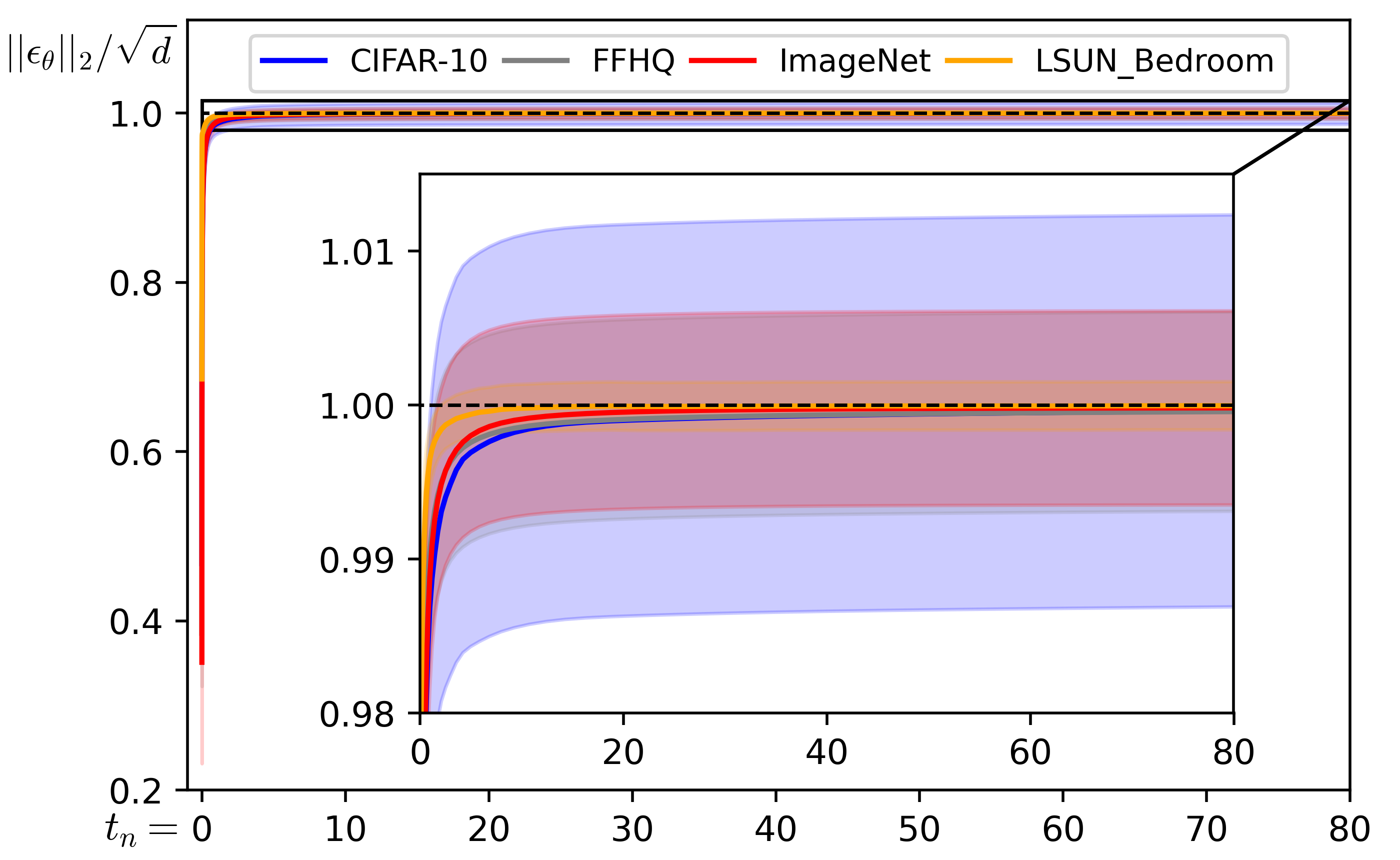}
        \caption{The $L^2$ norm of $\bfeps_{\bftheta}$.}
        \label{fig:eps}
    \end{subfigure}
    \quad
    \begin{subfigure}[t]{0.45\textwidth}
        \centering
        \includegraphics[width=\columnwidth]{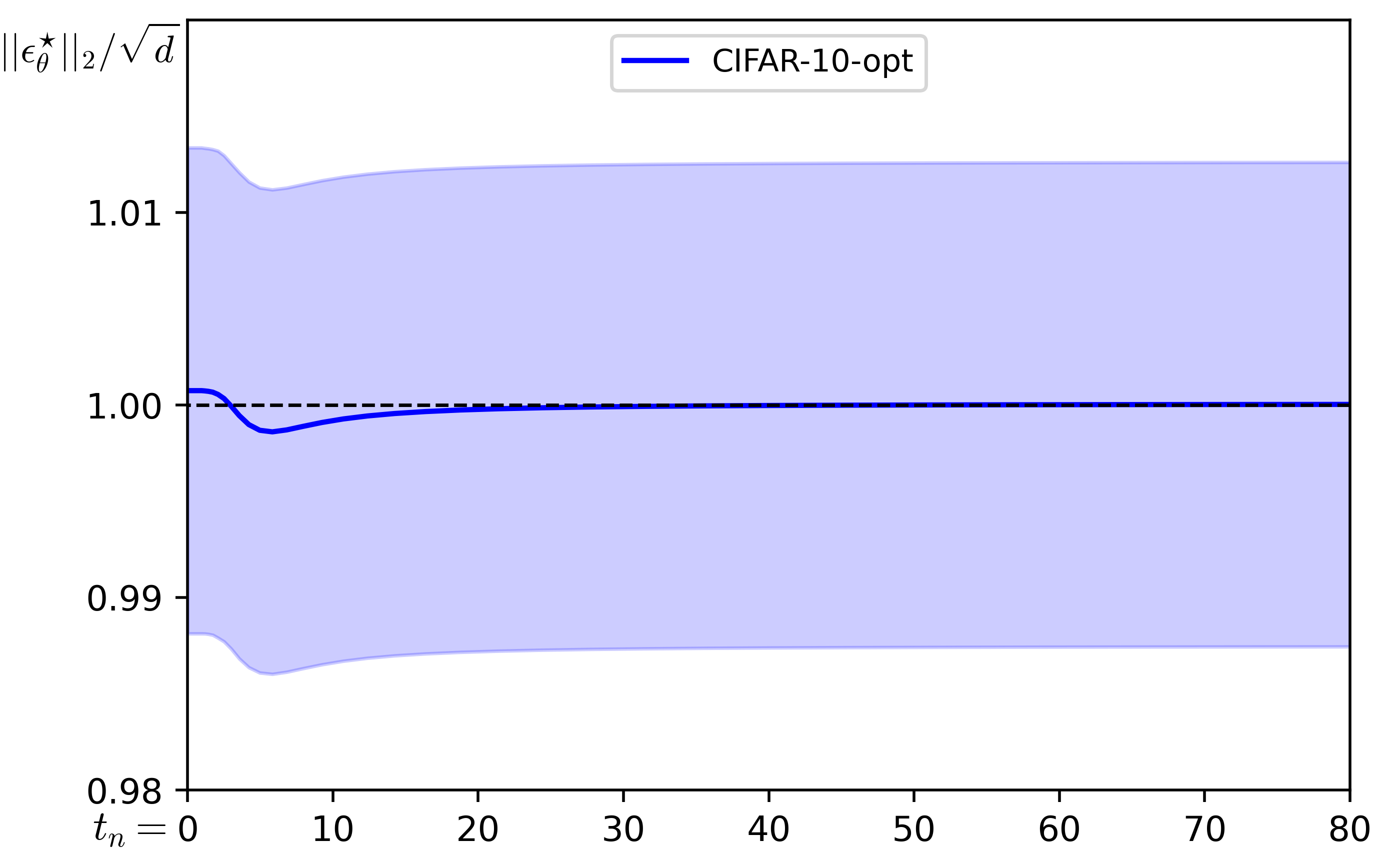}
        \caption{The $L^2$ norm of $\bfeps_{\bftheta}^{\star}$.}
        \label{fig:eps_opt}
    \end{subfigure}
    \caption{The optimal noise prediction $\lVert \bfeps_{\bftheta}^{\star}\rVert_2\approx \sqrt{d}$ in the whole sampling process, as the theoretical results guarantee. The actual noise prediction $\lVert \bfeps_{\bftheta}\rVert_2\approx \sqrt{d}$ in the most timestamps, but shrinks in the final stage (when the timestamp is very close to zero). Such norm shrinkage almost does not affect the trajectory length as the discretized time steps are very small in the final stage.}
    \label{fig:traj_len}
\end{figure*}

\begin{figure*}[t]
    \centering
    \begin{subfigure}[t]{0.45\textwidth}
        \centering
        \includegraphics[width=\columnwidth]{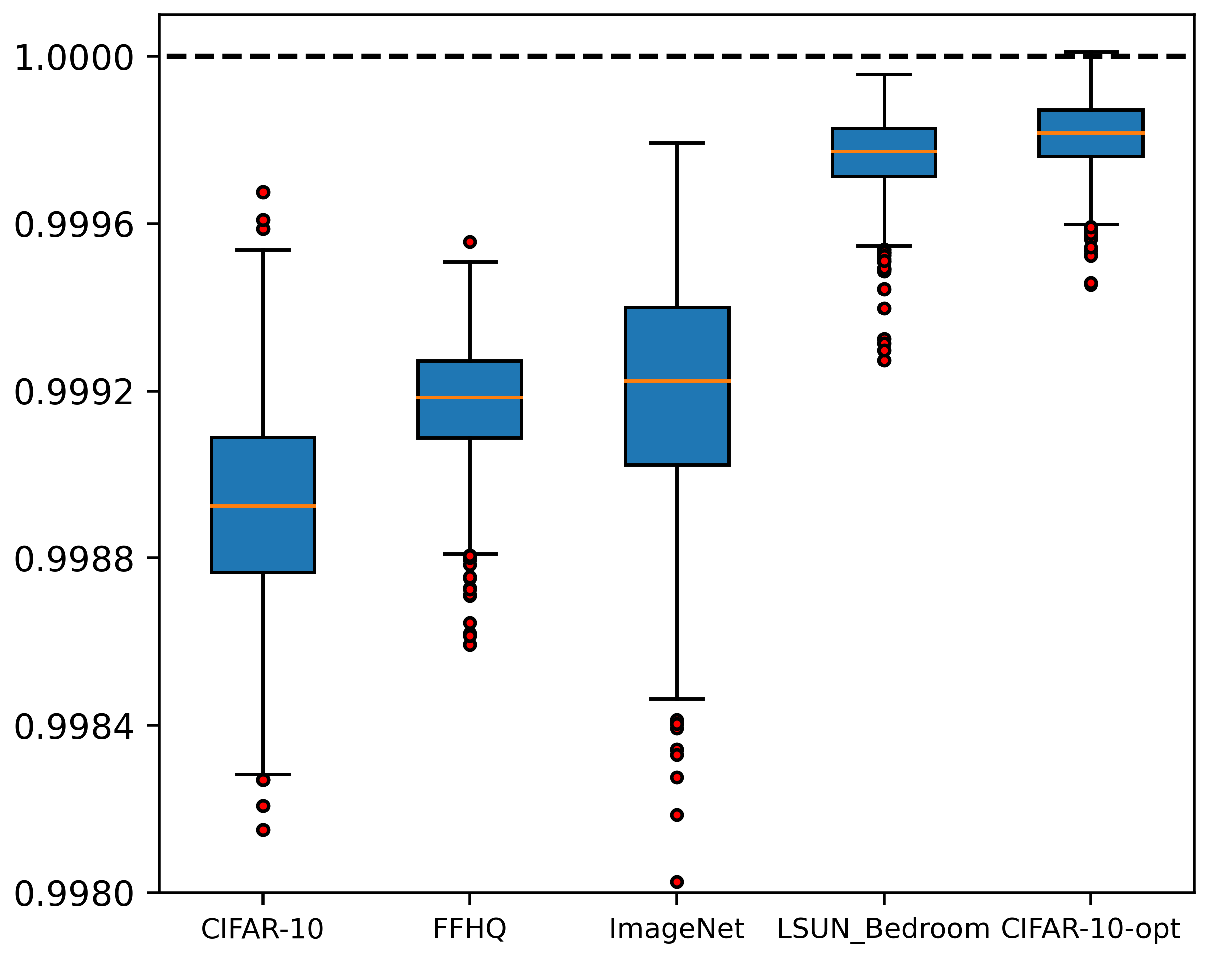}
        \label{fig:cumu}
    \end{subfigure}
    \begin{subfigure}[t]{0.51\textwidth}
        \centering
        \includegraphics[width=\columnwidth]{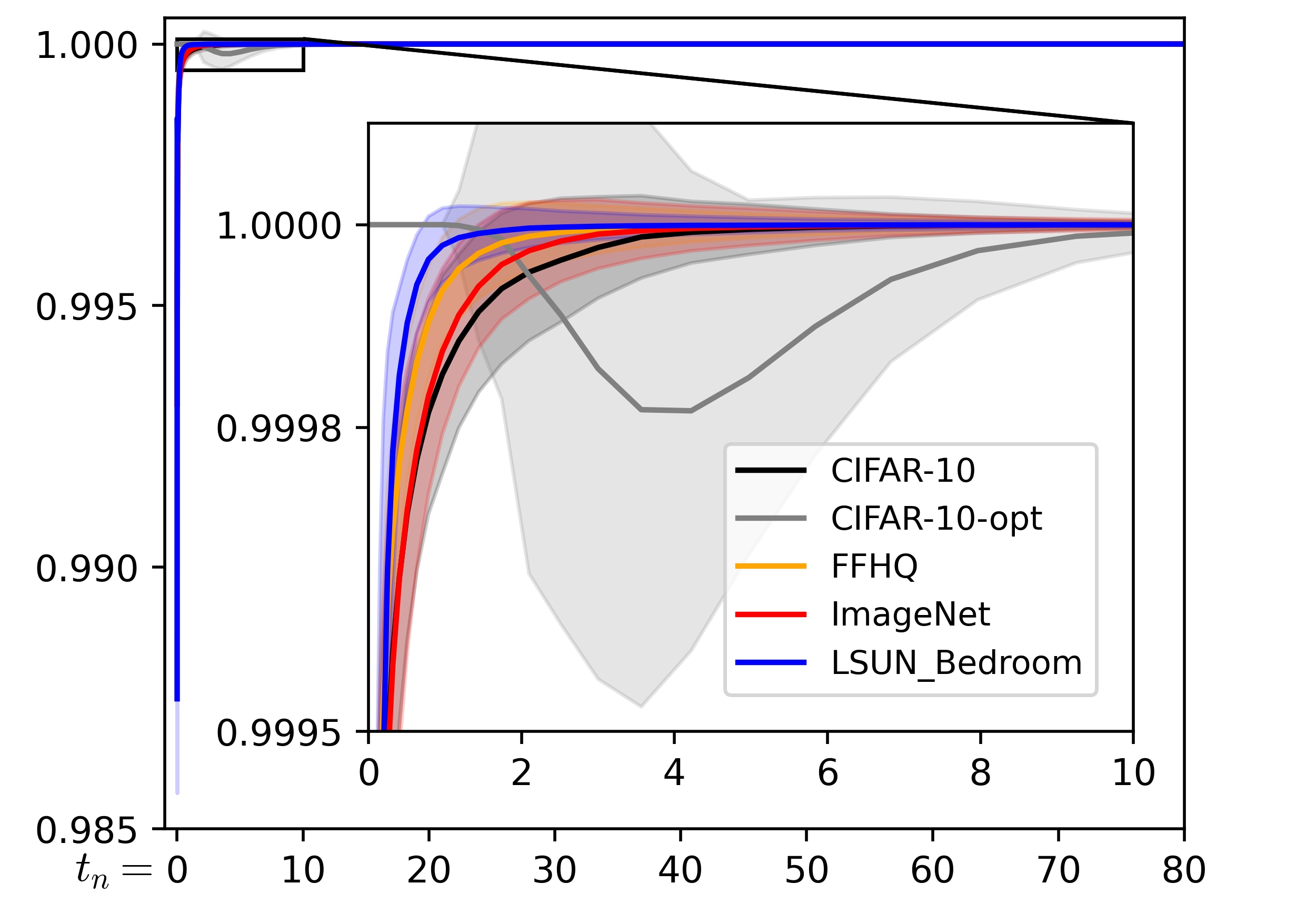}
        \label{fig:cos}
    \end{subfigure}
    \caption{\textit{Left:} The trajectory length is around $\sigma_T\sqrt{d}$ for both the optimal and actual diffusion models. \textit{Right:} The cosine value between two consecutive Euler steps is very small, which indicates the magnitude of each oscillation is extremely small (around 0$^\circ$).}
    \label{fig:cos_and_cumu}
\end{figure*}

\subsubsection{Monotone Increase in Sample Likelihood}
\label{subsubsec:likelihood}
In this section, we characterize the local behavior of the sampling process of diffusion models. To simplify notations, we denote the deviation of denoising output from the optimal counterpart as $d_1(\hatx_{t_{n}}) = \left\lVert r_{\bftheta}^{\star}(\hatx_{t_{n}}) - r_{\bftheta}(\hatx_{t_{n}}) \right\rVert_2$ and the distance between the optimal denoising output and the current position as $d_2(\hatx_{t_{n}}) = \left\lVert r_{\bftheta}^{\star}(\hatx_{t_{n}}) - \hatx_{t_{n}}\right\rVert_2$.
\begin{assumption}
    \label{assumption:ratio}
	We assume $d_1(\hatx_{t_{n}}) \le d_2(\hatx_{t_{n}})$ for all $\hatx_{t_{n}}$, $n\in[1, N]$ in the sampling trajectory.
\end{assumption}
This assumption requires that our learned denoising output $r_{\bftheta}(\hatx_{t_{n}})$ falls within a sphere centered at the optimal denoising output $r_{\bftheta}^{\star}(\hatx_{t_{n}})$ with a radius of $d_2(\hatx_{t_{n}})$. This radius controls the maximum deviation of the learned denoising output and shrinks during the sampling process. In practice, the assumption is relatively easy to satisfy for a well-trained diffusion model.
A visual illustration is provided in Figure~\ref{fig:meanshift}. 
\begin{proposition}
    \label{prop:likelihood}
    We have $p_{h}(r_{\bftheta}(\hatx_{t_n}))\ge p_{h}(\hatx_{t_n})$ and $p_{h}(\hatx_{t_{n-1}})\ge p_{h}(\hatx_{t_n})$ in terms of the KDE $p_{h}(\bfx)=\frac{1}{|\calI|}\sum_{i\in\calI}\calN(\bfx; \bfy_i, h^2\bfI)$ with any positive bandwidth $h$.
\end{proposition}

\begin{figure}
	\centering
    \begin{subfigure}[t]{0.3\textwidth}
        \centering
        \includegraphics[width=\columnwidth]{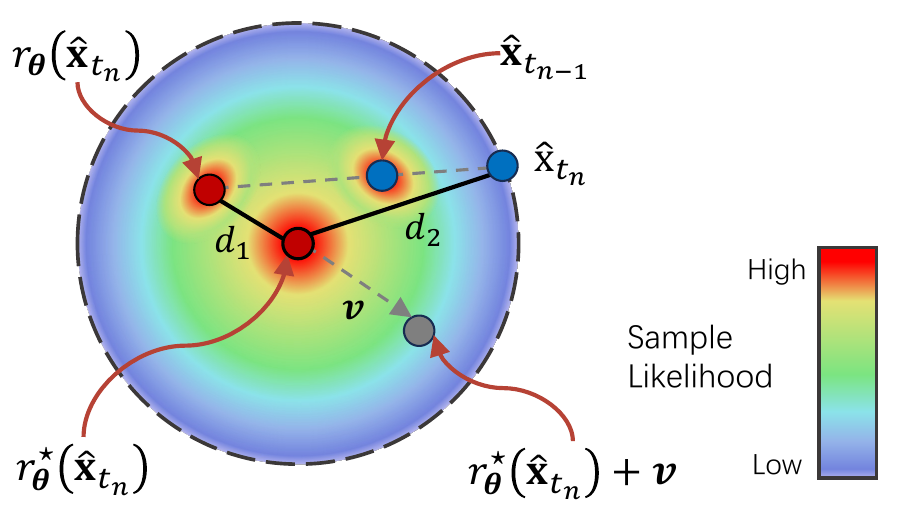}
    \end{subfigure}
    \begin{subfigure}[t]{0.40\textwidth}
        \centering
        \includegraphics[width=\columnwidth]{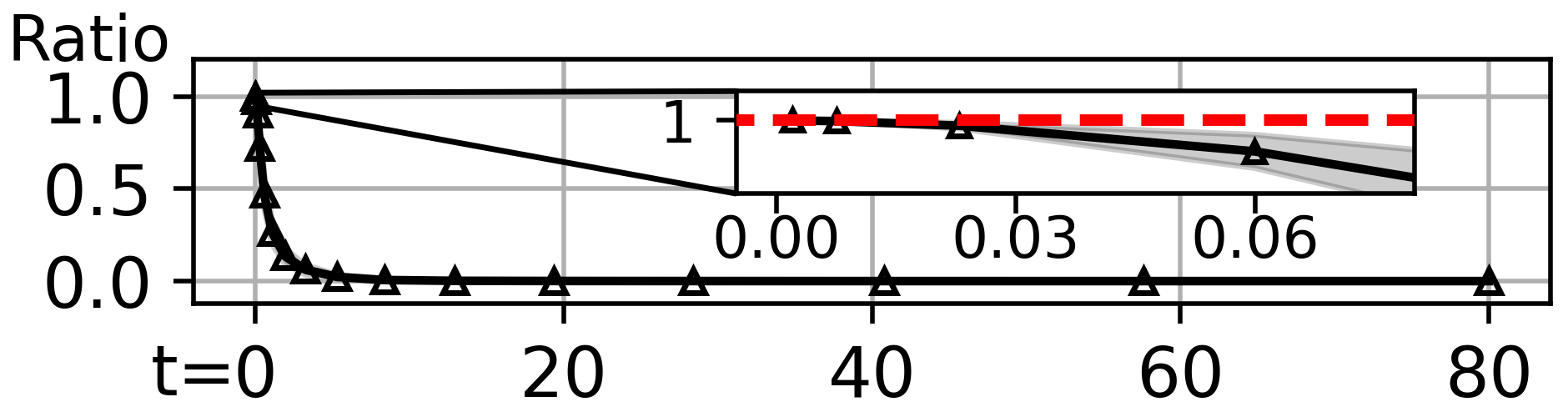}
    \end{subfigure}
	\caption{\textit{Left:} We have three likelihood rankings in the ODE-based diffusion sampling: (1) $p_{h}(r_{\bftheta}(\hatx_{t_n}))\ge p_{h}(\hatx_{t_n})$, (2) $p_{h}(\hatx_{t_{n-1}})\ge p_{h}(\hatx_{t_n})$, and (3) $p_{h}(r_{\bftheta}^{\star}(\hatx_{t_n}))\ge p_{h}(\hatx_{t_n})$. This figure complements the trajectory structure shown in Figure~\ref{fig:convex}. \textit{Right:} We compute the ratios $d_1(\hatx_{t_{n}})/d_2(\hatx_{t_{n}})$ along 50k sampling trajectories on CIFAR-10 (black curve), and find that these ratios are consistently lower than one from $t_N=80$ to $t_0=0.002$. This empirical evidence supports the validity of Assumption~\ref{assumption:ratio} in practice.}
	\label{fig:meanshift}
\end{figure}

\begin{proof}
	We first prove that given a random vector $\bfv$ falling within a sphere centered at the optimal denoising output $r_{\bftheta}^{\star}(\hatx_{t_n})$ with a radius of $\left\lVert r_{\bftheta}^{\star}(\hatx_{t_n})- \hatx_{t_n}\right\rVert_2$, \textit{i.e.}, $\left\lVert r_{\bftheta}^{\star}(\hatx_{t_n}) - \hatx_{t_n} \right\rVert_2 \ge \left\lVert\bfv  \right\rVert_2$, the sample likelihood is non-decreasing from $\hatx_{t_n}$ to  $r_{\bftheta}^{\star}(\hatx_{t_n})+\bfv$, \textit{i.e.}, $p_{h}(r_{\bftheta}^{\star}(\hatx_{t_n}) + \bfv)\ge p_{h}(\hatx_{t_n})$. Then, we provide two settings for $\bfv$ to finish the proof.
	
	The increase of sample likelihood from $\hatx_{t_n}$ to  $r_{\bftheta}^{\star}(\hatx_{t_n})+\bfv$ in terms of $p_{h}(\bfx)$ is
	\begin{equation}
		\begin{aligned}
			&\quad\, p_h(r_{\bftheta}^{\star}(\hatx_{t_n}) + \bfv) - p_h(\hatx_{t_n})=\frac{1}{|\calI|}\sum_{i}
			\left[\calN\left(r_{\bftheta}^{\star}(\hatx_{t_n}) + \bfv;\bfy_i, h^2\bfI\right)-\calN\left(\hatx_{t_n};\bfy_i, h^2\bfI\right)\right]\\
			&\overset{\text{(i)}}{\ge} \frac{1}{2h^2|\calI|}\sum_{i}\calN\left(\hatx_{t_n};\bfy_i, h^2\bfI\right)\left[\lVert \hatx_{t_n} - \bfy_i \rVert^2_2 
			-\lVert r_{\bftheta}^{\star}(\hatx_{t_n}) + \bfv - \bfy_i \rVert^2_2 \right]\\
			&=\frac{1}{2h^2|\calI|}\sum_{i}\calN\left(\hatx_{t_n};\bfy_i, h^2\bfI\right)\left[\lVert \hatx_{t_n}\rVert^2_2   - 2\hatx_{t_n}^{T}\bfy_i -\lVert r_{\bftheta}^{\star}(\hatx_{t_n}) + \bfv \rVert^2_2 + 2\left(r_{\bftheta}^{\star}(\hatx_{t_n}) + \bfv\right)^{T}\bfy_i  \right]\\
			&\overset{\text{(ii)}}{=}\frac{1}{2h^2|\calI|}\sum_{i}\calN\left(\hatx_{t_n};\bfy_i, h^2\bfI\right)\left[\lVert \hatx_{t_n}\rVert^2_2 - 2\hatx_{t_n}^{T}r_{\bftheta}^{\star}(\hatx_{t_n}) -\lVert r_{\bftheta}^{\star}(\hatx_{t_n}) + \bfv \rVert^2_2 + 2\left(r_{\bftheta}^{\star}(\hatx_{t_n}) + \bfv\right)^{T} r_{\bftheta}^{\star}(\hatx_{t_n}) \right]\\
			&=\frac{1}{2h^2|\calI|}\sum_{i}\calN\left(\hatx_{t_n};\bfy_i, h^2\bfI\right)\left[\lVert \hatx_{t_n}\rVert^2_2 - 2\hatx_{t_n}^T r_{\bftheta}^{\star}(\hatx_{t_n})+\lVert r_{\bftheta}^{\star}(\hatx_{t_n})\rVert^2_2 - \lVert \bfv \rVert^2_2 \right]\\
			&=\frac{1}{2h^2|\calI|}\sum_{i}\calN\left(\hatx_{t_n};\bfy_i, h^2\bfI\right)\left[\lVert r_{\bftheta}^{\star}(\hatx_{t_n}) - \hatx_{t_n}\rVert^2_2  - \lVert \bfv \rVert^2_2\right] \ge 0, 
		\end{aligned}
	\end{equation}
	where 
	(i) uses the definition of convex function $f(\bfx_2)\ge f(\bfx_1)+f^{\prime}(\bfx_1)(\bfx_2-\bfx_1)$ with $f(\bfx)=\exp \left(-\frac{1}{2}\lVert\bfx\rVert^2_2\right)$, $\bfx_1=\left(\hatx_{t_n}-\bfy_i \right)/h$ and $\bfx_2=\left(r_{\bftheta}^{\star}(\hatx_{t_n}) + \bfv-\bfy_i \right)/h$; 
	(ii) uses the relationship between two consecutive steps $\hatx_{t_n}$ and $r_{\bftheta}^{\star}(\hatx_{t_n})$ in mean shift with the Gaussian kernel (see \eqref{eq:mean_shift})
	\begin{equation}
		r_{\bftheta}^{\star}(\hatx_{t_n})=\bfm(\hatx_{t_n})=\sum_{i} \frac{\exp \left(-\lVert \hatx_{t_n} - \bfy_i \rVert^2_2/2h^2\right)}{\sum_{j}\exp \left(-\lVert \hatx_{t_n} - \bfy_j \rVert^2_2/2h^2\right)} \bfy_i,
	\end{equation}
	which implies the following equation also holds
	\begin{equation}
		\sum_i \calN\left(\hatx_{t_n};\bfy_i, h^2\bfI\right) \bfx_i = \sum_i \calN\left(\hatx_{t_n};\bfy_i, h^2\bfI\right) r_{\bftheta}^{\star}(\hatx_{t_n}).
	\end{equation}
	Since $\left\lVert r_{\bftheta}^{\star}(\hatx_{t_n}) - \hatx_{t_n} \right\rVert_2 \ge \left\lVert\bfv  \right\rVert_2$, or equivalently, $\left\lVert r_{\bftheta}^{\star}(\hatx_{t_n}) - \hatx_{t_n} \right\rVert^2_2 \ge \left\lVert\bfv  \right\rVert^2_2$, we conclude that the sample likelihood monotonically increases from $\hatx_{t_n}$ to $r_{\bftheta}^{\star}(\hatx_{t_n})+\bfv$ unless $\hatx_{t_n}=r_{\bftheta}^{\star}(\hatx_{t_n})+\bfv$, in terms of the kernel density estimate $p_{h}(\bfx)=\frac{1}{|\calI|}\sum_{i}\calN(\bfx; \bfy_i, h^2\bfI)$ with any positive bandwidth $h$. 
	
	We next provide two settings for $\bfv$, which trivially satisfy the condition $\left\lVert r_{\bftheta}^{\star}(\hatx_{t_n}) - \hatx_{t_n} \right\rVert_2 \ge \left\lVert\bfv  \right\rVert_2$
	, and have the following corollaries:
	\begin{itemize}
		\item $p_{h}(r_{\bftheta}(\hatx_{t_n}))\ge p_{h}(\hatx_{t_n})$, when $\bfv=r_{\bftheta}(\hatx_{t_n})-r_{\bftheta}^{\star}(\hatx_{t_n})$. 
		\item $p_{h}(\hatx_{t_{n-1}})\ge p_{h}(\hatx_{t_n})$, when  $\bfv=r_{\bftheta}(\hatx_{t_n})-r_{\bftheta}^{\star}(\hatx_{t_n})+\frac{t_{n-1}}{t_n}\left(\hatx_{t_n}-r_{\bftheta}(\hatx_{t_n})\right)$. 
	\end{itemize}
\end{proof}
Therefore, each sampling trajectory monotonically converges ($p_{h}(\hatx_{t_{n-1}})\ge p_{h}(\hatx_{t_n})$), and its coupled denoising trajectory converges even faster ($p_{h}(r_{\bftheta}(\hatx_{t_n}))\ge p_{h}(\hatx_{t_n})$) in terms of the sample likelihood.
Given an empirical data distribution, Proposition~\ref{prop:likelihood} applies to any marginal distributions of our forward SDE $\{p_t(\bfx)\}_{t=0}^T$, which should include the optimal bandwidth for the considered dataset. 

We can also obtain the standard monotone convergence property of mean shift~\cite{comaniciu2002mean} from Proposition~\ref{prop:likelihood} when diffusion models are trained to achieve the optimal parameters.
\begin{corollary}
	\label{corollary:meanshift}
	We have $p_{h}(\bfm(\hatx_{t_n}))\ge p_{h}(\hatx_{t_n})$, when $r_{\bftheta}(\hatx_{t_n})=r_{\bftheta}^{\star}(\hatx_{t_n})=\bfm (\hatx_{t_n})$. 
\end{corollary}

This connection also implies that once a diffusion model has converged to the optimum, all ODE trajectories will be uniquely determined and governed by a bandwidth-varying mean shift. In this case, the forward (encoding) process and backward (decoding) process only depend on the data distribution and the given noise distribution, regardless of model architectures or optimization algorithms. Such a property was previously referred to as \textit{uniquely identifiable encoding} and empirically verified in \cite{song2021sde}, while we theoretically characterize the optimum with annealed mean shift, and thus reveal the asymptotic behavior of diffusion models. 

Besides, the optimal diffusion models simply memorize the dataset and replay a certain discrete data point in sampling. 
We argue that in practice, a slight score deviation from the optimum ensures the generative ability of diffusion models while greatly alleviating the mode collapse issue. Experimental results are provided in the next section.

\subsubsection{Diagnosis of Score Deviation}
\label{subsec:score_deviation}
We simulate four new trajectories based on the optimal denoising output $r_{\bftheta}^{\star}(\cdot)$ to monitor the score deviation from the optimum. The first one is \textit{optimal sampling trajectory} $\{\hatx_{t}^{\star}\}$, where we generate samples as the sampling trajectory $\{\hatx_{t}\}$ by simulating \eqref{eq:epf_ode} but adopt $r^{\star}_{\bftheta}(\cdot)$ rather than $r_{\bftheta}(\cdot)$ for score estimation. The other three trajectories are simulated by tracking the (optimal) denoising output of each sample in $\{\hatx_{t}^{\star}\}$ or $\{\hatx_{t}\}$, and designated as $\{r_{\bftheta}(\hatx_{t}^{\star})\}$, $\{r^{\star}_{\bftheta}(\hatx_{t}^{\star})\}$, $\{r_{\bftheta}^{\star}(\hatx_{t})\}$. 
According to \eqref{eq:convex} and $t_0=0$, we have $\hatx_{t_0}^{\star}=r_{\bftheta}^{\star}(\hatx_{t_1}^{\star})$, and similarly, $\hatx_{t_0}=r_{\bftheta}(\hatx_{t_1})$. As $t\rightarrow 0$, $r^{\star}_{\bftheta}(\hatx_{t}^{\star})$ and $r_{\bftheta}^{\star}(\hatx_{t})$ serve as the approximate nearest neighbors of $\hatx_{t}^{\star}$ and $\hatx_{t}$ to the real data, respectively.

\begin{figure*}[t]
	\centering
	\begin{subfigure}[t]{\textwidth}
		\centering
		\includegraphics[width=0.85\textwidth]{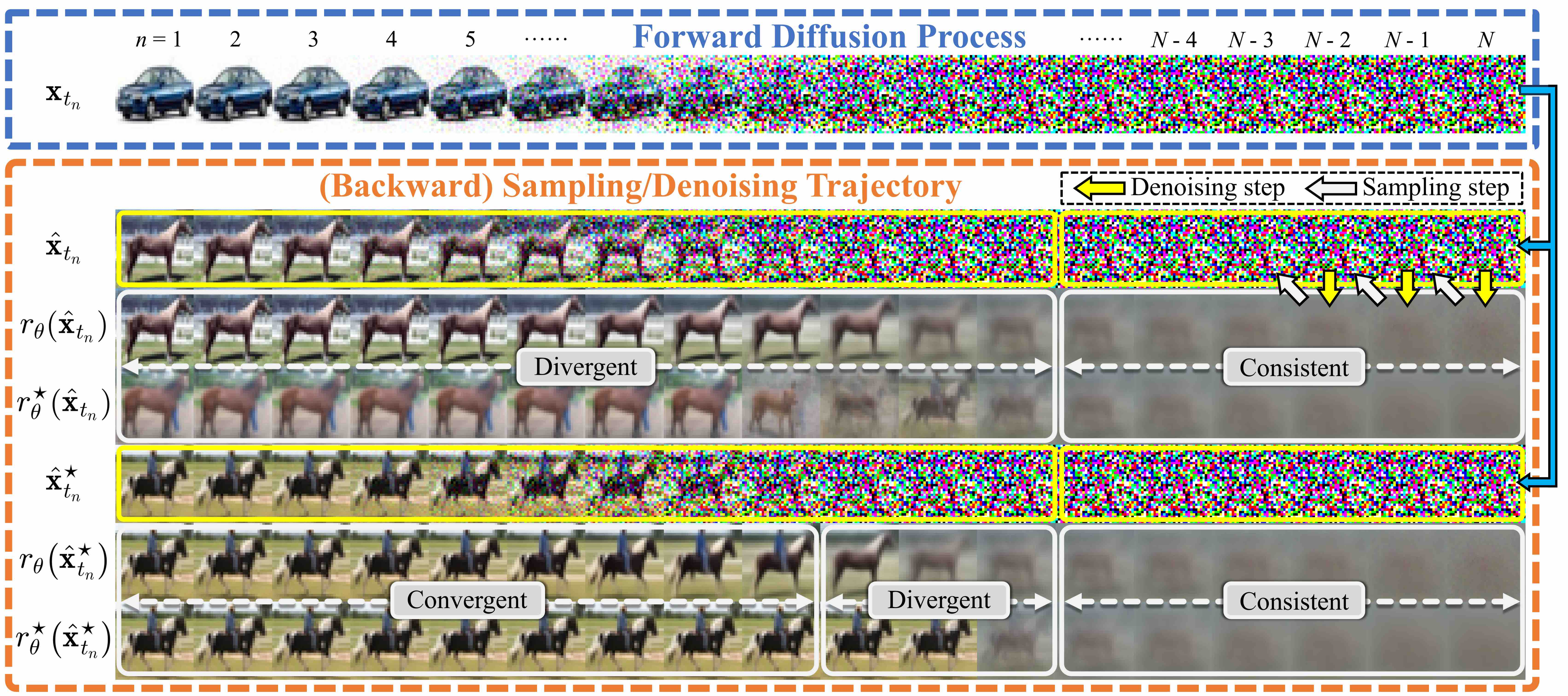}
	\end{subfigure}
	\hfil
	\begin{subfigure}[t]{\textwidth}
		\centering
		\includegraphics[width=0.85\textwidth]{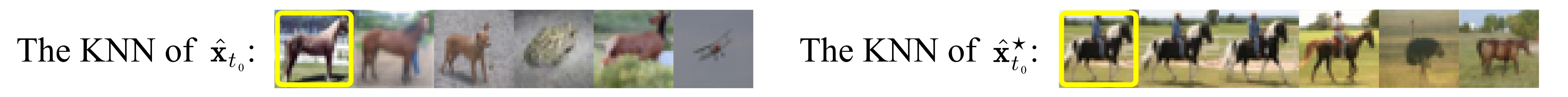}
	\end{subfigure}
	\caption{\textit{Top}: We visualize a forward diffusion process of a randomly-selected image to obtain its encoding $\hatx_{t_N}$ (first row) and simulate multiple trajectories starting from this encoding (other rows). 
		\textit{Bottom}: The k-nearest neighbors (k=5) of $\hatx_{t_0}$ and $\hatx_{t_0}^{\star}$ to real samples in the dataset.
	}
	\label{fig:diagnosis}
\end{figure*}
We calculate the deviation of denoising output to quantify the score deviation across all time steps using the $L^2$ distance, though they should differ by a factor $t^2$, and have the following observation: 
\textit{The learned score is well-matched to the optimal score in the large-noise region, otherwise they may diverge or almost coincide depending on different regions. }
In fact, our learned score has to moderately diverge from the optimum to guarantee the generative ability. Otherwise, the ODE-based sampling reduces to an approximate (single-step) annealed mean shift for global mode-seeking (see Section~\ref{subsubsec:meanshift}), and simply replays the dataset. 
As shown in Figure~\ref{fig:diagnosis}, the nearest sample of $\hatx_{t_0}^{\star}$ to the real data is almost the same as itself, which indicates the optimal sampling trajectory has a very limited ability to synthesize novel samples. Empirically, score deviation in a small region is sufficient to bring forth a decent generative ability. 

From the comparison of $\{r_{\bftheta}(\hatx_t^{\star})\}$, $\{r_{\bftheta}^{\star}(\hatx_t^{\star})\}$ sequences in Figures~\ref{fig:diagnosis} and~\ref{fig:score_deviation}, we can clearly see that \textit{along the optimal sampling trajectory}, the deviation between the learned denoising output $r_{\bftheta}(\cdot)$ and its optimal counterpart $r_{\bftheta}^{\star}(\cdot)$ behaves differently in three successive regions:
the deviation starts off as almost negligible (about $10<t\le 80$), gradually increases (about $3<t\le10$), and then drops down to a low level once again (about $0\le t\le3$).
This phenomenon was also validated by a recent work \cite{xu2023stable} with a different perspective. 
We further observe that \textit{along the sampling trajectory}, this phenomenon disappears and the score deviation keeps increasing (see $\{r_{\bftheta}(\hatx_t)\}$, $\{r_{\bftheta}^{\star}(\hatx_t)\}$ sequences in Figures~\ref{fig:diagnosis} and~\ref{fig:score_deviation}). Additionally, samples in the latter half of $\{r_{\theta}^{\star}(\hatx_t)\}$ appear almost the same as the nearest sample of $\hatx_{t_0}$ to the real data, as shown in Figure~\ref{fig:diagnosis}. This indicates that our score-based model strives to explore novel regions, and synthetic samples in the sampling trajectory are quickly attracted to a real-data mode but do not fall into it.

\begin{figure*}[t]
	\centering
	\includegraphics[width=0.85\textwidth]{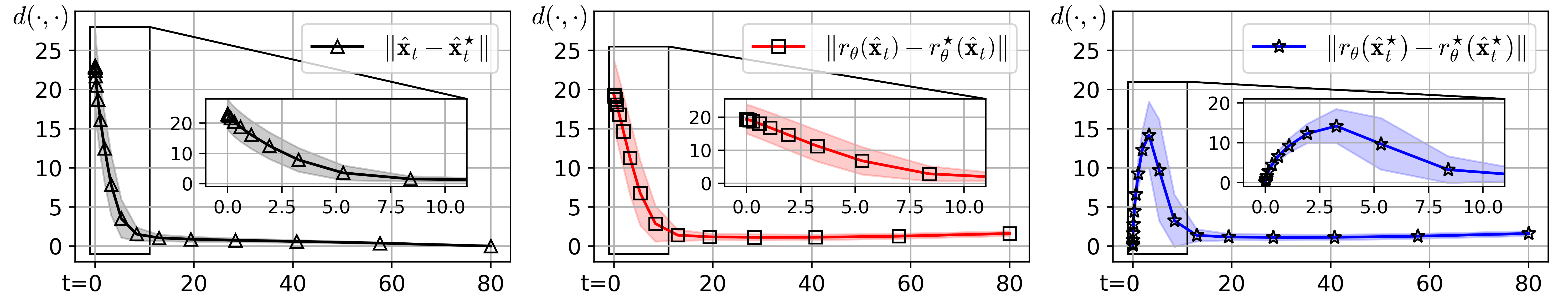}
	\caption{The deviation (measured by $L^2$ distance) of outputs from their corresponding optima.}
	\label{fig:score_deviation}
\end{figure*}

 \subsection{Additionally Experimental Results}
\label{subsec:appendix_results}

\subsubsection{Dynamic Programming}
\label{subsubsec:dp}

In this paper, as a simple illustration, we formulated the seeking of an optimal time schedule as an integer programming problem and solved it with the standard dynamic programming~\cite{cormen2022introduction}, as shown in Algorithm~\ref{alg:gits}. We also tried the Branch and Bound Algorithm and obtained similar results. Still, there also exist many different ways to determine the time schedule (e.g., using a trainable neural network) by leveraging our discovered trajectory regularity.

Specifically, we estimate the local truncation errors of PF-ODE based sampling trajectories, and compute the cost matrix for dynamic programming. In the sampling process of diffusion models, the global truncation error does not equal to the accumulation of local truncation error in each step. We thus introduce a coefficient $\gamma$ to compensate this effect.

\begin{algorithm}[tb]
    \caption{Geometry-Inspired Time Scheduling (standard dynamic programming, similar to~\cite{watson2021learning})}
    \label{alg:gits}
    \begin{algorithmic}[1]
        \STATE {\bfseries Input:} the number of teacher NFE $N_t$ for fine-grained sampling, the number of student NFE $N_{s}$ (NFE budget), the coefficient $\gamma$, a cost matrix with $ c_{jk} \coloneqq \calD(\hatx_{t_{N-j}\rightarrow t_{N-k}}, \bfx_{t_{N-j}\rightarrow t_{N-k}})$, $0 \leq j < k \leq N_t$. 
        \STATE Initialize $V_{ij} = +\infty$ $(0 \leq i \leq N_{s}, 0 \leq j \leq N_t)$, time schedule $\Gamma = [0]$, $m=0$
        \STATE $V_{i1} = c_{iN}$ $(0 \leq i \leq N_{s})$
        \FOR{$k=2$ {\bfseries to} $N_{s}$}
            \FOR{$j=0$ {\bfseries to} $N_t - 1$}
                \FOR{$i=j+1$ {\bfseries to} $N_t - 1$}
                    \STATE $V_{jk} = \min \{V_{jk}, \, \gamma c_{ji} + V_{ik-1} \}$
                \ENDFOR
            \ENDFOR
        \ENDFOR
        \STATE \# Fetch shortest path of $N_s$ steps
        \FOR{$k=N_{s}$ {\bfseries to} 1}
            \FOR{$j=m+1$ {\bfseries to} $N_t$}
                \IF{$V_{mk} == \gamma c_{mj} + V_{jk-1}$}
                    \STATE $\Gamma$.append($j$)
                    \STATE $m = j$
                    \STATE {\bfseries break}
                \ENDIF
            \ENDFOR
        \ENDFOR
        \STATE $\Gamma$.append($N_t$)
        \STATE {\bfseries Return:} $\Gamma$
    \end{algorithmic}
\end{algorithm}

\begin{table*}[t]
	\caption{Ablation study on the ``warmup'' sample size with iPNDM+GITS on CIFAR-10. $\dagger$ denotes that we search for a unique time schedule for each sample of 50k generated samples. This special case is more time-consuming while achieving similar results (due to the strong trajectory regularity).}
	\label{tab:warmup_size}
	\vskip 0.1in
	\centering
	\fontsize{7}{10}\selectfont
	\begin{tabular}{lcccccccc}
		\toprule
		\multirow{2}{*}{NFE} & \multicolumn{8}{c}{SAMPLE SIZE} \\
		\cmidrule{2-9}
		& 1$\dagger$ & 16 & 64 & 128 & \textbf{256} & 512 & 1024 & 2048 \\
		\midrule
		5  & 9.25 & 9.55$\pm$0.75 & 9.57$\pm$0.97 & 9.21 $\pm$0.44 & 8.84$\pm$0.30 & 8.81$\pm$0.04 & 8.89$\pm$0.11 & 8.88$\pm$0.12 \\
		6  & 5.12 & 5.36$\pm$0.61 & 5.16$\pm$0.28 & 4.99$\pm$0.18 & 5.03$\pm$0.25 & 5.20$\pm$0.27 & 5.01$\pm$0.19 & 4.92$\pm$0.08 \\
		8  & 3.13 & 3.25$\pm$0.13 & 3.22$\pm$0.08 & 3.28$\pm$0.10 & 3.27$\pm$0.11 & 3.30$\pm$0.11 & 3.29$\pm$0.08 & 3.33$\pm$0.10 \\
		10 & 2.41 & 2.46$\pm$0.11 & 2.46$\pm$0.05 & 2.45$\pm$0.05 & 2.46$\pm$0.04 & 2.45$\pm$0.04 & 2.44$\pm$0.05 & 2.44$\pm$0.05 \\
		\bottomrule
	\end{tabular}
	\label{tab:reb_ablation}
\end{table*}

\textbf{Discussion.}
\citet{watson2021learning} was the first one leveraging the idea of dynamic programming to re-allocate the time schedule in diffusion models. 
However, our motivation behind using DP significantly differs from this previous work. 
\citet{watson2021learning} exploit the fact that ELBO can be decomposed into separate KL terms and utilize DP to find the optimal discrete-time schedule that maximizes the training ELBO, but this strategy even worsens the sample quality, as admitted by authors. In contrast, we first found a strong trajectory regularity shared by all sampling trajectories (Section~\ref{sec:trajectory_visualization}), and then used some ``warmup'' samples to estimate the trajectory curvature to determine a better time schedule for the sampling of diffusion models. 

\begin{table}[t!]
	\caption{Used time (\textbf{seconds}) under different ``warmup'' sample sizes (iPNDM+GITS on CIFAR-10). ``warmup'' samples are generated by 60 NFE and the NFE budge for dynamic programming is 10.}
	\label{tab:vary_warmup_time}
	\vskip 0.1in
	\centering
	\fontsize{7}{10}\selectfont
	\begin{tabular}{lcccc}
		\toprule
		\multirow{2}{*}{SAMPLE SIZE} & sample & cost & dynamic & total \\
		& generation & matrix & programming & time (s) \\
		\midrule
		16             & 5.68   & 0.73   & 0.015  & 6.42   \\
		64             & 8.59   & 1.08   & 0.015  & 9.68   \\
		128            & 13.22  & 2.61   & 0.015  & 15.84  \\
		\rowcolor[gray]{0.9} 256 (default)   & 27.47  & 5.29   & 0.015  & 32.78  \\
		512            & 40.20  & 14.16  & 0.015  & 54.46  \\
		1024           & 75.90  & 29.58  & 0.015  & 105.58 \\
		2048           & 149.12 & 40.83  & 0.015  & 189.96 \\
		\bottomrule
	\end{tabular}
\end{table}

\textbf{More results.} We provide ablation studies on analytic first step (AFS) and sensitivity analysis of the coefficient $\gamma$ in Table~\ref{tab:fid-full}. We provide ablation studies on the grid size for generating the fine-grained sampling trajectory in Table~\ref{tab:ablation_teacher}, and the number of ``warmup'' sample size and used time in Table~\ref{tab:warmup_size} and Table~\ref{tab:vary_warmup_time}. 

\begin{table}[t!]
    \caption{Sample quality in terms of Fr\'echet Inception Distance (FID~\cite{heusel2017gans}, lower is better) on four datasets (resolutions ranging from $32\times32$ to $256\times256$). $\dagger$: After obtaining the DP schedule, we could further optimize the first time step with AFS, using the same 256 ``warmup'' samples as generating the fine-grained (teacher) sampling trajectory. 
    \textit{The default setting in our main submission dose not use AFS and keeps the coefficient in dynamic programming as 1.1 for LSUN Bedroom and 1.15 otherwise. Although as shown in the Table, the performance can be further improved by carefully tuning the coefficient and using AFS.}
    }
    \label{tab:fid-full}
    \vskip 0.1in
    \centering
    \fontsize{7}{10}\selectfont
    \begin{tabular}{lcccccccccc}
        \toprule
        \multirow{2}{*}{METHOD} & \multirow{2}{*}{Coeff} & \multirow{2}{*}{AFS$\dagger$} & \multicolumn{8}{c}{NFE} \\
        \cmidrule{4-11}
        & & & 3 & 4 & 5 & 6 & 7 & 8 & 9 & 10 \\
        \midrule
        \multicolumn{11}{l}{\fontsize{7}{1}\selectfont \textbf{CIFAR-10 32$\times$32}~\cite{krizhevsky2009learning}} \\
        \midrule
        DDIM~\cite{song2021ddim}          & -    & \usym{2613}      & 93.36 & 66.76 & 49.66 & 35.62 & 27.93 & 22.32 & 18.43 & 15.69 \\
        \rowcolor[gray]{0.9} DDIM + GITS  & 1.10 & \usym{2613}      & 88.68 & 46.88 & 32.50 & 22.04 & 16.76 & 13.93 & 11.57 & 10.09 \\
        \rowcolor[gray]{0.9} DDIM + GITS  (default) & 1.15 & \usym{2613}      & 79.67 & 43.07 & 28.05 & 21.04 & 16.35 & 13.30 & 11.62 & 10.37 \\
        \rowcolor[gray]{0.9} DDIM + GITS  & 1.20 & \usym{2613}      & 77.22 & 43.16 & 29.06 & 22.69 & 18.91 & 14.22 & 12.03 & 11.38 \\
        iPNDM~\cite{zhang2023deis}        & -    & \usym{2613}      & 47.98 & 24.82 & 13.59 & 7.05 & 5.08 & 3.69 & 3.17 & 2.77 \\
        \rowcolor[gray]{0.9} iPNDM + GITS & 1.10 & \usym{2613}      & 51.31 & 17.19 & 12.90 & 5.98 & 6.62 & 4.36 & 3.59 & 3.14 \\
        \rowcolor[gray]{0.9} iPNDM + GITS (default) & 1.15 & \usym{2613}      & 43.89 & 15.10 & 8.38  & 4.88 & 5.11 & 3.24 & 2.70 & 2.49 \\
        \rowcolor[gray]{0.9} iPNDM + GITS & 1.20 & \usym{2613}      & 42.06 & 15.85 & 9.33  & 7.13 & 5.95 & 3.28 & 2.81 & 2.71 \\
        \rowcolor[gray]{0.9} iPNDM + GITS & 1.10 & \usym{1F5F8}     & 34.22 & 11.99 & 12.44 & 6.08 & 6.20 & 3.53 & 3.48 & 2.91  \\
        \rowcolor[gray]{0.9} iPNDM + GITS & 1.15 & \usym{1F5F8}     & 29.63 & 11.23 & 8.08  & 4.86 & 4.46 & 2.92 & 2.46 & \textbf{2.27}  \\
        \rowcolor[gray]{0.9} iPNDM + GITS & 1.20 & \usym{1F5F8}     & \textbf{25.98} & \textbf{10.11} & \textbf{6.77}  & \textbf{4.29} & \textbf{3.43} & \textbf{2.70} & \textbf{2.42} & 2.28  \\
        \midrule
        \multicolumn{11}{l}{\fontsize{7}{1}\selectfont \textbf{FFHQ 64$\times$64}~\cite{karras2019style}} \\
        \midrule
        DDIM~\cite{song2021ddim}          & -    & \usym{2613}      & 78.21 & 57.48 & 43.93 & 35.22 & 28.86 & 24.39 & 21.01 & 18.37  \\
        \rowcolor[gray]{0.9} DDIM + GITS  & 1.10 & \usym{2613}      & 62.70 & 43.12 & 31.01 & 24.62 & 20.35 & 17.19 & 14.71 & 13.01  \\
        \rowcolor[gray]{0.9} DDIM + GITS  (default) & 1.15 & \usym{2613}      & 60.84 & 40.81 & 29.80 & 23.67 & 19.41 & 16.60 & 14.46 & 13.06  \\
        \rowcolor[gray]{0.9} DDIM + GITS  & 1.20 & \usym{2613}      & 59.64 & 40.56 & 30.29 & 23.88 & 20.07 & 17.36 & 15.40 & 14.05  \\
        iPNDM~\cite{zhang2023deis}        & -    & \usym{2613}      & 45.98 & 28.29 & 17.17 & 10.03 & 7.79 & 5.52 & 4.58 & 3.98  \\
        \rowcolor[gray]{0.9} iPNDM + GITS & 1.10 & \usym{2613}      & 34.82 & 18.75 & 13.07 & 7.79  & 8.30 & 4.76 & 5.36 & 3.47  \\
        \rowcolor[gray]{0.9} iPNDM + GITS (default) & 1.15 & \usym{2613}      & 33.09 & 17.04 & 11.22 & 7.00  & 6.72 & 4.52 & 4.33 & 3.62  \\
        \rowcolor[gray]{0.9} iPNDM + GITS & 1.20 & \usym{2613}      & 31.70 & 16.87 & 10.83 & 7.10  & 6.37 & 5.78 & 4.81 & 4.39  \\
        \rowcolor[gray]{0.9} iPNDM + GITS & 1.10 & \usym{1F5F8}     & 33.19 & 19.88 & 12.90 & 8.29 & 7.50 & 4.26 & 4.95 & \textbf{3.13}  \\
        \rowcolor[gray]{0.9} iPNDM + GITS & 1.15 & \usym{1F5F8}     & 30.39 & 15.78 & 10.15 & 6.86 & 5.97 & \textbf{4.09} & \textbf{3.76} & 3.24  \\
        \rowcolor[gray]{0.9} iPNDM + GITS & 1.20 & \usym{1F5F8}     & \textbf{26.41} & \textbf{13.59} & \textbf{8.85}  & \textbf{6.39} & \textbf{5.36} & 4.91 & 3.89 & 3.51  \\
        \midrule
        \multicolumn{11}{l}{\fontsize{7}{1}\selectfont \textbf{ImageNet 64$\times$64}~\cite{russakovsky2015ImageNet}} \\
        \midrule
        DDIM~\cite{song2021ddim}          & -    & \usym{2613}      & 82.96 & 58.43 & 43.81 & 34.03 & 27.46 & 22.59 & 19.27 & 16.72  \\
        \rowcolor[gray]{0.9} DDIM + GITS  & 1.10 & \usym{2613}      & 60.11 & 36.23 & 27.31 & 20.82 & 16.41 & 14.16 & 11.95 & 10.84  \\
        \rowcolor[gray]{0.9} DDIM + GITS (default)  & 1.15 & \usym{2613}      & 57.06 & 35.07 & 24.92 & 19.54 & 16.01 & 13.79 & 12.17 & 10.83  \\
        \rowcolor[gray]{0.9} DDIM + GITS  & 1.20 & \usym{2613}      & 54.24 & 34.27 & 24.67 & 19.46 & 16.66 & 14.15 & 13.41 & 11.87  \\
        iPNDM~\cite{zhang2023deis}        & -    & \usym{2613}      & 58.53 & 33.79 & 18.99 & 12.92 & 9.17 & 7.20 & 5.91 & 5.11  \\
        \rowcolor[gray]{0.9} iPNDM + GITS & 1.10 & \usym{2613}      & 36.18 & 19.64 & 13.18 & 9.58  & 7.68 & 6.44 & 5.24 & 4.59  \\
        \rowcolor[gray]{0.9} iPNDM + GITS (default) & 1.15 & \usym{2613}      & 34.47 & 18.95 & 10.79 & 8.43  & 6.83 & 5.82 & 4.96 & 4.48  \\
        \rowcolor[gray]{0.9} iPNDM + GITS & 1.20 & \usym{2613}      & 32.70 & 18.59 & 11.04 & 9.23  & 7.18 & 6.20 & 5.50 & 5.08  \\
        \rowcolor[gray]{0.9} iPNDM + GITS & 1.10 & \usym{1F5F8}     & 31.50 & 21.50 & 13.73 & 10.74 & 7.99 & 6.88 & 5.29 & 4.64  \\
        \rowcolor[gray]{0.9} iPNDM + GITS & 1.15 & \usym{1F5F8}     & 28.01 & 18.28 & 10.28 & 8.68  & 6.76 & 5.90 & 4.81 & \textbf{4.40}  \\
        \rowcolor[gray]{0.9} iPNDM + GITS & 1.20 & \usym{1F5F8}     & \textbf{26.41} & \textbf{16.41} & \textbf{9.85}  & \textbf{8.39}  & \textbf{6.44} & \textbf{5.64} & \textbf{4.79} & 4.47  \\
        \midrule
        \multicolumn{11}{l}{\fontsize{7}{1}\selectfont \textbf{LSUN Bedroom 256$\times$256}~\cite{yu2015lsun} (pixel-space)} \\
        \midrule
        DDIM~\cite{song2021ddim}          & -    & \usym{2613}      & 86.13 & 54.45 & 34.34 & 25.25 & 19.49 & 15.71 & 13.26 & 11.42  \\
        \rowcolor[gray]{0.9} DDIM + GITS  & 1.05 & \usym{2613}      & 81.77 & 36.89 & 27.46 & 18.78 & 13.60 & 12.23 & 10.29 & 8.77   \\
        \rowcolor[gray]{0.9} DDIM + GITS (default)  & 1.10 & \usym{2613}      & 61.85 & 35.12 & 22.04 & 16.54 & 13.58 & 11.20 & 9.82  & 9.04   \\
        \rowcolor[gray]{0.9} DDIM + GITS & 1.15 & \usym{2613}      & 60.11 & 31.02 & 23.65 & 17.18 & 13.42 & 12.61 & 10.89 & 10.57  \\
        iPNDM~\cite{zhang2023deis}        & -    & \usym{2613}      & 80.99 & 43.90 & 26.65 & 20.73 & 13.80 & 11.78 & 8.38 & 5.57  \\
        \rowcolor[gray]{0.9} iPNDM + GITS & 1.05 & \usym{2613}      & 59.02 & 24.71 & 19.08 & 12.77 & \textbf{8.19}  & \textbf{6.67}  & \textbf{5.58} & \textbf{4.83}  \\
        \rowcolor[gray]{0.9} iPNDM + GITS (default) & 1.10 & \usym{2613}      & 45.75 & 22.98 & \textbf{15.85} & \textbf{10.41} & 8.63  & 7.31  & 6.01 & 5.28  \\
        \rowcolor[gray]{0.9} iPNDM + GITS & 1.15 & \usym{2613}      & \textbf{44.78} & \textbf{21.67} & 17.29 & 11.52 & 9.59  & 8.82  & 7.22 & 5.97  \\
        \bottomrule
    \end{tabular}
	\vspace{-1em}
\end{table}

\begin{table}[t]
    \caption{Ablation study on the grid size used for the fine-grained sampling on CIFAR-10 with iPNDM. The DP coefficient is kept as 1.15.
    }
    \label{tab:ablation_teacher}
    \vskip 0.1in
    \centering
    \fontsize{7}{10}\selectfont
    \begin{tabular}{lccccccc}
        \toprule
        \multirow{2}{*}{GRID SIZE} & \multicolumn{6}{c}{NFE BUDGET} \\
        \cmidrule{2-8}
        & 4 & 5 & 6 & 7 & 8 & 9 & 10 \\
        \midrule
        11       & 20.88 & 10.15 & 5.11 & 4.63 & 3.16 & 2.78 & 2.77  \\
        21       & 16.22 & 9.87  & 4.83 & 3.76 & 3.39 & 3.20 & 2.81  \\
        41       & 15.34 & 9.34  & 4.83 & 5.54 & 3.01 & 2.66 & 2.53  \\
        \rowcolor[gray]{0.9} 61 (default)      & 15.10 & 8.38  & 4.88 & 5.11 & 3.24 & 2.70 & 2.49  \\
        81       & 15.74 & 8.57  & 5.09 & 5.38 & 3.10 & 2.93 & 2.38  \\
        101      & 15.03 & 8.72  & 5.02 & 5.19 & 3.12 & 2.81 & 2.41  \\
        \midrule
        iPNDM & 24.82 & 13.59 & 7.05 & 5.08 & 3.69 & 3.17 & 2.77  \\
        \bottomrule
    \end{tabular}
\end{table}

\begin{table}[t!]
    \caption{Time schedule on CIFAR-10 found by the dynamic programming. See the formula in texts for more details.} 
    \label{tab:schedule}
    \vskip 0.1in
    \centering
    \fontsize{8}{10}\selectfont
    \begin{tabular}{clc}
        \toprule
        NFE & TIME SCHEDULE & FID \\
        \midrule
        \multicolumn{3}{l}{\fontsize{7}{1}\selectfont \textbf{Uniform}} \\
        \midrule
        3  & [80.0000,  6.9503, 1.2867, 0.0020]                                                         & 50.44 \\
        4  & [80.0000, 11.7343, 2.8237, 0.8565, 0.0020]                                                 & 18.73 \\
        5  & [80.0000, 16.5063, 4.7464, 1.7541, 0.6502, 0.0020]                                         & 17.34 \\
        6  & [80.0000, 20.9656, 6.9503, 2.8237, 1.2867, 0.5272, 0.0020]                                 & 9.75 \\
        7  & [80.0000, 25.0154, 9.3124, 4.0679, 2.0043, 1.0249, 0.4447, 0.0020]                         & 12.50 \\
        8  & [80.0000, 28.6496, 11.7343, 5.4561, 2.8237, 1.5621, 0.8565, 0.3852, 0.0020]                 & 7.56 \\
        9  & [80.0000, 31.8981, 14.1472, 6.9503, 3.7419, 2.1599, 1.2867, 0.7382, 0.3401, 0.0020]         & 10.60 \\
        10 & [80.0000, 34.8018, 16.5063, 8.5141, 4.7464, 2.8237, 1.7541, 1.0985, 0.6502, 0.3047, 0.0020] & 7.35 \\
        \midrule
        \multicolumn{3}{l}{\fontsize{7}{1}\selectfont \textbf{LogSNR}} \\
        \midrule
        3  & [80.0000,  2.3392, 0.0684, 0.0020]                                                         & 88.38 \\
        4  & [80.0000,  5.6569, 0.4000, 0.0283, 0.0020]                                                 & 35.59 \\
        5  & [80.0000,  9.6090, 1.1542, 0.1386, 0.0167, 0.0020]                                         & 19.87 \\
        6  & [80.0000, 13.6798, 2.3392, 0.4000, 0.0684, 0.0117, 0.0020]                                 & 10.68 \\
        7  & [80.0000, 17.6057, 3.8745, 0.8527, 0.1876, 0.0413, 0.0091, 0.0020]                         & 6.56 \\
        8  & [80.0000, 21.2732, 5.6569, 1.5042, 0.4000, 0.1064, 0.0283, 0.0075, 0.0020]                 & 4.74 \\
        9  & [80.0000, 24.6462, 7.5929, 2.3392, 0.7207, 0.2220, 0.0684, 0.0211, 0.0065, 0.0020]         & 3.53 \\
        10 & [80.0000, 27.7258, 9.6090, 3.3302, 1.1542, 0.4000, 0.1386, 0.0480, 0.0167, 0.0058, 0.0020] & 2.94 \\
        \midrule
        \multicolumn{3}{l}{\fontsize{7}{1}\selectfont \textbf{Polynomial ($\rho = 7$)}} \\
        \midrule
        3  & [80.0000,  9.7232, 0.4700, 0.0020]                                                         & 47.98 \\
        4  & [80.0000, 17.5278, 2.5152, 0.1698, 0.0020]                                                 & 24.82 \\
        5  & [80.0000, 24.4083, 5.8389, 0.9654, 0.0851, 0.0020]                                         & 13.59 \\
        6  & [80.0000, 30.1833, 9.7232, 2.5152, 0.4700, 0.0515, 0.0020]                                 & 7.05 \\
        7  & [80.0000, 34.9922, 13.6986, 4.6371, 1.2866, 0.2675, 0.0352, 0.0020]                         & 5.08 \\
        8  & [80.0000, 39.0167, 17.5278, 7.1005, 2.5152, 0.7434, 0.1698, 0.0261, 0.0020]                 & 3.69 \\
        9  & [80.0000, 42.4152, 21.1087, 9.7232, 4.0661, 1.5017, 0.4700, 0.1166, 0.0204, 0.0020]         & 3.17 \\
        10 & [80.0000, 45.3137, 24.4083,12.3816, 5.8389, 2.5152, 0.9654, 0.3183, 0.0851, 0.0167, 0.0020] & 2.77 \\
        \midrule
        \multicolumn{3}{l}{\fontsize{7}{1}\selectfont \textbf{GITS (ours)}} \\
        \midrule
        3  & [80.0000,  3.8811, 0.9654, 0.0020]                                                         & 43.89 \\
        4  & [80.0000,  5.8389, 1.8543, 0.4700, 0.0020]                                                 & 15.10 \\
        5  & [80.0000,  6.6563, 2.1632, 0.8119, 0.2107, 0.0020]                                         & 8.38 \\
        6  & [80.0000, 10.9836, 3.8811, 1.5840, 0.5666, 0.1698, 0.0020]                                 & 4.88 \\
        7  & [80.0000, 12.3816, 3.8811, 1.5840, 0.5666, 0.1698, 0.0395, 0.0020]                         & 3.76 \\
        8  & [80.0000, 10.9836, 3.8811, 1.8543, 0.9654, 0.4700, 0.2107, 0.0665, 0.0020]                 & 3.24 \\
        9  & [80.0000, 12.3816, 4.4590, 2.1632, 1.1431, 0.5666, 0.2597, 0.1079, 0.0300, 0.0020]         & 2.70 \\
        10 & [80.0000, 12.3816, 4.4590, 2.1632, 1.1431, 0.5666, 0.3183, 0.1698, 0.0665, 0.0225, 0.0020] & 2.49 \\
        \bottomrule
    \end{tabular}
\end{table}

\textbf{Time schedule.}
The uniform schedule is widely used in papers using a DDPM~\cite{ho2020ddpm} backbone. Following the implementation in EDM~\cite{karras2022edm}, we transfer this schedule from its original range $[\epsilon_s, 1]$ to $[t_0, t_N]$ where $\epsilon_s = 0.001$, $t_0=0.002$ and $t_N=80$. We first uniformly sample $\tau_n$ ($n\in [0,N]$) from $[\epsilon_s, 1]$ and then calculate $t_n$ by
\begin{equation}
    t_n = \sqrt{e^{\frac{1}{2}{\beta_d}{\tau_n}^2 + \beta_{\min}\tau_n} - 1}
\end{equation}
where
\begin{equation}
    \beta_d = \frac{2}{\epsilon_s - 1} \frac{\log(1+{t_0}^2)}{\epsilon_s} - \log(1+{t_N}^2), \quad
    \beta_{\min} = \log(1+{t_N}^2) - \frac{1}{2}\beta_d.
\end{equation}

The losSNR time schedule is proposed for fast sampling in DPM-Solver~\cite{lu2022dpm}. We first uniformly sample $\lambda_n$ ($n\in [0,N]$) from $[\lambda_{\min}, \lambda_{\max}]$ where $\lambda_{\min} = -\log t_N$ and $\lambda_{\max} = -\log t_0$. The logSNR schedule is given by $t_n = e^{-\lambda_n}$.

The polynomial time schedule $t_n=(t_0^{1/\rho}+\frac{n}{N}(t_N^{1/\rho}-t_0^{1/\rho}))^{\rho}$ is proposed in EDM~\cite{karras2022edm}, where $t_0=0.002$, $t_N=80$, $n\in [0,N]$, and $\rho=7$.

The optimized time schedules for stable diffusion 1.5 in Figure~\ref{fig:sd_ays} include
\begin{itemize}
	\item AYS~\cite{sabour2024align}: [999, 850, 736, 645, 545, 455, 343, 233, 124, 24, 0].
	\item GITS: [999, 783, 632, 483, 350, 233, 133, 67, 33, 17, 0].
\end{itemize} 

\subsubsection{Regularity of Sampling Trajectory}

Figure~\ref{fig:app_deviation} provides more experiments about 1-D projections on  CIFAR-10, FFHQ, and LSUN Bedroom. Figure~\ref{fig:recon_cifar10} provides more experiments about Multi-D projections on CIFAR-10 and FFHQ. Figures~\ref{fig:cifar10_comparison}, \ref{fig:ffhq_comparison}, \ref{fig:imagenet_comparison}, and \ref{fig:bedroom_comparison} visualize more generated samples on four datasets.

\begin{figure*}[t!]
    \centering
    \begin{subfigure}[t]{0.33\textwidth}
        \includegraphics[width=\columnwidth]{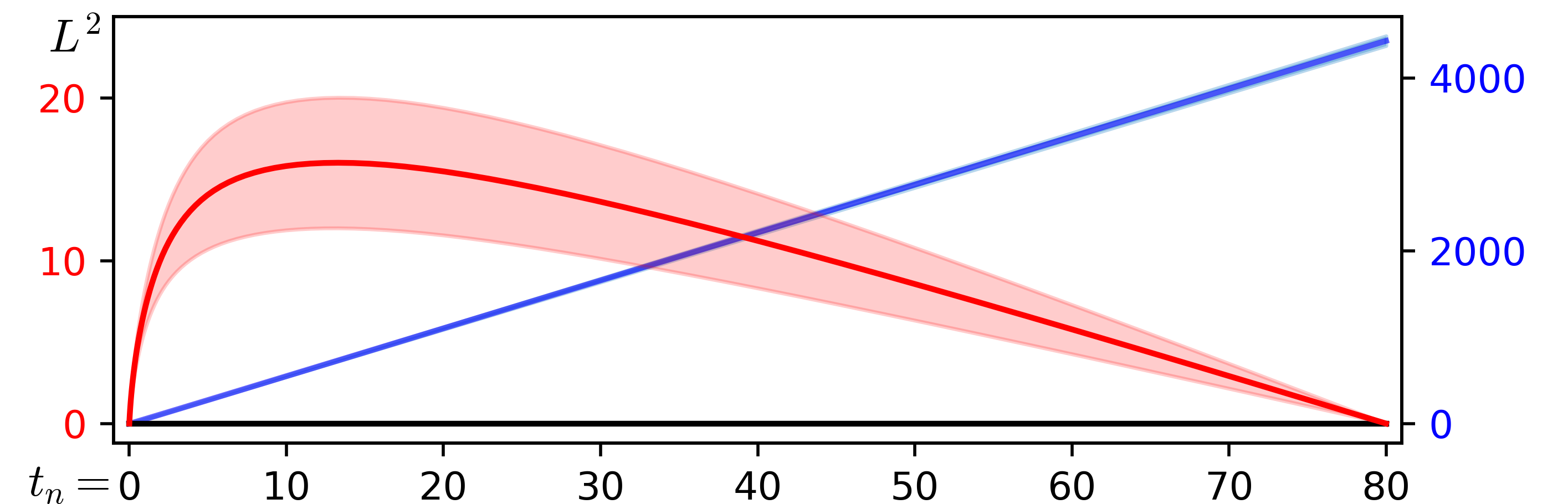}
        \caption{CIFAR-10 32$\times$32.}
        \label{fig:deviation_cifar10}
    \end{subfigure}
    \hfill
    \begin{subfigure}[t]{0.33\textwidth}
        \includegraphics[width=\columnwidth]{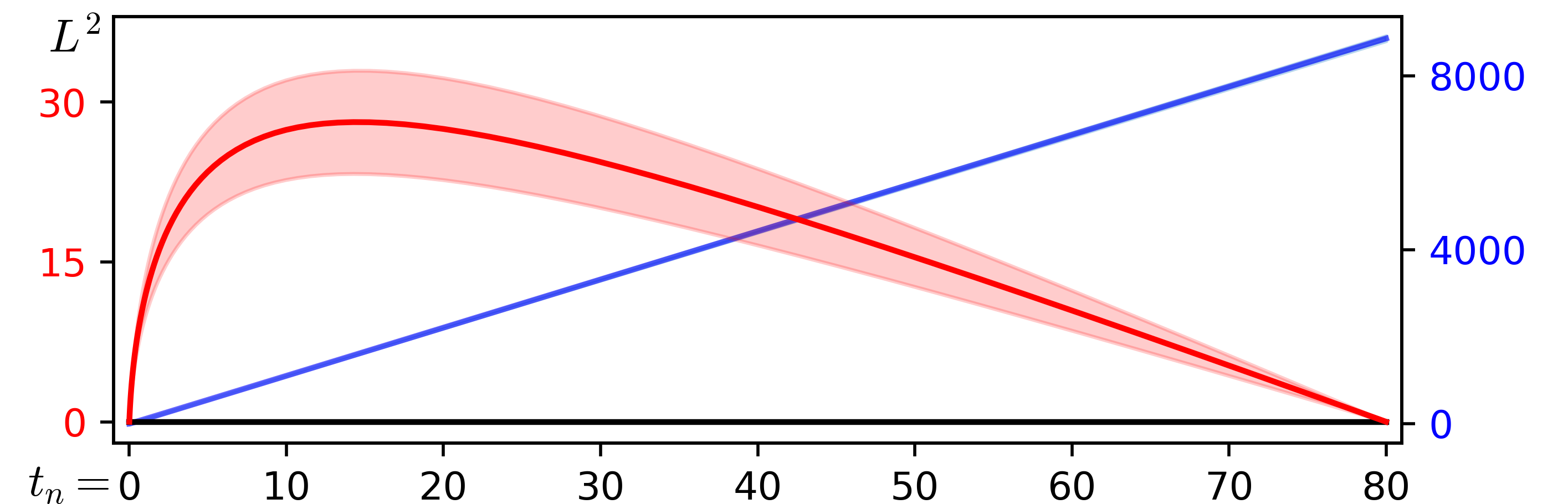}
        \caption{FFHQ 64$\times$64.}
        \label{fig:deviation_ffhq}
    \end{subfigure}
    \hfill
    \begin{subfigure}[t]{0.33\textwidth}
        \includegraphics[width=\columnwidth]{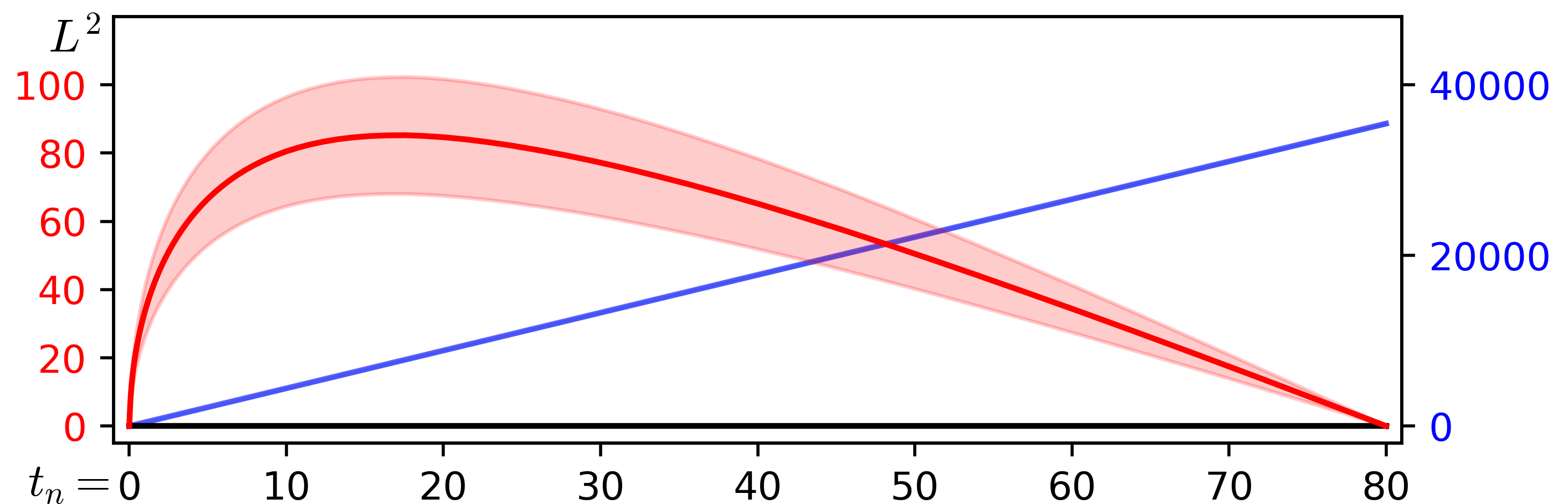}
        \caption{LSUN Bedroom 256$\times$256.}
        \label{fig:deviation_bedroom}
    \end{subfigure}
    \hfill
    \caption{
    Trajectory deviation (red curve) compared to the sample distance (blue curve) in the sampling process starting from $t_{N}=80$ to $t_0=0.002$. Each trajectory is simulated with Euler method and 100 NFE. The results are averaged by 5k generated samples.
    }
    \label{fig:app_deviation}
\end{figure*}

\begin{figure*}[t!]
    \centering
    \begin{subfigure}[t]{0.29\textwidth}
        \includegraphics[width=\columnwidth]{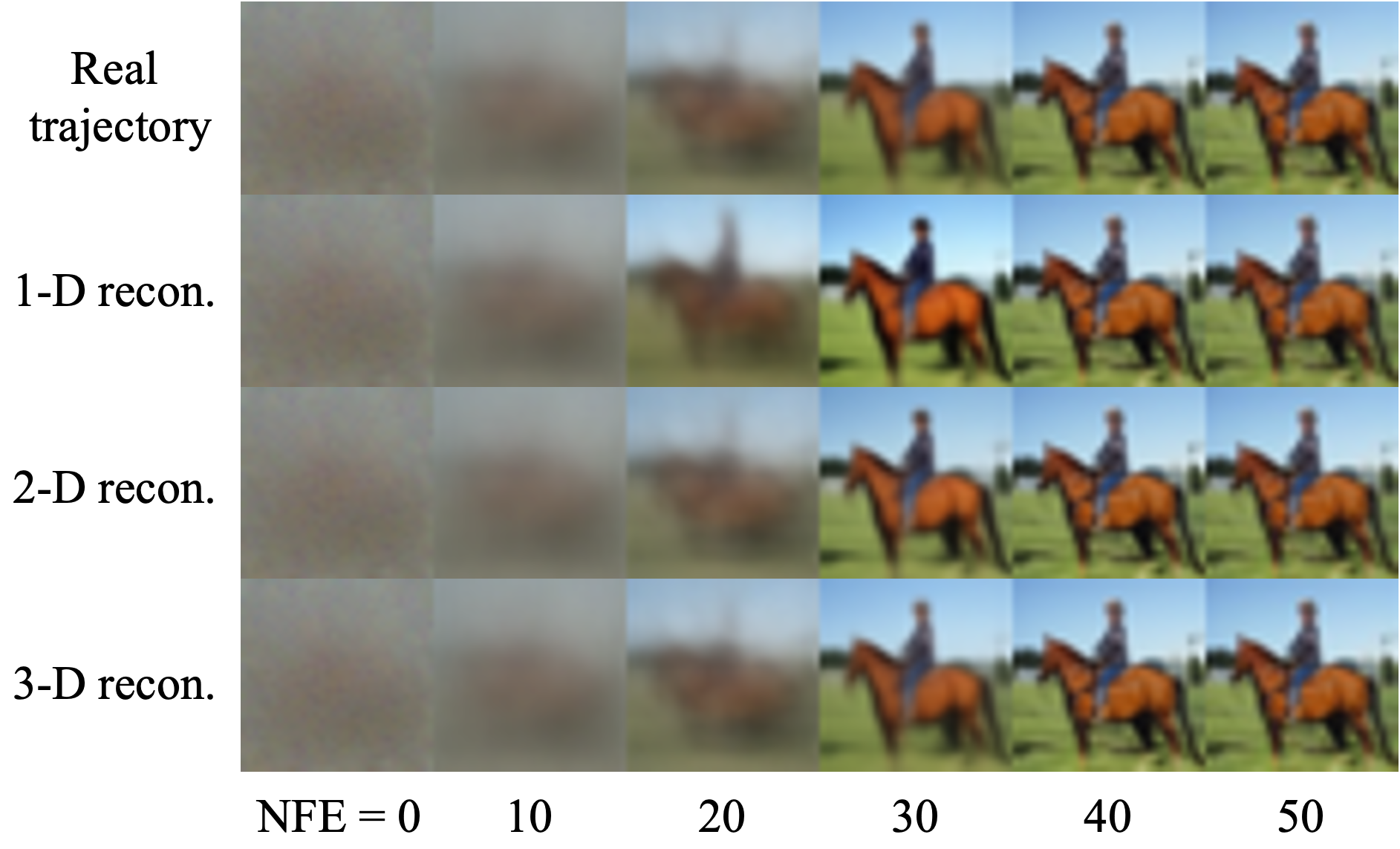}
        \caption{Visual Comparison.}
        \label{fig:recon_visual_cifar10}
    \end{subfigure}
    \hfill
    \begin{subfigure}[t]{0.16\textwidth}
        \includegraphics[width=\columnwidth]{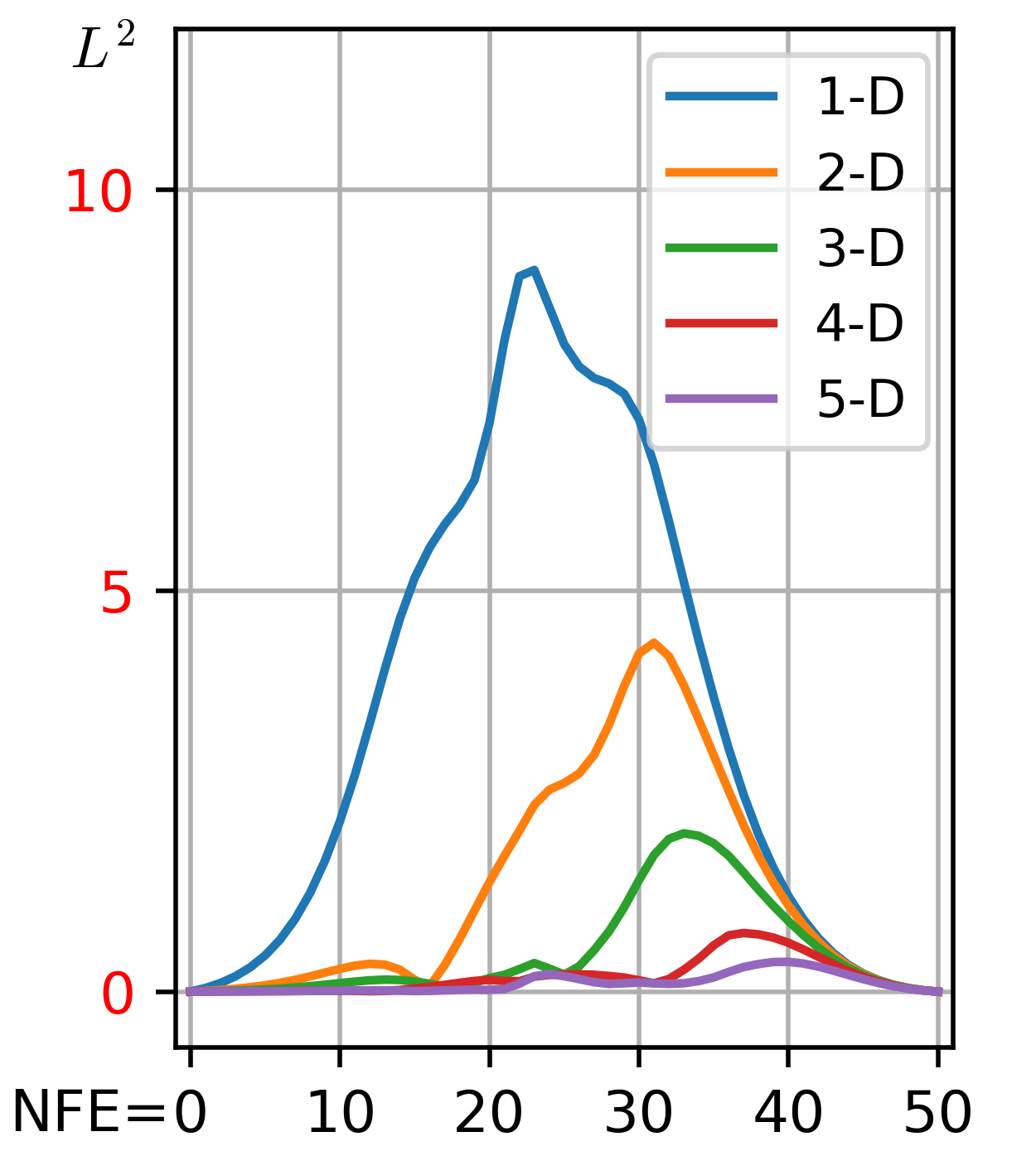}
        \caption{$L^2$ error.}
        \label{fig:recon_l2_cifar10}
    \end{subfigure}
    \hfill
    \begin{subfigure}[t]{0.29\textwidth}
    	\includegraphics[width=\columnwidth]{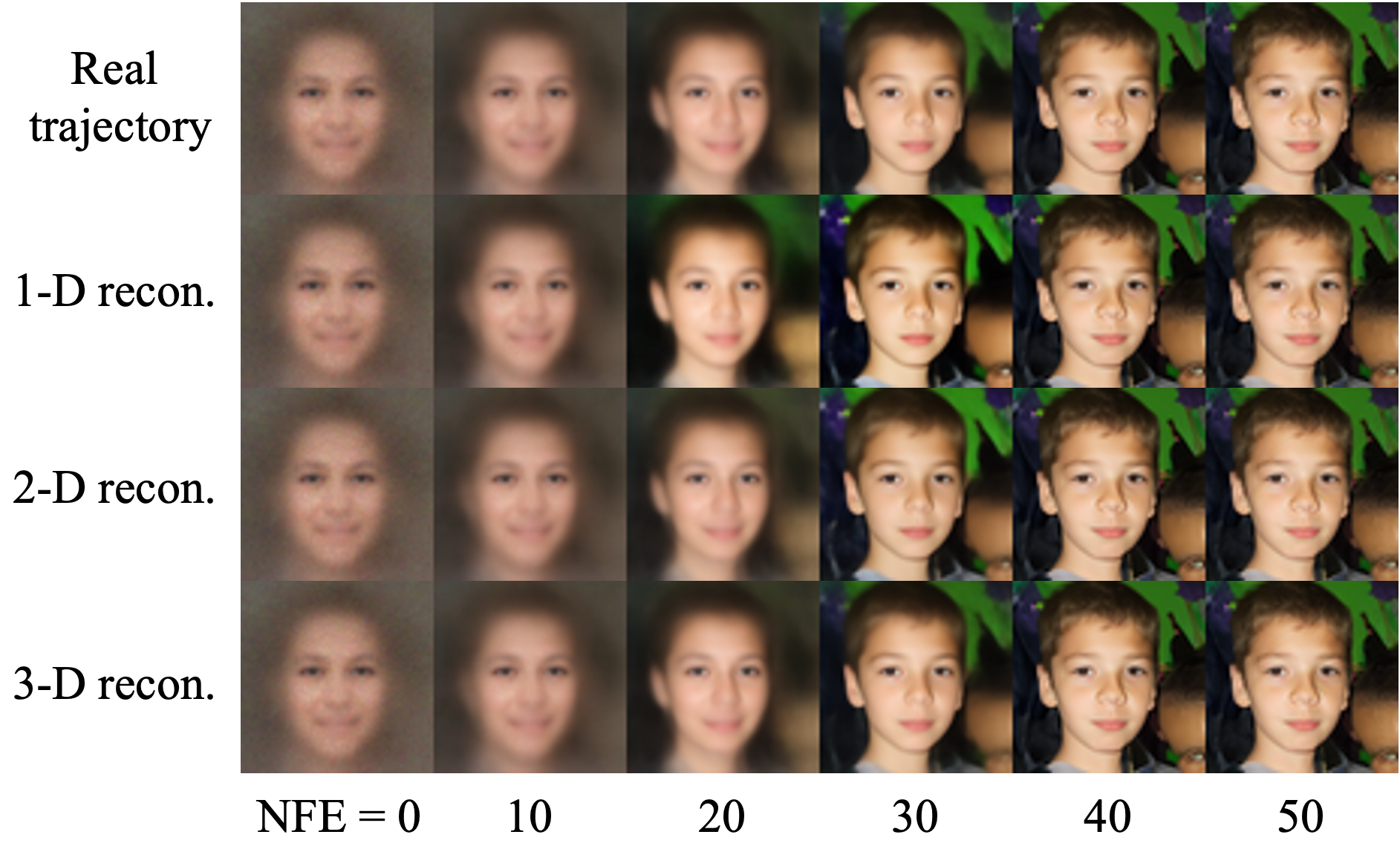}
    	\caption{Visual Comparison.}
    	\label{fig:recon_visual_ffhq}
    \end{subfigure}
    \hfill
    \begin{subfigure}[t]{0.16\textwidth}
    	\includegraphics[width=\columnwidth]{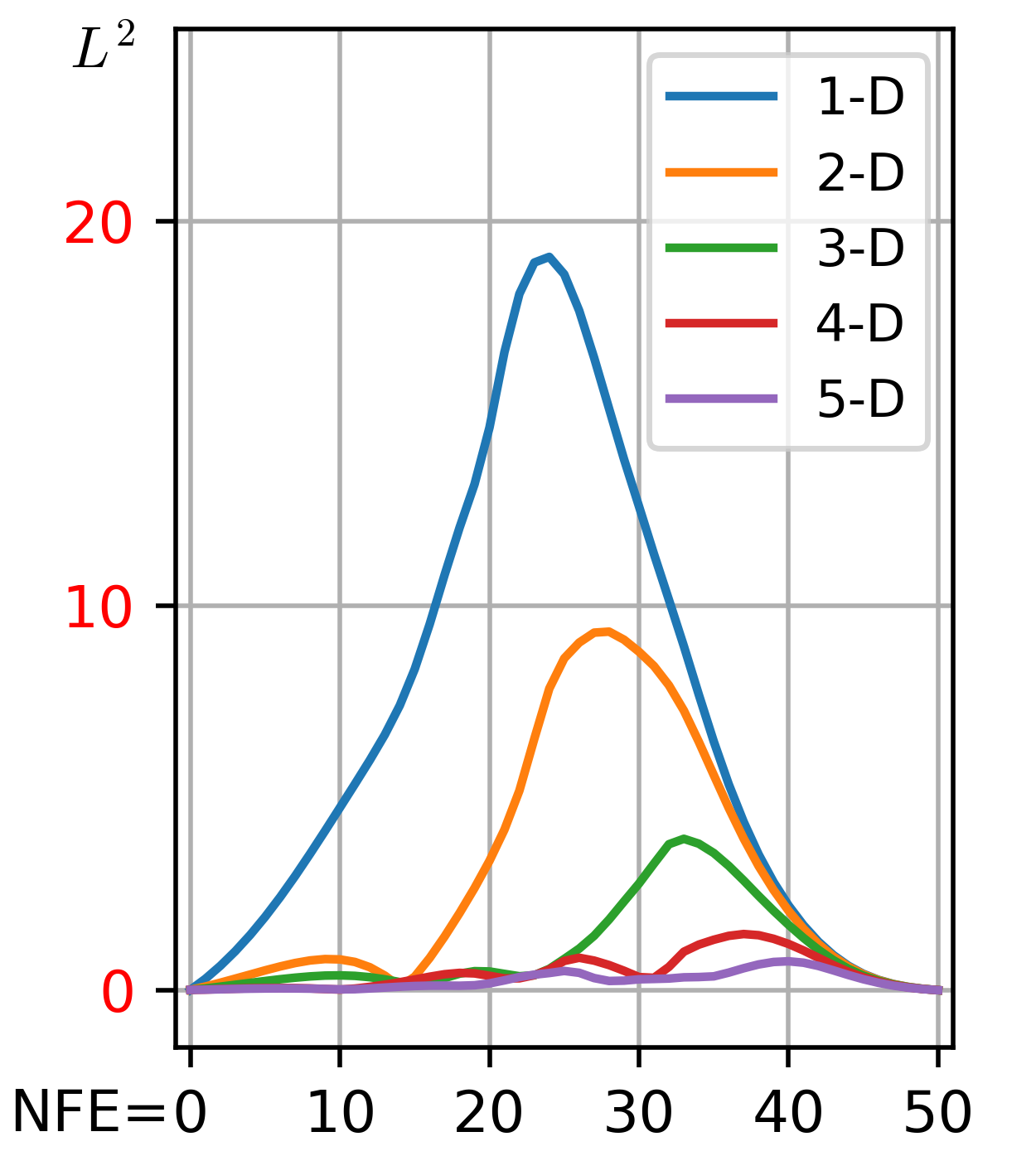}
    	\caption{$L^2$ error.}
    	\label{fig:recon_l2_ffhq}
    \end{subfigure}
    \hfill
    \caption{
    (a/c) The visual comparison of trajectory reconstruction on CIFAR-10, and FFHQ. We reconstruct the real sampling trajectory (top row) using $\hatx_{t_N}-\hatx_{t_0}$ (1-d recon.) along with its top 1 or 2 principal components (2-D or 3-D recon.). To amplify the visual difference, we present the denoising outputs of these trajectories. 
    (b/d) We compute the $L^2$ distance between the real trajectory the reconstructed trajectories up to 5-D reconstruction. 
    }
    \label{fig:recon_cifar10}
\end{figure*}

%

\begin{figure*}[t!]
    \centering
    \begin{subfigure}[t]{0.48\textwidth}
        \centering
        \includegraphics[width=\columnwidth]{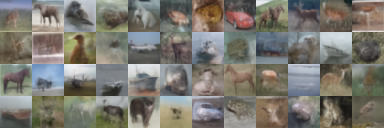}
        \caption{DDIM, NFE = 5, FID = 49.66.}
        \label{fig:cifar10_1}
    \end{subfigure}
    \begin{subfigure}[t]{0.48\textwidth}
        \centering
        \includegraphics[width=\columnwidth]{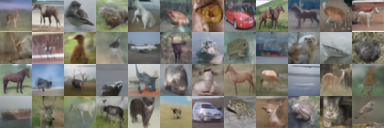}
        \caption{DDIM + GITS, NFE = 5, FID = 28.05.}
        \label{fig:cifar10_2}
    \end{subfigure}
    \\
    \vspace{0.1cm}
    \begin{subfigure}[t]{0.48\textwidth}
        \centering
        \includegraphics[width=\columnwidth]{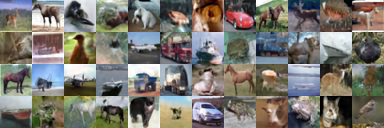}
        \caption{iPNDM, NFE = 5, FID = 13.59.}
        \label{fig:cifar10_3}
    \end{subfigure}
    \begin{subfigure}[t]{0.48\textwidth}
        \centering
        \includegraphics[width=\columnwidth]{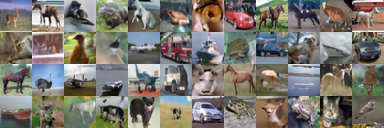}
        \caption{iPNDM + GITS, NFE = 5, FID = 8.38.}
        \label{fig:cifar10_4}
    \end{subfigure}
    \caption{Synthesized samples on CIFAR-10 with DDIM (+ GITS) and iPNDM (+ GITS).}
    \label{fig:cifar10_comparison}
\end{figure*}

\begin{figure*}[t]
    \centering
    \begin{subfigure}[t]{0.48\textwidth}
        \centering
        \includegraphics[width=\columnwidth]{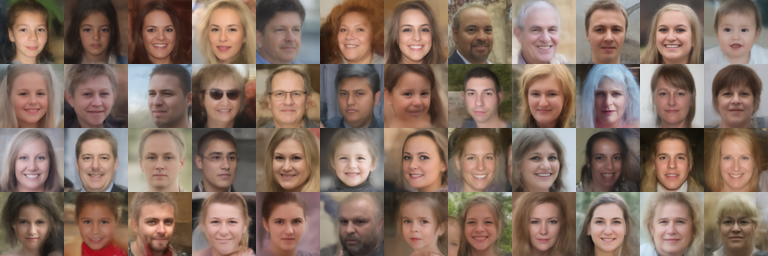}
        \caption{DDIM, NFE = 5, FID = 43.93.}
        \label{fig:ffhq_1}
    \end{subfigure}
    \begin{subfigure}[t]{0.48\textwidth}
        \centering
        \includegraphics[width=\columnwidth]{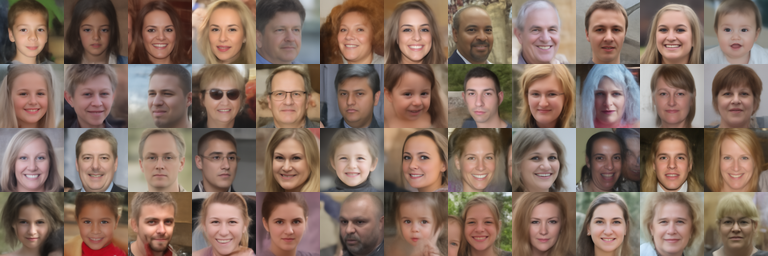}
        \caption{DDIM + GITS, NFE = 5, FID = 29.80.}
        \label{fig:ffhq_2}
    \end{subfigure}
    \\
    \vspace{0.1cm}
    \begin{subfigure}[t]{0.48\textwidth}
        \centering
        \includegraphics[width=\columnwidth]{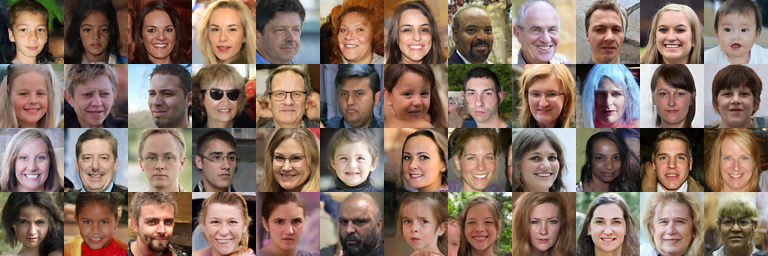}
        \caption{iPNDM, NFE = 5, FID = 17.17.}
        \label{fig:ffhq_3}
    \end{subfigure}
    \begin{subfigure}[t]{0.48\textwidth}
        \centering
        \includegraphics[width=\columnwidth]{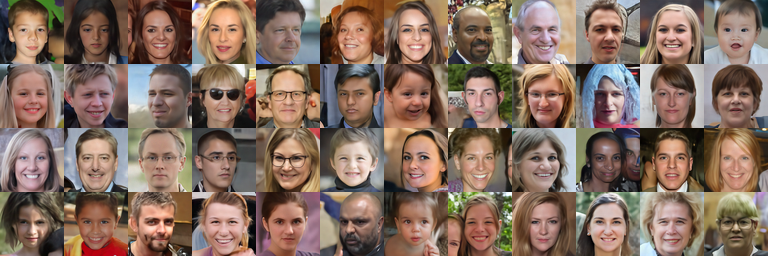}
        \caption{iPNDM + GITS, NFE = 5, FID = 11.22.}
        \label{fig:ffhq_4}
    \end{subfigure}
    \caption{Synthesized samples on FFHQ $64 \times 64$ with DDIM (+ GITS) and iPNDM (+ GITS).}
    \label{fig:ffhq_comparison}
\end{figure*}

\begin{figure*}[t]
    \centering
    \captionsetup[subfigure]{labelformat=simple}
    \begin{subfigure}[t]{0.48\textwidth}
        \centering
        \includegraphics[width=\columnwidth]{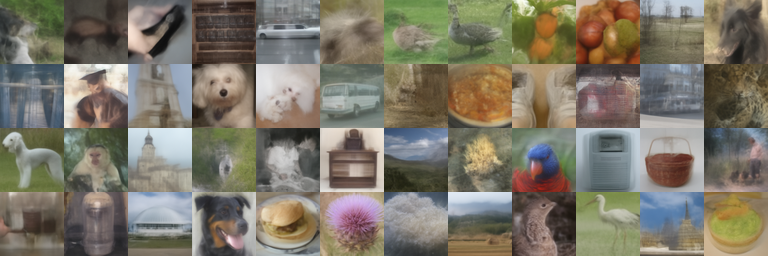}
        \caption{DDIM, NFE = 5, FID = 43.81.}
        \label{fig:imagenet_1}
    \end{subfigure}
    \begin{subfigure}[t]{0.48\textwidth}
        \centering
        \includegraphics[width=\columnwidth]{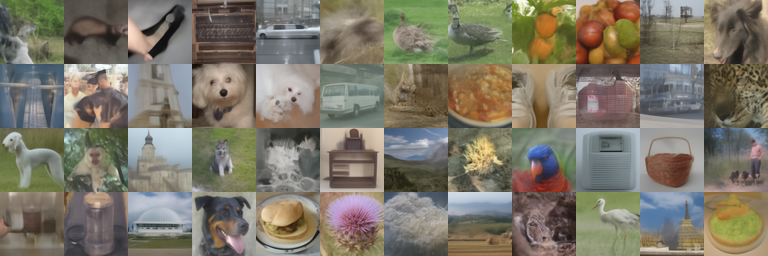}
        \caption{DDIM + GITS, NFE = 5, FID = 24.92.}
        \label{fig:imagenet_2}
    \end{subfigure}
    \\
    \vspace{0.1cm}
    \begin{subfigure}[t]{0.48\textwidth}
        \centering
        \includegraphics[width=\columnwidth]{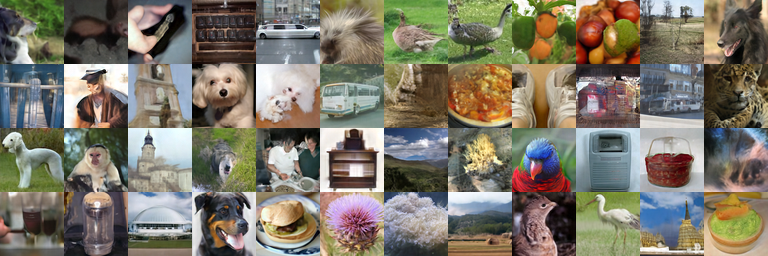}
        \caption{iPNDM, NFE = 5, FID = 18.99.}
        \label{fig:imagenet_3}
    \end{subfigure}
    \begin{subfigure}[t]{0.48\textwidth}
        \centering
        \includegraphics[width=\columnwidth]{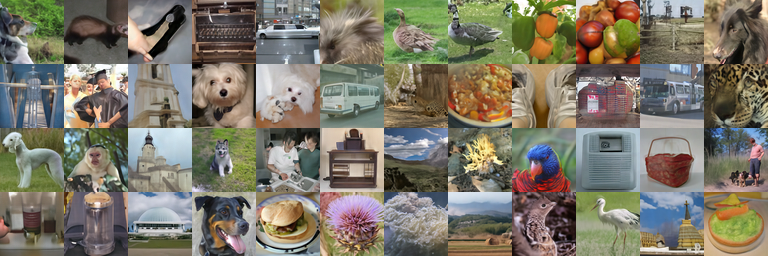}
        \caption{iPNDM + GITS, NFE = 5, FID = 10.79.}
        \label{fig:imagenet_4}
    \end{subfigure}
    \caption{Synthesized samples on ImageNet $64 \times 64$ with DDIM (+ GITS) and iPNDM (+ GITS).}
    \label{fig:imagenet_comparison}
\end{figure*}

\begin{figure*}[t]
    \centering
    \begin{subfigure}[t]{0.48\textwidth}
        \centering
        \includegraphics[width=\columnwidth]{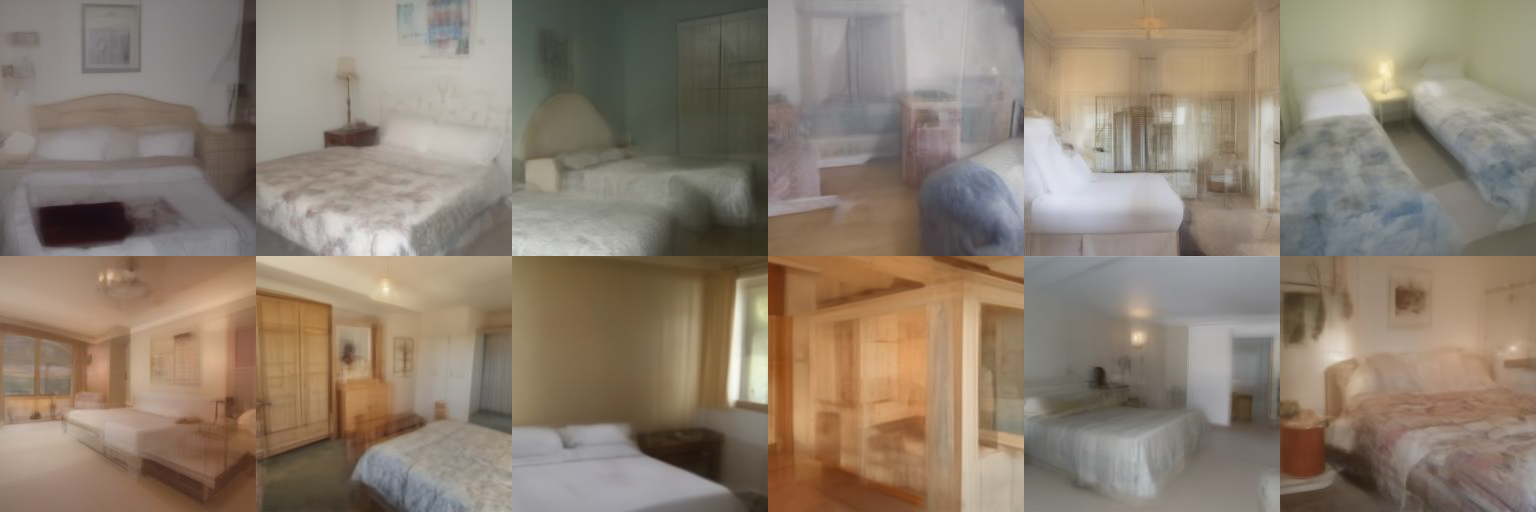}
        \caption{DDIM, NFE = 5, FID = 34.34.}
        \label{fig:bedroom_1}
    \end{subfigure}
    \begin{subfigure}[t]{0.48\textwidth}
        \centering
        \includegraphics[width=\columnwidth]{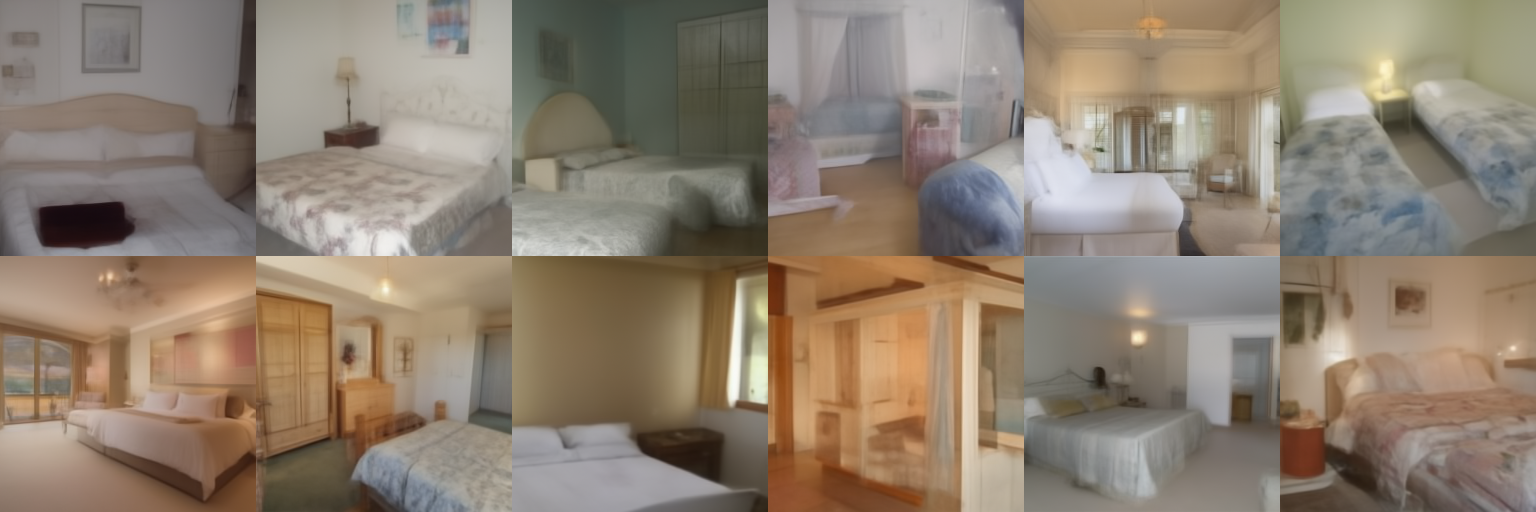}
        \caption{DDIM + GITS, NFE = 5, FID = 22.04.}
        \label{fig:bedroom_2}
    \end{subfigure}
    \\
    \vspace{0.1cm}
    \begin{subfigure}[t]{0.48\textwidth}
        \centering
        \includegraphics[width=\columnwidth]{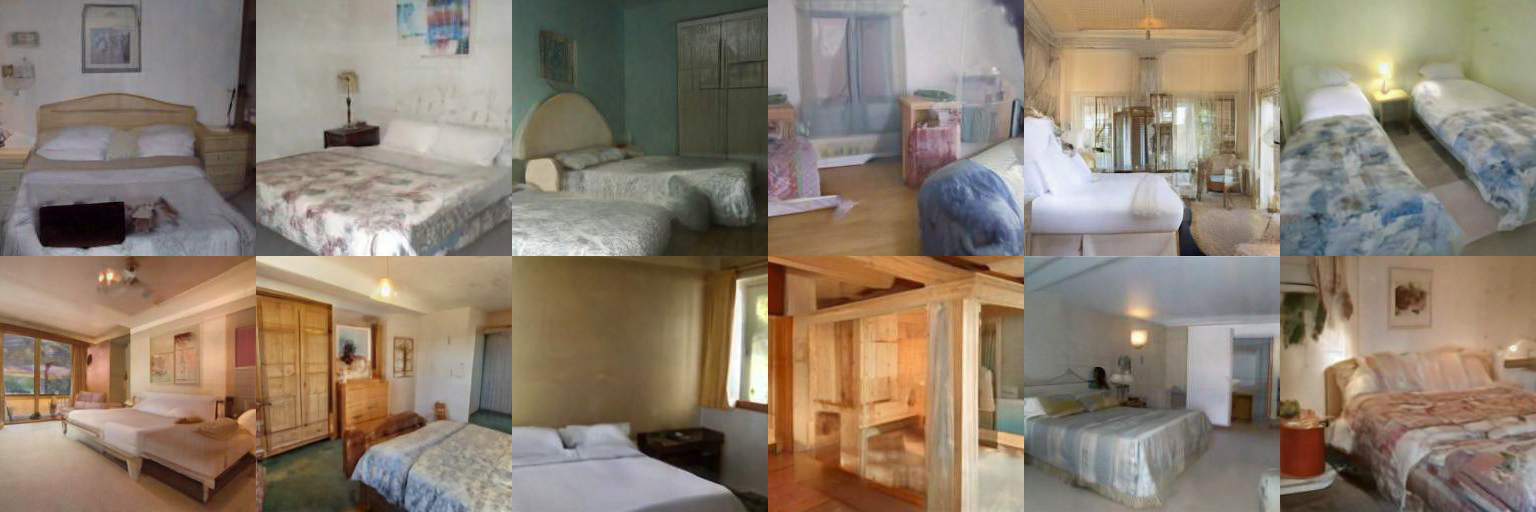}
        \caption{iPNDM, NFE = 5, FID = 26.65.}
        \label{fig:bedroom_3}
    \end{subfigure}
    \begin{subfigure}[t]{0.48\textwidth}
        \centering
        \includegraphics[width=\columnwidth]{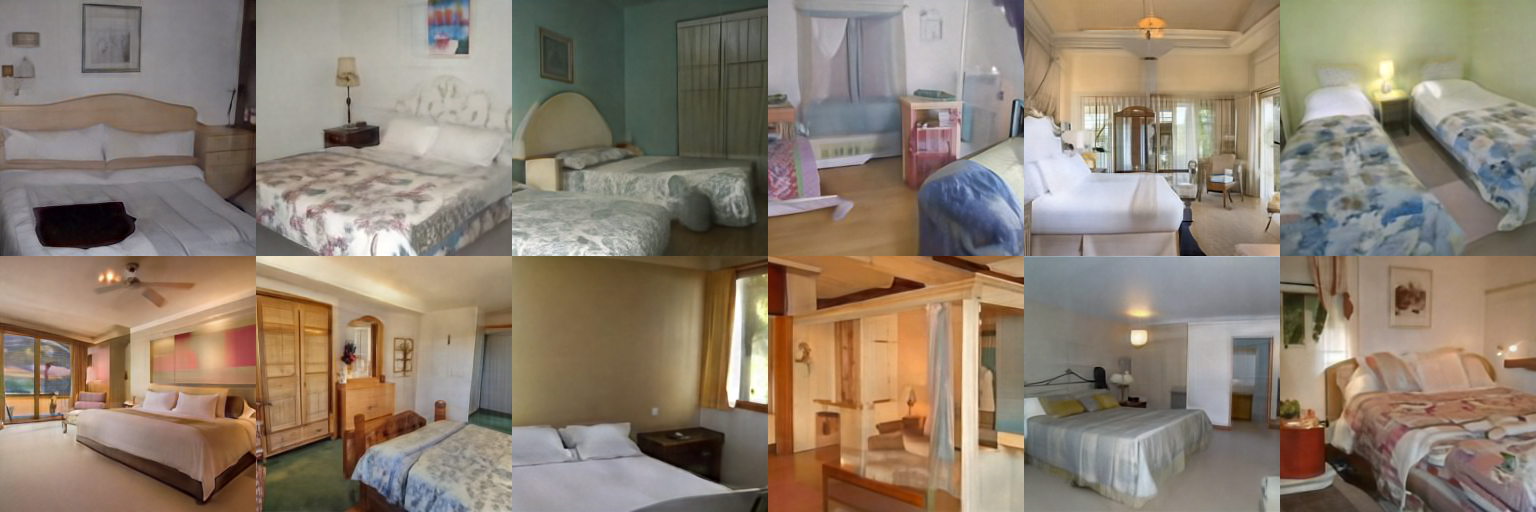}
        \caption{iPNDM + GITS, NFE = 5, FID = 15.85.}
        \label{fig:bedroom_4}
    \end{subfigure}
    \caption{Synthesized samples on LSUN Bedroom $256 \times 256$ (pixel-space) with DDIM (+ GITS) and iPNDM (+ GITS).}
    \label{fig:bedroom_comparison}
\end{figure*}

\end{document}